%% file: main.tex
\documentclass[11pt]{article}
\usepackage{graphicx} 
\def\withcolors{1}
\def\withnotes{1}
\usepackage{macros}
\usepackage{multicol}
\usepackage{empheq}
\usepackage{multirow}
\usepackage{bbm}
\usepackage{graphicx}
\usepackage{wrapfig}
\usepackage{braket}
\usepackage{algorithm}
\usepackage{algpseudocode}
\usepackage{turnstile}
\allowdisplaybreaks
\usepackage[margin=1in]{geometry}

\graphicspath{{./figs/}}

\input{locdef}

\begin{document}


\title{High-Dimensional Robust Mean Estimation from Doubly Corrupted and Heterogeneous Batches}

\author{
Maryam Aliakbarpour\thanks{\texttt{maryama@rice.edu}. Department of Computer Science, Ken Kennedy Institute, Rice University.}
\and
Vladimir Braverman\thanks{\texttt{vova@cs.jhu.edu}. Johns Hopkins University.}
\and
Yuhan Liu\thanks{\texttt{yuhan-liu@rice.edu}. Rice University.}
\and
Junze Yin\thanks{\texttt{jy158@rice.edu}. Rice University.}
}
\date{}

\begin{titlepage}
  \maketitle
  \begin{abstract}
\input{abstract_new}

  \end{abstract}
  \thispagestyle{empty}
\end{titlepage}












\input{introduction}

\input{technique_overview}

\clearpage
\appendix
\thispagestyle{empty}


\input{preliminaries}

\input{0_refinement}

\input{1_collaborate_with_close_mean}

\input{2_collaborate_with_arbitrary_corruption}

\input{3_adversarial_mean_shift}

\input{hardness}



\bibliography{refs}
\bibliographystyle{alpha}

\end{document}

%% file: locdef.tex
\newcommand{\dims}{d}

\newcommand{\cA}{\mathcal{A}}

\def\multiset#1#2{\ensuremath{\left(\kern-.3em\left(\genfrac{}{}{0pt}{}{#1}{#2}\right)\kern-.3em\right)}}

\newcommand{\eye}{\mathbb{I}}



%% file: abstract_new.tex
Modern data systems often require aggregating information from multiple sources: some might be entirely untrustworthy, while the remaining presumably reliable sources may still suffer from a small amount of corruption.
We study robust mean estimation under this practical scenario, which is formulated through the \emph{double corruption model}: samples are divided among $N$ users, each with $n$ samples. $\varepsilon$ fraction of the users could be arbitrarily corrupted, while each of the remaining ``good'' users still have $\alpha$-fraction ($\alpha>0$) of corrupted samples. 
The goal is to
estimate the unknown mean $\mu$ with small $\ell_2$ error. 
This is a challenging scenario because we must handle two types of adversarial corruption: user-level corruption and sample-level corruption.
When the unknown distribution has covariance matrix bounded by the identity matrix, we design a sum-of-squares (SoS) algorithm that achieves an $\ell_2$ error of $\sqrt{\alpha}$ with probability at least $2/3$ as long as the total number of samples $nN=\widetilde{\Omega}(d/\alpha)$.
The estimation error and sample complexity are min-max optimal under the mild condition. 
For comparison, treating all $nN$ samples as $(\varepsilon+\alpha)$-corrupted and applying standard robust mean estimation only achieves an error of $\sqrt{\varepsilon+\alpha}$. 
Thus, our new algorithm removes the adverse effect of untrustworthy users and achieves the optimal error as if only sample-level corruption is present.
Also, our SoS algorithm extends to the heterogeneous setting where each user's local mean could be additionally shifted by $\sqrt{\beta}$ in $\ell_2$ distance. 
We show that an error of $O(\sqrt{\alpha+\beta})$ can be achieved with probability at least $2/3$ when $nN=\widetilde{\Omega}(d/(\alpha+\beta))$.

%% file: introduction.tex
\section{Introduction}

Modern data analysis increasingly relies on information aggregated from decentralized and heterogeneous sources \citep{vhj22,kls21,yfd+23,lksj21}. In such settings, data may be collected by different institutions, devices, or individuals under varying measurement procedures, demographic compositions, and environmental conditions. Consequently, the classical assumption that all samples are i.i.d.\ draws from a single target distribution is often too restrictive \citep{fass21,shk+22,skk21}. At the same time, pooling data across sources can make it possible to solve statistical tasks that would be impossible for any individual source. This raises a fundamental question: how should one use data from sources that may differ from the target distribution and may not all be trustworthy?

We study this question in the context of perhaps the most fundamental task in high-dimensional statistics: estimating the mean $\mu$ of a target distribution over $\R^\dims$. The data are distributed among $N$ users, each holding a batch of $n$ samples. We work in the spirit of the \emph{untrusted batch model} \citep{QiaoV18untrusted}, in which an $\eps$-fraction of users may be completely adversarial. The remaining users are not assumed to draw their data perfectly i.i.d.\ from the target distribution. Instead, each good user is required only to be \emph{relevant} to the target in a quantitative sense.

There are two natural and complementary ways in which a relevant source may deviate from the target. First, a source may be off \emph{everywhere, but only mildly}: for example, a hospital may serve a slightly different population, a sensor may have a small calibration bias, or data may be collected during a different season. We model this deviation as heterogeneous mean shift: user $i$ draws samples from a distribution $P_i$ with mean $\mu_i$ satisfying
\(
    \norm{\mu_i-\mu}_2 \le \sqrt{\beta} .
\)
Second, a source may be off \emph{rarely, but arbitrarily}: most records are drawn from the appropriate population, but a small fraction may be corrupted by logging errors, sensor failures, or malicious manipulation. We model this deviation as sample-level strong contamination: after observing the clean samples, an adversary may replace an $\alpha$-fraction of the samples inside each good user's batch by arbitrary or even adversarially chosen points.

Our main contribution is a polynomial-time sum-of-squares (SoS) algorithm for robust mean estimation in this doubly corrupted and heterogeneous batch model. Informally, up to logarithmic factors, our estimator achieves the nearly optimal rate
\[
    \wt{\Theta}\left(
        \sqrt{\frac{\eps}{n}}
        + \sqrt{\frac{\dims}{Nn}}
        + \sqrt{\alpha+\beta}
    \right).
\]
This rate has a natural interpretation. The terms $\alpha$ and $\beta$ represent inherent statistical difficulty among the relevant sources. In contrast, the effect of fully adversarial users is reduced by the batch size: averaging $n$ clean samples in a good batch decreases the covariance scale by a factor of $1/n$, so the user-level corruption contributes only $\sqrt{\eps/n}$ rather than $\sqrt\eps$.

\subsection{Problem Setup}
\label{subsec:problem_setup}

We start by introducing the double corruption model, which applies to a setting where a group of $N$ users, holding $\ns$ data points each, collectively performs an estimation task. 
The term \emph{double corruption} refers to the two levels at which strong contamination may occur: an $\eps$-fraction of entire user batches may be corrupted, and within every remaining batch an $\alpha$-fraction of individual samples may also be corrupted. 

\begin{definition}[($\eps,\alpha$)-double corruption]
    Let $\eps,\alpha\in[0,1]$ be corruption parameters. Given $N$ users where user $i$ has data points $z_{i,j}\in\R^\dims$, $j\in[n]$, an adversary may examine the entire dataset and make the following changes:
    \begin{enumerate}
        \item choose a set of users $\mathcal B\subseteq[N]$ with $|\mathcal B|\le \eps N$;
        \item for every user $i\in\mathcal B$, replace all data points $z_{i,j}$ arbitrarily;
        \item for every user $i\notin\mathcal B$, replace at most $\alpha n$ points in the multiset $\{z_{i,j}:j\in[n]\}$ arbitrarily.
    \end{enumerate}
\end{definition}
Our formulation is inspired by the strong contamination model \citep{diakonikolas2023algorithmic}, but introduces an additional level of adversarial corruption at the user level. We study robust mean estimation under this new model of corruption.
\begin{restatable}[Robust mean estimation under double corruption]{problem}{ProbDoubleCorruption}
Let $\alpha,\eps\in(0,0.5)$\footnote{When $\eps\in[0.5,1)$, it is impossible to perform robust mean estimation with a single output in general. One may instead ask for list-decodable mean estimation, where the goal is to output a short list of candidates containing at least one accurate estimate \citep{csv17}; see \cref{sec:related}.}. There are $N$ users. User $i$ has $n$ \iid samples $\{z_{i,j}:j\in[n]\}$ from a distribution $P$ with mean $\mu\in\R^\dims$ and covariance $\Sigma\preceq\eye_\dims$. 
The samples are $(\eps,\alpha)$-double corrupted. The goal is to estimate $\mu$ in $\ell_2$ distance.
\end{restatable}

\paragraph{Special cases}
Before discussing the full model, it is helpful to isolate the roles of the two parameters $\eps$, $\alpha$. For simplicity, consider a high constant success probability (say at least 2/3).

First, without corruption, i.e. $\eps=\alpha=0$, standard statistics show that an $\ell_2$ error of $\Theta(\sqrt{d/Nn})$ is optimal.
Second, when $\eps=0, \alpha>0$, all $Nn$ samples can be viewed as $\alpha$-contaminated, and thus we can directly use standard robust mean estimation algorithms~\citep{dkk+19,lrv16,DiakonikolasKP20} which achieve an error of $O(\sqrt{\alpha})$ using $Nn= O\paren{\dims/\alpha}$~\cite{pp25} samples\footnote{\cite{dkk+19,DiakonikolasKP20} proved a sample complexity of $O(\dims\log(\dims)/\alpha)$, and \cite{pp25} removed the $\log(\dims)$ factor.}.
Third, when $\eps>0, \alpha=0$, the empirical mean of each user has covariance bounded by $\eye_\dims/n$ instead of $\eye_\dims$. Thus, we can scale all samples by $\sqrt{n}$ and apply standard robust mean estimation, which gives an error of $\sqrt{\eps/\ns}$ with $N= O(\dims/\eps)$.

\paragraph{Difficulty of double corruption}
While the special cases above have simple solutions, combining user-level corruption with sample-level corruption creates a new difficulty. 
If we ignore the batch structure, then the total fraction of corrupted samples is at most $\eps+\alpha$, and a standard robust mean estimator gives error $O\paren{\sqrt{\eps+\alpha}}$ by $Nn=O\paren{\dims/(\eps+\alpha)}$ samples. It is sample-efficient but loses the $1/\sqrt n$ gain that should be possible for user-level corruption.

A second attempt is to robustly estimate a mean separately within each user and then aggregate these estimates. 
This can remove the dependence on $\eps$ and achieve error $O\paren{\sqrt\alpha}$, but it requires every user to have enough samples for high-dimensional robust mean estimation, namely $n=\Omega\paren{\dims/\alpha}$. This is much stronger than the optimal total-sample requirement $Nn= O\paren{\dims/\alpha}$, and is incompatible with the important regime $1\le n\le \dims$. Thus, neither approach simultaneously achieves the right error and the right sample complexity. This motivates the question:
\begin{question}
\label{ques:double}
Under the double corruption model, can we design a robust mean estimation algorithm that is both accurate and sample-efficient, suppressing the contribution of $\eps$ by the batch size while using only near-linear total samples in $\dims$?
\end{question}

Answering this question already requires ideas beyond a black-box application of existing robust mean estimation algorithms.

\paragraph{Heterogeneous data}
Different sources of data may have distinct data distributions in practice. 
We further combine data heterogeneity with the double corruption model, therefore introducing the most general problem formulation studied in this work.
\begin{restatable}[Robust mean estimation under double corruption and heterogeneous data]{problem}{ProbHeterogeneousTwoLevel}
\label{prob:heterogeneous_two_level}
Let $\alpha,\eps\in(0,0.5)$ and $\beta\ge0$. There are $N$ users. User $i$ has $n$ \iid samples $\{z_{i,j}:j\in[n]\}$ from a distribution $P_i$ with mean $\mu_i\in\R^\dims$ and covariance $\Sigma_i\preceq\eye_\dims$. We assume that there exists a target mean $\mu\in\R^\dims$ such that, for every $i\in[N]$,
    $\norm{\mu_i-\mu}_2 \le \sqrt\beta$.
The samples are then $(\eps,\alpha)$-double corrupted by an adversary. The goal is to estimate $\mu$ in $\ell_2$ distance.
\end{restatable}

\begin{figure}[!ht]
    \centering
    \includegraphics[width=\linewidth]{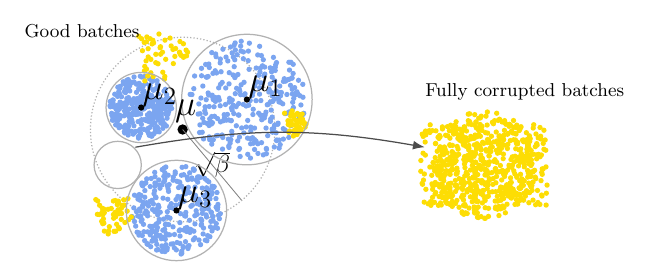}
    \caption{A visualization of \cref{prob:heterogeneous_two_level}. There are $N=4$ batches, each containing $n=350$ samples. The batches are $1/4$-strongly contaminated. In each uncorrupted batch, the samples are $1/7$-strongly contaminated. The heterogeneous batch means are within distance $\sqrt\beta$ of the target mean $\mu$. Blue points are clean samples, and yellow points are corrupted samples.}
    \label{fig:model_3}
\end{figure}

It is worth noting that the two notions of relevance—the $\beta$-mean shift and the $\alpha$-fraction corruption—are not interchangeable in high dimensions. In a bounded domain, moving the mean may require changing the entire probability mass of $P$. Conversely, replacing even a single sample in a batch can move the empirical mean of that batch by an arbitrarily large amount, which cannot be captured by any bounded-magnitude mean shift. 
Our model therefore allows both types of deviation simultaneously. Together with an $\eps$-fraction of fully adversarial users, this yields a two-level relevance model that is robust to corrupted batches: as long as a $(1-\eps)$-fraction of the users are both $\alpha$- and $\beta$-relevant to the target distribution, the model permits arbitrary deviations among the remaining users.

Without double corruption, $\eps=\alpha=0$, the empirical mean estimator achieves
$\sqrt{\beta + \frac{\dims}{N\ns}}$ error.
However, because of the inherent difficulty of the double-corruption model and the fundamentally different nature of these two types of local perturbation, it is unclear how to handle them simultaneously.

\paragraph{Relation to existing corruption models}
Our formulation appears related to a line of work on the untrusted batch model for discrete distribution learning~\citep{QiaoV18untrusted,ChenLM2020untrusted,jain2020robustbatch}, where an $\eps$-fraction of the batches may be arbitrary, while each uncorrupted batch is drawn from a distribution that is close to the target distribution in total variation distance.

For discrete distributions, sample corruption and total-variation perturbation are closely related: changing a distribution by at most $\alpha$ in total variation can be viewed, under an appropriate coupling, as changing an $\alpha$-fraction of the samples in expectation. These changes occur randomly according to some potentially worst-case distribution within total variation distance $\alpha$ of the target. This formulation is particularly convenient for distribution learning, because the empirical distribution of each batch changes by approximately $\alpha$ in total variation distance. In our model, however, the changes are made adversarially. Even a single corrupted sample can move the empirical mean of a batch arbitrarily far. This is the key distinction between our work and~\citep{QiaoV18untrusted,ChenLM2020untrusted,jain2020robustbatch}.

\subsection{Main Results}
\label{sub:result}

Our theorem answers \cref{ques:double} and extends it to the full heterogeneous model in \cref{prob:heterogeneous_two_level}.

\begin{theorem}\label{thm:heterogeneous_two_level_upper}
Let $\delta\in \paren{0, 1/10}$ and $c_1,c_2 > 0$ be sufficiently small. Given $\eps, \alpha\in(0,c_1)$, $\beta\in[0,c_2)$, and $n$ samples from each of $N\in\Z_+$ users satisfying
        $nN
        \ge
        \Omega\left(
            \frac{\dims}{\alpha+\beta+\eps/n}
            \log\left(\frac{\dims}{\delta}\right)
        \right)$,
    there exists an algorithm with running time $(Nn\dims+\dims^2)^{O(1)}$ that solves \cref{prob:heterogeneous_two_level}. With probability at least $1-\delta$, the algorithm outputs $\wh\mu\in\R^\dims$ satisfying
    \[
        \norm{\wh\mu-\mu}_2
        =
        O\paren{\sqrt{\alpha+\beta+\frac{\eps}{n}}}.
    \]
\end{theorem}

Equivalently, the ordinary statistical term $\sqrt{\dims/(Nn)}$ is dominated under the stated sample size. In the non-vacuous sample-corruption regime $\alpha n\ge1$, the user-level term satisfies $\eps/n\le O(\alpha)$ for constant $\eps$, and the error simplifies to $O\paren{\sqrt{\alpha+\beta}}$. Thus the fully corrupted users have no asymptotic effect beyond what is already forced by the sample-level corruption and heterogeneity.

We complement the upper bound with a matching information-theoretic lower bound. Under the mild condition $\alpha n
    \ge
    4\log\paren{\frac{4}{\eps}}$ needed in the lower-bound construction, all estimators incur error at least
    $\Omega\paren{\sqrt{\alpha+\beta+\eps/n}}$.
In this lower-bound regime, the error rate in \cref{thm:heterogeneous_two_level_upper} is minimax optimal up to universal constants, and the sample complexity is near-optimal up to logarithmic factors. Under the upper-bound sample-size requirement and $\alpha n\ge1$, the corruption and heterogeneity terms reduce to $\Omega\paren{\sqrt{\alpha+\beta}}$, matching the simplified upper bound.

\paragraph{Algorithmic ideas}
The algorithm is based on a constant-degree SoS relaxation that searches simultaneously over two hidden objects: a large set of reliable users and, inside each selected user, a large subset of samples that agrees with the observations. The sample-level constraints provide a crude refinement that removes enough individual outliers to control the aggregate sample covariance. The user-level constraints then operate on cleaned batch means and enforce a sharper covariance bound at scale $1/n$ (plus a carefully chosen slack term). This coupling lets the algorithm combine the two advantages that the naive approaches could only obtain separately: near-linear total sample complexity and the $1/\sqrt n$ suppression of user-level corruption.

Extending the double-corruption algorithm to heterogeneous means requires additional care. The sample-level truncation must be performed around each user's local mean rather than around the target mean, while the user-level refinement must still aggregate toward a common target $\mu$. We introduce relaxed selection parameters that balance the sample-level corruption $\alpha$, the heterogeneity $\beta$, and the fraction of fully corrupted users $\eps$. The resulting SoS system remains satisfiable with near-optimal sample complexity, and its identifiability proof yields the rate in \cref{thm:heterogeneous_two_level_upper}. A detailed technical overview is given in \cref{sec:technique}.

\begin{remark}
    Our sample complexity is nearly optimal and has a logarithmic factor on the failure probability. In classical robust mean estimation and in recent one-level local/global corruption models, one can often obtain a sub-Gaussian dependence where the logarithmic term appears additively \citep{DiakonikolasKP20,pp25}. Extending such rates to the present two-level corruption model is nontrivial. We specify obstacles in \cref{sub:arbitrary_alpha:subgau}.
\end{remark}

\begin{remark}
    Our running time is polynomial in $N$, $n$, and $\dims$, and is comparable
to standard SoS algorithms for robust mean estimation
\citep{hopkins2018mixture,kothari2018robust}. 
As in prior SoS work, we do not
optimize the polynomial degree, and improving the running time is an interesting
direction for future work. 
The main advantage of SoS over more computationally
efficient approaches such as filtering is that it handles the sample-level and
user-level corruptions jointly. In contrast, a naive sequential filtering
approach may fail because the sample-level filter must make its decisions
without knowing which users will later be identified as corrupted. In
\cref{sub:arbitrary_alpha:filter}, we construct two candidate clean
distributions whose means differ by $\Theta(\sqrt{\eps})$ but which, after
corrupting an $\eps$-fraction of users, induce exactly the same pooled sample
distribution. Hence, any filtering algorithm that ignores the batch structure
cannot distinguish which candidate is the true clean distribution and may
discard $\Omega(\eps)$ genuinely clean samples in one of the two instances.
These deletions are irreversible: even if a subsequent user-level filter
perfectly identifies all corrupted users, it cannot recover the clean samples
discarded in the first stage. This motivates our joint SoS formulation, which
selects the sample-level and user-level refinements simultaneously. We specify obstacles in \cref{sub:arbitrary_alpha:filter}.
\end{remark}

\subsection{Related Work}
\label{sec:related}
\paragraph{Robust statistics} 

Robust mean estimation is a long-standing topic in statistics, dating back to~\cite{huber1964robust}.
The Tukey median~\citep{tukey1960survey} achieves the optimal error for mean
estimation but is computationally inefficient in high dimensions. The works
of~\cite{dkk+19,lrv16} resolved this computational inefficiency and inspired
many subsequent works on efficient and robust learning algorithms in various
settings of practical interest. \cite{CDG17,DiakonikolasKP20,dkpp22} improved
the computational efficiency, sample complexity, and memory efficiency, achieving nearly-linear
time, nearly sub-Gaussian rates, and near-optimal memory, respectively. 
Other works have
studied robust estimation under heterogeneous corruption rates across samples
\citep{chaudhuri2025robust}, sparse mean estimation~\citep{dkk+22}, and
high-dimensional robust mean estimation via gradient descent~\citep{cdgs20}.
The SoS method~\citep{nesterov2000squared, lasserre2001new} has proven to be a powerful tool for robust statistics~\citep{kothari2017outlier, hopkins2017SoS,kothari2018robust}. It provides a unified framework to design efficient algorithms as opposed to relying on problem-specific heuristics. For many robust estimation problems, optimal or best-known time and sample complexity can often be achieved using the SoS approach~\citep{BakshiDHKKK20robustSoS,KothariMZ22, Hopkins2020robustheavy,hopkins2025subGmean, DiakonikolasHPT25SoS}.

Another line of work combines strong contamination model with local perturbation. It was first studied by \cite{zjt20}, and in this setting, \cite{ll23} focused on covariance estimation and linear regression. However, these works studied the statistical aspect instead of giving a computational efficient algorithm. \citep{nietert24localglobal} designed efficient algorithm for
robust mean estimation in this strong contamination and perturbation setting. \citep{pp25} strengthened this by considering a stronger directional local-perturbation model.

\paragraph{Learning from multiple (corrupted) sources and tasks}

With no corruption, a line of work studies learning from many related
sources or tasks when each source provides only a small number of samples.
\cite{as20} studied distribution testing from multiple sources with few
samples per source.
\cite{abb+24} developed a sample- and task-complexity theory
for meta- and multitask learning with very few samples per task.
\cite{abs+25} further studied privacy notions for personalized learning.
Another line of work studies learning from corrupted batches.
\cite{QiaoV18untrusted} studied robust density estimation from
corrupted batches for discrete distributions under TV distance.
Subsequent works \citep{ChenLM2020untrusted,jain2020robustbatch} further
improved the computational and sample complexities. 
In
addition to density estimation for discrete distributions, \cite{jo20,jo21}
studied robust density estimation for structured continuous distributions and
under $\mathcal A_k$ distance, which is always upper bounded by TV
distance. \cite{jo20} also designed algorithms for piecewise interval
classification. \cite{jsk+24} analyzed linear regression in the presence of
heterogeneous batches.

\paragraph{List decodable mean estimation}

When the strong contamination parameter becomes $\alpha \geq 1/2$, the clean
samples form only a minority of the observed dataset. In this regime, the
adversary may create several plausible clusters, and it is information-theoretically impossible to determine which cluster corresponds to the clean
distribution without additional information. This motivates the
list-decodable learning framework, where the goal is to output a short list of
candidate estimates such that at least one element of the list is close to the
true mean. This problem was formally introduced and analyzed by~\cite{csv17}. Subsequent works improved the statistical guarantees and computational
efficiency of list-decodable mean estimation and related mixture-learning
problems~\citep{dks18,cmy20,dkk+21,
dkk20}. Other variants, such as list-decodable sparse mean
estimation, have also been studied~\citep{zs22}. Recently,
\cite{cckl26} studied high-accuracy list-decodable mean estimation,
where one can trade a larger list size for better accuracy.

%% file: technique_overview.tex
\section{Technical Overview}
\label{sec:technique}
In this section, we present the main technical ideas in proving our main result~\cref{thm:heterogeneous_two_level_upper}. In \Cref{sub:technique:background}, we briefly review the key insight for standard robust mean estimation and the sum-of-squares algorithm. 
Then, in \cref{sub:technique:ques1}, we present the algorithm for the double corruption model ($\eps>0, \alpha>0, \beta=0$), which contains the key insight on how to handle two layers of corruption and resolves \cref{ques:double}. 
In \cref{sub:technique:ques2}, we extend our solution to the general case where all of $\eps,\alpha,\beta$ can be larger than zero.

\paragraph{Notation}
We define $[N] := \{1, 2, \dots, N\}$ and $\Z_+$ as the set of positive integers. 
For a vector $x \in \R^d$, we write $\|x\|_2 := \sqrt{\sum_{i \in [d]} x_i^2}$ for its $\ell_2$ norm. 
For all vectors $x, y \in \R^d$, we denote $\langle x, y \rangle := \sum_{i \in [d]} x_i \cdot y_i$ as the inner product of $x$ and $y$.
We let $\mathrm{tr}\bracket{A} := \sum_{i \in [d]} A_{i, i}$ be the trace of $A \in \R^{d \times d}$.
For matrices $A, B \in \R^{d \times d}$, we write $A \succeq B$ to denote that $A - B$ is positive semidefinite, namely for all $x \in \R^d$, we have $x^\top \paren{A - B} x \geq 0$. 
For a random variable $X \in \R^d$, $\E[X] \in \R^d$ denotes the expectation and $\mathrm{Cov}[X] := \E[(X - \E[X])(X - \E[X])^\top] \in \R^{d \times d}$ denotes the covariance matrix. 
The identity matrix in $\R^{d \times d}$ is denoted by $\mathbb{I}_d$. 

\subsection{Background: Sum-of-Squares for Robust Mean Estimation}
\label{sub:technique:background}

We first revisit the key insights for standard robust mean estimation 
\cite{hopkins2018mixture,kothari2018robust}, which are necessary to explain the technical ideas in our algorithm. Let \(\{x_j\}_{j=1}^n\subset\R^d\) be $\eps_0$-corrupted samples from a distribution with $\Sigma\preceq\eye_\dims$.
The main structural insight is that the empirical covariance of a set of samples provides a verifiable certificate on whether it is a ``good'' set whose empirical mean is an accurate estimate of the true mean. 
Roughly speaking, with high probability, there exists a set of samples with size at least $(1-\Theta(\eps_0))n$ that has empirical covariance bounded by $O(\eye_\dims)$, 
and with high probability, for all sets with the above property, the empirical mean over the set is $O(\sqrt{\eps_0})$ close to the true mean $\mu$ (even though it may contain corrupted samples).

For a more detailed explanation, please refer to \cite[Chapter 2]{diakonikolas2023algorithmic}.
The main challenge is then to design an efficient algorithm that can find such a subset, without doing a brute-force search.
The sum-of-squares method~\cite{nesterov2000squared,lasserre2001new} provides one such approach, and has been widely used in robust statistics~\cite{kothari2017outlier,kothari2018robust, DiakonikolasHPT25SoS}.
We encode the search for such a candidate set 
through polynomial constraints, which is expressed in terms of  SoS variables
\(\{\mathsf Y_j\}_{j=1}^n\), \(\{W_j\}_{j=1}^n\), and
\(B\in\R^{d\times d}\), and $\ov{\mathsf Y}:=\frac1n\sum_{j=1}^n \mathsf Y_j$, for all $j\in[n]$,
\begin{align}
    W_j^2 &= W_j, \label{eq:t1}\\
    \sum_{j=1}^n W_j &\geq (1-2\eps_0)n, \label{eq:t2}\\
    W_j(\mathsf{Y}_j-x_j) &= 0, \label{eq:t3}\\
    \frac1n\sum_{j=1}^n (\mathsf{Y}_j-\ov{\mathsf{Y}})(\mathsf{Y}_j-\ov{\mathsf{Y}})^\top
    &= 3\mathbb{I}_d - BB^\top . \label{eq:t4}
\end{align}
\cref{eq:t1} makes \(W_j\) an indicator variable, representing whether or not the $j$-th sample is selected.
\cref{eq:t2} ensures at least
\((1-2\eps_0)\)-fraction of the observed samples is selected.
The variable $\mathsf Y_j$ represents what we believe to be the value of the clean samples. \cref{eq:t3} enforces \(\mathsf Y_j=x_j\) whenever \(W_j=1\). Finally,
\eqref{eq:t4} imposes the empirical covariance bound
    $\frac1n\sum_{j=1}^n (\mathsf{Y}_j-\ov{\mathsf{Y}})(\mathsf{Y}_j-\ov{\mathsf{Y}})^\top
    \preceq 3\mathbb{I}_d$.

A brute-force search over all choices of \(W_j\) is exponential.
The SoS instead returns a \emph{pseudoexpectation operator}  \(\widetilde{\E}\), which is a linear operator that takes as input a polynomial in the SoS variables and outputs a real number.
The operator $\widetilde{\E}$ satisfies the constraints, in the sense that when applied to both sides of Eqs.~\eqref{eq:t1}-\eqref{eq:t4}, the equations/inequalities are still valid.
The algorithm then outputs $\wh\mu:=\widetilde{\E}\left[\ov{\mathsf Y}\right]$ as the estimate of the mean.
To prove the correctness of the algorithm involves two main steps,
\begin{enumerate}
    \item \textbf{Satisfiability}: there exists an assignment of variables such that all constraints \eqref{eq:t1}-\eqref{eq:t4} are satisfied with high probability, implying the existence of a pseudoexpectation \(\widetilde{\E}\) that satisfies the constraints.
    \item \textbf{Identifiability}: any pseudoexpectation operator $\widetilde{\E}$ that satisfies \eqref{eq:t1}-\eqref{eq:t4} gives an accurate estimate of the mean.
\end{enumerate}
For more details, please refer to \Cref{sec:rme_SoS}.

\subsection{Double Corruption Without Heterogeneity: Solving Question~\ref{ques:double}}
\label{sub:technique:ques1}
We now consider the double-corruption setting with $\alpha,\eps>0$ and $\beta=0$.

The main challenge comes from the corruption in each user:
with the observed sample $x_{i, j} \in \R^d$, the observed batch mean
\(\frac1n \sum_{j=1}^n x_{i,j}\) can be shifted by an arbitrary amount, and thus applying standard robust mean estimation over the batch means does not lead to any meaningful result.
On the other hand, recall from~\cref{subsec:problem_setup}, applying robust mean estimation over $N\ns$ samples with corruption level $\eps+\alpha$ ignores the batch structure, and thus only achieves an error of $O(\sqrt{\eps+\alpha})$.

To resolve this dilemma, we design a two-level polynomial system which consists of \textbf{crude sample-level refinement} and \textbf{user-level refinement}. 
The crude refinement step performs an initial filter on the bad samples, in a way similar to applying standard robust mean estimation to all $Nn$ samples with corruption level $\eps+\alpha$.
After this step, the empirical mean of good users over the remaining samples would not shift arbitrarily, making it possible to further refine the estimate by leveraging the concentration afforded by the batch structure.

\paragraph{Polynomial system design: Two-level refinement}
With SoS variables $\{\mathsf Z_{i,j}\}_{i\in[N],j\in[n]} \subset \R^d$, $\{\mathsf Y_i\}_{i=1}^N \subset \R^d$, $\{W_{i,j}\}_{i\in[N],j\in[n]} \subset \{0, 1\}$, $\{U_i\}_{i=1}^N \subset \{0, 1\}$, and $B_1,B_2\in\R^{d\times d}$, we define $\mathsf Z_i = \frac1n\sum_{j=1}^n \mathsf Z_{i,j}$, $\overline{\mathsf Z} = \frac1{Nn}\sum_{i=1}^N\sum_{j=1}^n \mathsf Z_{i,j}$, and $\ov{\mathsf Y}:=\frac1N\sum_{i=1}^N \mathsf Y_i$. We introduce a novel SoS system that combines two levels of refinement, for all $i\in[N]$ and $j\in[n]$,

\textit{Crude sample-level refinement.}
\begin{align}
W_{i,j}^2 &= W_{i,j},
\label{eq:T_1}\\
\sum_{i=1}^N\sum_{j=1}^n U_iW_{i,j}
&\ge
(1-4\alpha)n\sum_{i=1}^N U_i,
\label{eq:T_8}\\
W_{i,j}(\mathsf Z_{i,j}-x_{i,j})
&=0,
\label{eq:T_3}\\
\frac{1}{Nn}\sum_{i=1}^N\sum_{j=1}^n
(\mathsf Z_{i,j}-\overline{\mathsf Z})
(\mathsf Z_{i,j}-\overline{\mathsf Z})^\top
&=
O(1)\mathbb I_d-B_1B_1^\top.\notag
\end{align}
\textit{User-level refinement.}
\begin{align}
U_i^2 &= U_i,
\label{eq:t6}\\
\sum_{i=1}^N U_i
&\ge(1-2\eps)N,
\label{eq:t7}\\
U_i(\mathsf Y_i-\mathsf Z_i)
&=0,
\label{eq:t8}\\
\frac1N\sum_{i=1}^N
(\mathsf Y_i-\overline{\mathsf Y})
(\mathsf Y_i-\overline{\mathsf Y})^\top
&=
O\left(\frac1n+\tau\right)\mathbb I_d
-B_2B_2^\top. \notag
\end{align}
\emph{Crude sample-level refinement} is similar to the structure of Eq.~\eqref{eq:t1}--\eqref{eq:t4}.
It is applied to the sample-level variables
inside selected users. For each user \(i\), we introduce variables
\(\{\mathsf Z_{i,j}\}_{j\in[n]}\) and indicators \(\{W_{i,j}\}_{j\in[n]}\).
The variables \(W_{i,j}\) play the same role as the one-level indicators
\(W_j\), and the variables \(\mathsf Z_{i,j}\) play the role of the one-level candidate
cleaned samples \(\mathsf Y_j\) from \Cref{sub:technique:background}. 
The user indicator $U_i$ links the sample-level and user-level refinements. The global cardinality constraint Eq.~\eqref{eq:T_8} requires that, across all selected users, at least a $(1-4\alpha)$-fraction of their samples be certified. Thus, sufficiently many candidate samples $\mathsf Z_{i,j}$ are constrained to agree with the observations $x_{i, j}$ (Eq.~\eqref{eq:T_3}).
This refinement controls the total surviving
sample-level corruptions, but by itself it does not fully exploit the
\(1/n\) variance reduction from averaging within each batch.

The second component is the \emph{user-level refinement}. We
introduce variables \(\mathsf Y_i\) for the cleaned user-level means and select
at least a \((1-2\eps)\)-fraction of users through the indicators \(U_i\). Eq.~\eqref{eq:t8} couples two levels: for
every selected user $i$ ($U_i = 1$), the user-level point \(\mathsf Y_i\) must equal the
cleaned batch mean \(\mathsf Z_i\). \(1/n\) is the variance
reduction term from averaging \(n\) clean samples, while the additional
\(\tau=\alpha/\eps\) is the key new parameter, which we will specify below. In short, it is chosen to balance the two
requirements: it is large enough to improve \textbf{satisfiability}, but small enough
that the \textbf{identifiability} error remains \(O(\sqrt{\alpha})\).
The algorithm computes a constant degree pseudoexpectation satisfying this
two-level system and outputs
    $\wh\mu:=\widetilde{\E}\left[\ov{\mathsf Y}\right]$.

\paragraph{Satisfiability: the double truncation structure and parameter \(\tau\)}
Now, we discuss that, with high probability, there exists an assignment to all SoS variables satisfying crude sample-level refinement and user-level refinement simultaneously. 

The \emph{crude sample-level refinement} follows the same idea as
\Cref{sub:technique:background}. Note that even when we have only clean samples, statistically far samples can cause the empirical covariance to deviate drastically, so the constraint may fail to be satisfied even by an all-clean sample set. To avoid this issue, we consider a truncated version of the samples. Starting from the clean samples
\(z_{i,j}\in\R^d\), we truncate them
\(z_{i,j}^{\tr}\in\R^d\) by replacing \(z_{i,j}\) by \(\mu\) whenever
\(\norm{z_{i,j}-\mu}_2\) is larger than $2\sqrt{\frac{d}{\alpha}}$. This changes only an
\(O(\alpha)\)-fraction of samples, makes the empirical covariance bounded by
\(O(\mathbb{I}_d)\), and shifts the mean by only \(O(\sqrt{\alpha})\). This gives a feasible assignment
for the crude sample-level refinement by assigning the truncated samples $z_{i,j}^{\tr}$ to the SoS variables $\mathsf{Z}_{i, j}$. We assign $W_{i, j} = 1$ only when $z_{i,j}^{\tr} = x_{i, j}$. Since the
adversary also changes at most an \(\alpha\)-fraction of samples inside each good
batch ($U_i = 1$), the observed samples differ from this truncated sample set
on only \(O(\alpha)\)-fraction. Thus, Eq.~\eqref{eq:T_8} can be satisfied.

Next, we define
    $z_i^{\tr}:=\frac1n\sum_{j=1}^n z_{i,j}^{\tr}$,
and $\mu^{\tr}:=\E[z_i^{\tr}]$\footnote{Note that in this case $\beta$ is equal to $0$, hence all users have the same mean.}. 
We need to further truncate $z_i^{\tr}$'s so that we can show the satisfiability of the \emph{user-level refinement} constraints.  
As each \(z_i^{\tr}\) is an average of \(n\) truncated samples, we note that $\cov\bracket{z_i^{\tr}} \preceq (1 / n) \mathbb{I}_d$. 
Our approach is to define the second truncation, namely $z_i^{\tr, (2)}$, similar to before: when $\norm{z_i^{\tr} - \mu^{\tr}}_2$ is larger than $2\sqrt{\frac{d}{n \eps}}$, we replace with $\mu^{\tr}$. Unfortunately, using the naive bound on $\norm{z_i^{\tr} - \mu^{\tr} }_2$ as before requires $nN \geq \Omega\paren{\frac d{\eps/n}\log(d/\delta)}$ to ensure
\[
    \frac1N\sum_{i=1}^N
    (\mathsf Y_i-\overline{\mathsf Y})
    (\mathsf Y_i-\overline{\mathsf Y})^\top \preceq O\left(\frac1n\right) \mathbb{I}_d,
\]
with probability at least $1 - \delta$.
This is much larger than the desired sample complexity $nN \geq \Omega\paren{\frac d\alpha \log(d/\delta)}$ as \(\alpha \geq 1/n > \eps/n\). To improve the sample complexity, we revisit our first truncation step and observe that the truncation-induced mean shift is $O(\sqrt{\alpha})$. On the other hand, the purpose of our empirical covariance bound $O\left(\frac1n\right) \mathbb{I}_d$ is to ensure that the batch-level strong contamination parameter $\eps$ has less influence on the error $O\paren{\sqrt{\frac{\eps}{n}}}$. If we must include an $O(\sqrt{\alpha})$ term in our error guarantee, then it is acceptable to slightly weaken the empirical covariance bound to $O\left(\frac1n+\tau\right)$, where $\tau=\alpha/\eps$, in exchange for improved sample complexity.
This relaxation gives us our desired
sample complexity
    $nN \geq \Omega\paren{\frac d\alpha \log(d/\delta)}$.
Below, we explain why this choice of $\tau$ does not influence the error $O\paren{\sqrt{\alpha}}$.

\paragraph{Identifiability: controlling two sources of error}

The identifiability proof shows that every feasible pseudoexpectation must
output a mean \(\widetilde{\E}[\ov{\mathsf Y}]\) close to the true mean \(\mu\).
The proof is algebraic in nature, but its main idea can be understood as
separately controlling the two ways in which the SoS solution may differ from
the double-truncated clean reference \(\ov z^{\tr,(2)}\eqdef \frac{1}{N}\sum_{i}z_{i}^{\tr, (2)}\).

First, the user-level refinement controls the effect of fully corrupted users
and unselected users. Indeed, after the second truncation, the clean
user-level reference samples have empirical covariance at scale
\(1/n+\tau\), with $\tau = \alpha / \eps$. Following the standard SoS approach, we can get
\[
    O\left(\sqrt{\eps\left(\frac1n+\tau\right)}\right)
    =
    O\left(\sqrt{\frac{\eps}{n}}+\sqrt{\eps\tau}\right) = O\paren{\sqrt{\alpha}},
\]
so $\tau$ does not introduce any error larger than the unavoidable \(O(\sqrt{\alpha})\) term.

Second, the crude sample-level refinement controls the additional error caused
by sample-level corruptions that survive inside selected users. This is the
main technical obstacle that does not appear in prior one-level SoS analyses.
In those settings, the identifiability argument only needs to control the
effect of selected versus unselected samples, or corrupted versus uncorrupted
samples. Here, however, the difficulty comes from a two-level
interaction: sample-level corrupted samples may remain inside batches that are
selected at the user level. Such errors cannot be controlled by a purely
user-level covariance argument, because even one corrupted sample inside a
selected batch can arbitrarily shift the cleaned batch mean \(\mathsf Z_i\). The global sample-count constraint (Eq.~\eqref{eq:T_8}) ensures that, across all users selected at the user level, only an $O(\alpha)$-fraction of sample–user pairs can remain uncertified or disagree with the truncated clean references. After a degree-6 SoS Cauchy--Schwarz step, the contribution of these
surviving corrupted samples is bounded by a product of two factors: one factor
counts the surviving corrupted sample mass and is \(O(\alpha)\), while the
other factor is controlled by the sample-level covariance constraint together
with bridge relations between \(\ov{\mathsf Z}\), \(\ov{\mathsf Y}\), and the
double-truncated reference mean. This yields an additional
\(O(\sqrt{\alpha})\) error.

Combining the user-level and sample-level bounds gives
    $\norm{\widetilde{\E}[\ov{\mathsf Y}]-\ov z^{\tr,(2)}}_2
    =
    O\left(
        \sqrt{\frac{\eps}{n}}+\sqrt{\alpha}
    \right)$.
Finally, we decompose the error to the true mean $\norm{\widetilde{\E}[\ov{\mathsf Y}]-\mu}_2$:
\[
    \norm{\widetilde{\E}[\ov{\mathsf Y}]-\ov z^{\tr,(2)}}_2
    +
    \norm{\ov z^{\tr,(2)}-\mu^{\tr,(2)}}_2
    +
    \norm{\mu^{\tr,(2)}-\mu}_2 .
\]
The first term is controlled by the two refinement arguments above. The second
term is the statistical fluctuation of the empirical mean from the true mean, which
is \(O(\sqrt{d\log(1/\delta)/(Nn)})\). The third term is the truncation bias, which is
\(O(\sqrt{\alpha})\) from the satisfiability analysis. Hence, under the target sample complexity \(nN\ge \widetilde{\Omega}(d/\alpha)\) and
\(\alpha n\ge 1\), the final error guarantee becomes
    $\norm{\widetilde{\E}[\ov{\mathsf Y}]-\mu}_2
    =
    O(\sqrt{\alpha})$.
For more details, please refer to \Cref{sec:arbitrary_alpha}.

\subsection{Extending Our Double Corruption Techniques to Heterogeneous Setting}
\label{sub:technique:ques2}

We now consider the full setting $\alpha,\eps,\beta > 0$. Besides the double corruption from \Cref{sub:technique:ques1}, the clean samples of user \(i\) are drawn from a distribution \(P_i\) with mean \(\mu_i\), where $\norm{\mu_i-\mu}_2\le \sqrt{\beta}$.

In this setting, we use a similar \emph{crude sample-level refinement} and \emph{user-level refinement} structure as in \Cref{sub:technique:ques1} with three major updates. First, since each batch $i$ has a slightly different mean $\mu_i$, truncating around $\mu$ might not effectively remove the samples that cause large empirical covariance. In this situation, the empirical mean of $n$ clean samples from the batch $i$ satisfy $\E\bracket{z_{i, j}} = \mu_i$. Thus, to get the desired empirical covariance bounds for crude sample-level and user-level refinements, we now truncate each sample $z_{i, j}$ around $\mu_i$, denoted as $z_{i, j}^{\tr}$. Similarly, for the second truncation step, we truncate $z_{i}^{\tr} = \frac{1}{n}\sum_{j = 1}^n z_{i, j}^{\tr}$ around $\mu_i^{\tr} = \E\bracket{z_{i, j}^{\tr}}$.

Second, with a small constant $c_\rho > 0$, we introduce more parameters \(\eta := \max\{\alpha,\beta\},\rho := \min\{\max\{\eps,n\eta\},c_\rho\},\tau := \frac{\eta}{\rho}\) to improve the sample complexity, without hurting the error.
The goal of $\eta,\rho$ is to relax the polynomial system by allowing it to select fewer
$\mathsf Z_{i, j}$ (with $\mathsf Z_{i, j} = x_{i, j}$, Eq.~\eqref{eq:T_3}) through \(W_{i,j}\) and fewer $\mathsf Y_i$ (with $\mathsf Y_i = \mathsf Z_{i}$, Eq.~\eqref{eq:t8}) through \(U_i\).
Specifically, we respectively replace Eq.~\eqref{eq:T_8} and Eq.~\eqref{eq:t7} by 
\begin{align*}
    \sum_{i=1}^N\sum_{j=1}^n U_iW_{i,j}
    &\ge (1-4\eta)n\sum_{i=1}^N U_i, \text{ and,} \\
    \sum_{i=1}^N U_i &\ge (1-2\rho)N,
\end{align*}
where $\eta, \rho$ are taken to be larger than the corresponding values $\alpha, \eps$ in Eq.~\eqref{eq:T_8} and Eq.~\eqref{eq:t7}. Therefore, with these weaker requirements, we could truncate more samples $z_{i, j}$ (and empirical mean of the samples after the first truncation $z_i^{\tr}$) around their respective means and assign them to $\mathsf Z_{i, j}$ and $\mathsf Y_i$. This will remove more ``extreme'' $z_{i, j}$ and $z_i^{\tr}$ and make them more likely to satisfy the bounded covariance requirements. In addition, we let
    $\tau= \frac{\eta}{\rho}$ be the empirical covariance bound relaxation parameter in user-level refinement to make empirical covariance bound even easier to be satisfied. The errors caused by the mean shift are $O\paren{\sqrt{\eta}}$ and $O\paren{\sqrt{\rho / n}}$, both of which, by their definition, are upper bounded by $O\paren{\sqrt{\alpha + \beta}}$.

Third, regarding the identifiability error, we combine the same procedure from
\Cref{sub:technique:ques1} with the polynomial-constraint updates discussed
above. We first bound the \(\ell_2\) error between the output of our algorithm
\(\widetilde{\E}[\overline{\mathsf Y}]\) and the doubly truncated empirical
mean
    $\overline z^{\tr,(2)}
    :=
    \frac1N\sum_{i=1}^N z_i^{\tr,(2)}$:
\begin{align*}
    \left\|
        \widetilde{\E}[\overline{\mathsf Y}]
        -
        \overline z^{\tr,(2)}
    \right\|_2
    \le & ~
    O\left(\sqrt{\rho(1/n+\beta+\tau)}+\sqrt{\eta}\right) \\
    = & ~
    O\left(\sqrt{\eps/n}+\sqrt{\eta}\right).
\end{align*}
The statistical fluctuation
\(\|\overline z^{\tr,(2)}-\overline\mu^{\tr,(2)}\|_2\) and the second
truncation bias
\(\|\overline\mu^{\tr,(2)}-\overline\mu^{\tr}\|_2\) are bounded by
\(O(\sqrt{d\log(1/\delta)/(nN)})\) and \(O(\sqrt{\eta})\), respectively. Finally, the
heterogeneity enters through
\begin{align*}
    \|\mu_i^{\tr}-\mu\|_2
    \le 
    \|\mu_i^{\tr}-\mu_i\|_2+\|\mu_i-\mu\|_2 
    \le 
    O(\sqrt{\alpha + \beta})
\end{align*}
Averaging over \(i\) gives $\|\overline\mu^{\tr}-\mu\|_2\le
    O(\sqrt{\alpha + \beta})$.

For more details, please refer to \Cref{sec:combined}. In addition, in \Cref{sec:bounded_mean}, we provide a warm-up in which we use SoS to analyze heterogeneity with one level of strong contamination ($\eps,\beta>0$ and $\alpha=0$).

\section*{Acknowledgment}

Maryam Aliakbarpour is affiliated with the Ken Kennedy Institute at Rice University.

%% file: preliminaries.tex
\paragraph{Roadmap}
The remainder of this paper is organized as follows. In \Cref{sec:additional_preli}, we present the technical preliminaries. 
In \Cref{sec:rme_SoS}, we revisit the SoS for the standard robust mean estimation $\eps=0, \alpha>0, \beta=0$.
In \Cref{sec:bounded_mean}, we study a special case of $\eps>0, \alpha=0, \beta>0$ (heterogeneous and batch corruption) as a warm-up. In \Cref{sec:arbitrary_alpha}, we provide a complete proof to the double corruption model (\cref{ques:double}).
\cref{sec:combined} contains the proof for our main \cref{prob:heterogeneous_two_level}.
Finally, in \Cref{sec:hard}, we formally prove our lower bound result.

\section{Preliminaries}
\label{sec:additional_preli}

In this section, we introduce the technical preliminaries.
In \Cref{sec:prelim_sos_app}, we give a brief overview of the sum-of-squares proof system. In \Cref{sub:preli:expectation}, we introduce the mathematical properties of pseudo-expectation. In \Cref{sub:preli:probability}, we present the probabilistic and algebraic tools used in this paper.

\paragraph{Notation}
We define $[N] := \{1, 2, \dots, N\}$ and $\Z_+$ as the set of positive integers. For all set $A$, we denote $\ov A$ as its complement.
For a vector $x \in \R^d$, we write $\|x\|_2 := \sqrt{\sum_{i \in [d]} x_i^2}$ for its $\ell_2$ norm. For all vectors $x, y \in \R^d$, we denote $\langle x, y \rangle := \sum_{i \in [d]} x_i \cdot y_i$ as the inner product of $x$ and $y$.
For a matrix $A \in \R^{d \times d}$, $\|A\|_2 := \sup_{\|x\|_2 = 1} \|Ax\|_2$ denotes its spectral norm, and $\|A\|_F := \sqrt{\sum_{i,j} A_{i, j}^2}$ denotes its Frobenius norm.
We use $\mathbf{0}_{d \times d} \in \R^{d \times d}$ as a matrix with all 0 entries. 
We let $\mathrm{tr}\bracket{A} := \sum_{i \in [d]} A_{i, i}$ be the trace of $A$.
For matrices $A, B \in \R^{d \times d}$, we write $A \succeq B$ to denote that $A - B$ is positive semidefinite, namely for all $x \in \R^d$, we have $x^\top \paren{A - B} x \geq 0$. For an arbitrary matrix $C \in \R^{n \times d}$, we let $C^\top$ be its transpose. 
For a probability space $(\Omega, \mathcal{F}, \Pr)$, $\Omega$ is the sample space, $\mathcal{F}$ is the $\sigma$-algebra of events, and $\Pr: \mathcal{F} \to [0, 1]$ is the probability measure.
For a random variable $X: \Omega \to \R^d$, $\E[X] \in \R^d$ denotes the expectation and $\mathrm{Cov}[X] := \E[(X - \E[X])(X - \E[X])^\top] \in \R^{d \times d}$ denotes the covariance matrix. 
For a random variable $Y: \Omega \to \R$, we let $\mathrm{Var}[Y] := \E[(Y - \E[Y])^2] \in \R$ to denote its variance.
For all $A \in \mathcal{F}$, we define the indicator function $\indic{A}: \Omega \to \{0, 1\}$ as $\indic{A}(\omega) := 1$ if $\omega \in A$ and $\indic{A}(\omega) := 0$ if $\omega \notin A$, for all $\omega \in \Omega$.
The identity matrix in $\R^{d \times d}$ is denoted by $\mathbb{I}_d$. 
We use $\mathsf{B} \sststile{k}{x}\{p(x) \geq q(x)\}$ to represent that there exists a degree-$k$ sum-of-squares proof that the polynomial inequality $p(x) \geq q(x)$ holds, with variable $x$, by assuming the polynomial constraints in the set $\mathsf{B}$. We define $\mathbb{R}[x_1, \dots, x_d]$ as the polynomial ring over the real variables $x_1, \dots, x_d \in \R$.

\subsection{Sum-of-Squares (SoS)}
\label{sec:prelim_sos_app}

Our algorithm relies on the sum-of-squares algorithm~\citep{nesterov2000squared,lasserre2001new}. We describe the results needed for our algorithm and refer the readers to \cite{hopkins2017SoS, kothari2018robust, BakshiDHKKK20robustSoS} for more detailed exposition of the method. The formal definition of the sum-of-squares (SoS) proof is as follows:

\begin{definition}[Sum-of-squares (SoS) proof]\label{def:sos_proof}
    For polynomials $p,q$ in $x\in \R^\dims$, we say that $p\ge q$ has a degree-$k$ sum-of-squares proof if there exists polynomials $s_1, \ldots, s_t$ with degree at most $k/2$ such that
    \[
    p-q=\sum_{i=1}^ts_i^2.
    \]
    We denote this as $\sststile{k}{x}\{p\ge q\}$.
    Given a set of polynomial constraints $\cA=\{f_i=0\}_{i\in [m]}\cup \{g_j\ge0\}_{j\in[n]}$, we say that there exists a degree-$k$ proof of $p\ge q$ modulo $\cA$ if there exists polynomials $a_i,b_j,s_t$ such that $\deg(a_if_i)\le k$, $\deg(b_j^2g_j)\le k$, and $\deg(s_t^2)\le k$ for all $i,j,t$, and
    \[
    p-q=\sum_{t}s_t^2+\sum_{i=1}^m a_if_i+\sum_{j=1}^{n}b_j^2g_j.
    \]
    This can be denoted as $\cA\sststile{k}{x}\{p\ge q\}$.
\end{definition}

Pseudo-distribution is a generalization of probability distributions in that their ``probability'' weight can be negative and ``expectations'' are only guaranteed to be non-negative for square polynomials of finite degree.
\begin{definition}
Let $\cX\subseteq \R^\dims$ be a finite set of vectors. A level-$\ell$ (or degree-$\ell$) pseudo-distribution over $\cX$ has a ``mass function'' $D(x)$ such that for all polynomial $f$ such that $\deg(f)\le \ell/2$.
\[
\sum_{x\in \cX}D(x)=1,\quad \sum_{x\in \cX}D(x)f(x)^2\ge 0,
\]
and we define the pseudo-expectation for $f(x)$ as 
\[
\wt\E_D\bracket{f(x)}\eqdef\sum_{x}D(x)f(x).
\]
\end{definition}

This relaxation is crucial for the Sum-of-Squares framework: rather than searching over discrete or combinatorial objects directly, we optimize over pseudo-expectations that behave like expectations on all low-degree polynomials. In particular, any feasible SoS solution corresponds to a pseudo-distribution whose pseudo-expectation operator satisfies basic probabilistic inequalities (such as Cauchy–Schwarz) up to the prescribed degree.

\begin{definition}[Polynomial constraints and satisfiability]
    Given a set of polynomial constraints $\cA=\{f_1\ge0, \ldots f_m\ge 0\}$, we say that $D$ satisfies $\cA$ at degree $r$ if for every $S\subseteq[m]$ and every sum-of-squares polynomial $h$ with $\deg(h)=\sum_{i\in S}\max\{\deg(f_i),r\}$, 
\[
\wt\E_D\bracket{h(x)\prod_{i\in S}f_i(x)}\ge 0.
\]
We say that $D$ satisfies $\cA$ approximately at degree $r$ if the above inequalities are satisfied up to an error of $2^{-\dims \ell}\norm{h}_2\prod_{i\in S}\norm{f_i}_2$, where $\norm{\cdot}_2$ is the 2-norm of the coefficients of each monomial.
\end{definition}

The following inference rules are helpful to establish new sum-of-square proofs.

\begin{fact}[SoS Cauchy-Schwarz, Claim 2.6 in \cite{Schramm2022ProofsToAlgs}]\label{fac:cauchy_schwarz}
Let $a, b$ be vector-valued polynomials of degree at most $k$. Then for any $\varepsilon > 0$,
\[
\vdash_{2k} \cbracket{\langle a, b \rangle \leq \frac{\varepsilon}{2} \norm{a}_2^2 + \frac{1}{2\varepsilon} \norm{b}_2^2}
\]
and
\[
\vdash_{4k} \cbracket{\langle a, b \rangle^2 \leq \norm{a}_2^2 \norm{b}_2^2}.
\]
\end{fact}

\begin{fact}[SoS operator norm, Claim 2.7 in \cite{Schramm2022ProofsToAlgs}]\label{fac:sos_operator_norm}
Let $y \in \mathbb{R}^n$, $M \in \mathbb{R}^{n \times n}$, and $B \in \mathbb{R}^{n \times k}$. Then
\[
\{ M = \lambda I - BB^\top \} \vdash_k \{y^\top M y \le \lambda \|y\|_2^2\},
\]
for
\[
k \ge \deg\bigl( y^\top M y + y^\top BB^\top y \bigr).
\]
\end{fact}

\begin{fact}\label{fac:xyz_square}
    Let $x, y, z \in \R$. Then, we have
    \begin{align*}
        \vdash_2\left\{ 3 \paren{x^2 + y^2 + z^2} \geq \paren{x + y + z}^2 \right\}
    \end{align*}
\end{fact}
\begin{proof}
    We have
    \begin{align*}
        3 \paren{x^2 + y^2 + z^2} - \paren{x + y + z}^2
        = & ~ 3x^2 + 3y^2 + 3z^2 - x^2 - y^2 - z^2 - 2xy - 2xz - 2yz\\
        = & ~ 2x^2 + 2y^2 + 2z^2 - 2xy - 2xz - 2yz\\
        = & ~ \paren{x - y}^2 + \paren{x - z}^2 + \paren{y - z}^2,
    \end{align*}
    which completes the proof by Definition~\ref{def:sos_proof}.
\end{proof}

\subsection{Pseudo-Expectation Properties}
\label{sub:preli:expectation}

We discuss some useful properties of \emph{pseudo-expectation} operators.

\begin{fact}[Linearity of Pseudoexpectation]\label{fac:linearity}
For any polynomials $f, g$ and scalars $a, b \in \mathbb{R}$, the pseudoexpectation operator $\wt{\mathbb{E}}$ is linear:
\begin{align*}
    \wt{\mathbb{E}}[a f + b g] = a \wt{\mathbb{E}}[f] + b \wt{\mathbb{E}}[g].
\end{align*}
\end{fact}

In addition, we can use Cauchy Schwarz inequality in pseudoexpectation.

\begin{fact}[Cauchy Schwarz for Pseudoexpectation, Fact 3.10 in \cite{bk20}]\label{fact:cs_pseudo}
Let $f, g$ be polynomials of degree at most $d$.
Then, we have
\begin{align*}
    \wt \E \bracket{f g} \leq \sqrt{\wt \E \bracket{f^2}} \sqrt{\wt \E\bracket{g^2}}.
\end{align*}
\end{fact}

\begin{fact}\label{fac:pseudoexpectation}
    We have the following facts:
    \begin{itemize}
        \item \textbf{Part 1.} If $X, \wt \E \bracket{X} \in \R$, then we can get
        \begin{align*}
            0 \leq \wt \E\bracket{\paren{X - \wt \E\bracket{X}}^2} = \wt \E\bracket{X^2} - \wt \E\bracket{X}^2.
        \end{align*}
        \item \textbf{Part 2.} If $X = (x_1, \dots, x_d)$ denotes the tuple of indeterminates, where each $x_i \in \R[x_1, \dots, x_d]$ and $\wt \E \bracket{X} = \paren{\wt \E \bracket{x_1}, \dots, \wt \E \bracket{x_d}}$, then we can get
        \begin{align*}
            \wt \E \bracket{\norm{X - \wt \E \bracket{X}}_2^2} = \wt \E \bracket{\norm{X}_2^2}- \norm{\wt \E \bracket{X}}_2^2.
        \end{align*}
        \item \textbf{Part 3.} Let $x \in \R^d$ be a given vector. If $X = (x_1, \dots, x_d)$ denotes the tuple of indeterminates, where each $x_i \in \mathbb{R}[x_1, \dots, x_d]$ and $\wt \E \bracket{X} = \paren{\wt \E \bracket{x_1}, \dots, \wt \E \bracket{x_d}}$, then we can get
        \begin{align*}
            \wt \E\bracket{\norm{x - X}_2^2} \geq \norm{x - \wt \E[X]}_2^2
        \end{align*}
    \end{itemize}
\end{fact}
\begin{proof}
\textbf{Proof of Part 1.}
First, we consider the case where $X, \wt \E \bracket{X} \in \R$.
We have
\begin{align*}
    \wt \E\bracket{\paren{X - \wt \E\bracket{X}}^2}
    = & ~ \wt \E\bracket{X^2 + \wt \E\bracket{X}^2 - 2X \wt \E\bracket{X}}\\
    = & ~ \wt \E\bracket{X^2} + \wt \E\bracket{X}^2 - 2\wt \E\bracket{X} \wt \E\bracket{X}\\
    = & ~ \wt \E\bracket{X^2} - \wt \E\bracket{X}^2,
\end{align*}
where the second step follows from the linearity of pseudo-expectation (see Fact~\ref{fac:linearity}).

\textbf{Proof of Part 2.}
Second, we consider the case where $X, \wt \E \bracket{X} \in \R^d$.
Note that 
\begin{align}\label{eq:variance_ii}
    \paren{\paren{X-\wt \E \bracket{X}} \paren{X-\wt \E \bracket{X}}^\top}_{i, i} = \paren{X-\wt \E \bracket{X}}_i \cdot \paren{X-\wt \E \bracket{X}}_i = \paren{X_i - \wt \E \bracket{X_i}}^2
\end{align}
Then, we have
\begin{align*}
\wt \E \bracket{\norm{X - \wt \E \bracket{X}}_2^2}
= & ~ \wt \E \bracket{\left \langle X-\wt \E \bracket{X}, X-\wt \E \bracket{X} \right \rangle} \\
= & ~ \mathrm{tr} \bracket{\wt \E \bracket{\paren{X-\wt \E \bracket{X}} \paren{X-\wt \E \bracket{X}}^\top}} \\
= & ~ \sum_{i=1}^d \paren{\wt \E\bracket{X_i^2} - \wt \E\bracket{X_i}^2} \\
= & ~ \wt \E\bracket{\sum_{i=1}^d X_i^2} - \sum_{i=1}^d \wt \E\bracket{X_i}^2 \\
= & ~ \wt \E \bracket{\norm{X}_2^2}- \norm{\wt \E \bracket{X}}_2^2,
\end{align*}
where third step follows from \textbf{Part 1} and Eq.~\eqref{eq:variance_ii}, the fourth step follows from linearity (see Fact~\ref{fac:linearity}), and the last step follows from the definition of the $\ell_2$ norm.

\textbf{Proof of Part 3.}
Furthermore, we have
\begin{align*}
    \wt \E\bracket{\norm{x - X}_2^2}
    = & ~ \wt \E\bracket{\norm{X}_2^2 + \norm{x}_2^2 - 2\langle x, X \rangle} \\
    = & ~ \wt \E\bracket{\norm{X}_2^2} + \norm{x}_2^2 - 2\left\langle x, \wt \E\bracket{X} \right\rangle \\
    = & ~ \paren{\wt \E\bracket{\norm{X}_2^2} - \norm{\wt \E\bracket{X}}_2^2} + \paren{\norm{\wt \E\bracket{X}}_2^2 + \norm{x}_2^2 - 2\left\langle x, \wt \E\bracket{X} \right\rangle} \\
    = & ~ \wt \E\bracket{\norm{X - \wt \E\bracket{X}}_2^2} + \norm{x - \wt \E[X]}_2^2 \\
    \geq & ~ \norm{x - \wt \E[X]}_2^2,
\end{align*}
where the second step follows from the linearity of the pseudo-expectation (see Fact~\ref{fac:linearity}), the fourth step follows from \textbf{Part 2} of Fact~\ref{fac:pseudoexpectation}, and the last step follows from the definition of pseudo-expectation that $\wt \E [X^2] \geq 0$.
\end{proof}

\subsection{Probabilistic and Algebraic Tools}
\label{sub:preli:probability}

In this section, we present basic probabilistic and algebraic tools.

\begin{fact}[Union Bound]\label{fac:union}
Let $A_1, A_2, \dots, A_m$ be events in a probability space. Then
\[
\Pr\left[\bigcup_{i=1}^m A_i\right] \leq \sum_{i=1}^m \Pr[A_i].
\]
\end{fact}

\begin{fact}[Markov's Inequality]\label{fac:markov}
Let $X$ be a non-negative random variable and let $a > 0$. Then
\[
\Pr[X \geq a] \leq \frac{\mathbb{E}[X]}{a}.
\]
\end{fact}

\begin{fact}[Multiplicative Chernoff Bound]\label{fac:chernoff}
Let $X_1, X_2, \dots, X_n$ be independent random variables taking values in $[0,1]$, and let $X = \sum_{i=1}^n X_i$ with $\mu = \mathbb{E}[X]$. Then for all $\Delta > 0$, the following holds:
\begin{align*}
    \Pr[X \geq (1+\Delta)\mu] &\leq \exp\left( -\frac{\Delta^2 \mu}{3} \right) \quad \text{for } 0 < \Delta \leq 1, \\
    \Pr[X \geq (1+\Delta)\mu] &\leq \exp\left( -\frac{\Delta \mu}{3} \right) \quad \text{for } \Delta > 1.
\end{align*}
\end{fact}

\begin{fact}\label{fac:spectral_ell_2}
Let $u \in \R^d$. Then
\begin{align*}
    \norm{u u^\top}_2 = \norm{u}_2^2.
\end{align*}
\end{fact}

\begin{proof}
If $u = \mathbf{0}_d$, then $u u^\top = 0$, so
    $\norm{u u^\top}_2 = 0 = \norm{u}_2^2$.
Thus, assume $u \neq \mathbf{0}_d$.
For all $v \in \R^d$ with $\norm{v}_2 = 1$, we have
\begin{align*}
    \norm{u u^\top v}_2
    = & ~ \norm{u (u^\top v)}_2 \\
    = & ~ |u^\top v| \cdot \norm{u}_2 \\
    \leq & ~ \norm{u}_2 \norm{v}_2 \norm{u}_2 \\
    = & ~ \norm{u}_2^2,
\end{align*}
where the inequality follows from the Cauchy--Schwarz inequality. Therefore,
\begin{align*}
    \norm{u u^\top}_2 \leq \norm{u}_2^2.
\end{align*}

On the other hand, taking $v = u / \norm{u}_2$, we have $\norm{v}_2 = 1$ and
\begin{align*}
    \norm{u u^\top v}_2
    = & ~ \norm{u u^\top \frac{u}{\norm{u}_2}}_2 \\
    = & ~ \norm{u \cdot \norm{u}_2}_2 \\
    = & ~ \norm{u}_2^2.
\end{align*}
Hence,
\begin{align*}
    \norm{u u^\top}_2 = \sup_{\norm{v}_2 = 1} \norm{u u^\top v}_2 \geq \norm{u}_2^2.
\end{align*}
Combining the two inequalities completes the proof.
\end{proof}

\begin{fact}[Theorem 5 in \cite{mkmp25}]\label{fac:vec_bernstein}
Let $X_1,\ldots,X_m$ be independent mean-zero random vectors in $\R^d$.
Suppose that $\|X_i\|_2\le c$ almost surely for every $i\in[m]$. Then, for
every $\delta\in(0,1)$, with probability at least $1-\delta$,
\begin{align*}
    \norm{\sum_{i=1}^m X_i}_2 \leq \sqrt{\sum_{i=1}^m \E\bracket{\norm{X_i}_2^2}} \left(1+\sqrt{2\log(1/\delta)}\right) + \frac{4c}{3}\log(1/\delta).
\end{align*}
\end{fact}

\begin{fact}[Matrix Chernoff Inequality, Theorem 3.12 from \cite{diakonikolas2023algorithmic}]\label{fac:matrix_Chernoff}
Let $R > 0$. Let $\{X_k\}$ be a sequence of independent random $d \times d$ symmetric matrices with 
\[
0 \preceq X_k \preceq R \cdot \mathbb{I}_d
\]
almost surely. Let $\lambda_{\max} = \norm{\mathbb{E}\bracket{\sum_k X_k}}_2$.
Then, for any $\Delta > 0$, we have that
\[
\Pr\bracket{ \norm{ \sum_k X_k }_2 > (1+\Delta)\lambda_{\max}}
\leq d \paren{ \frac{e^\Delta}{(1+\Delta)^{1+\Delta}} }^{\lambda_{\max}/R}.
\]
\end{fact}

\begin{fact}\label{fac:low_covariance}
    Let $X_1,\dots,X_n \in \mathbb{R}^d$ be \iid random variables satisfying $\mathbb{E}[X_i]=\mu \in \mathbb{R}^d$ and $\mathrm{Cov}\bracket{X_i}=\Sigma \in \mathbb{R}^{d\times d}$, for all $i \in [n]$. Let $\ov X := \frac{1}{n}\sum_{i=1}^n X_i$.
    Then, we can get
    \begin{align*}
        \mathrm{Cov}\bracket{\ov X} = \frac{1}{n} \Sigma.
    \end{align*}
\end{fact}
\begin{proof}
    We can get
    \begin{align*}
        \mathrm{Cov}\bracket{\ov X}
        = & ~ \mathrm{Cov}\bracket{\frac{1}{n}\sum_{i=1}^n X_i} \\
        = & ~ \frac{1}{n^2} \mathrm{Cov}\bracket{\sum_{i=1}^n X_i} \\
        = & ~ \frac{1}{n^2} \sum_{i=1}^n \mathrm{Cov}\bracket{X_i} \\
        = & ~ \frac{1}{n^2} \sum_{i=1}^n \Sigma \\
        = & ~ \frac{n}{n^2} \Sigma \\
        = & ~ \frac{1}{n} \Sigma.
    \end{align*}
\end{proof}

\begin{fact}\label{fac:from_mu_to_zbar}
    Let $v_1, \dots, v_N \in \R^d$ be deterministic vectors with empirical mean
    $\ov v = \frac{1}{N} \sum_{i=1}^N v_i$, and fix any $\mu \in \R^d$. Then,
    we have
    \begin{align*}
        \frac{1}{N} \sum_{i = 1}^N \paren{v_i - \ov v} \paren{v_i - \ov v}^\top
        \preceq
        \frac{1}{N} \sum_{i = 1}^N \paren{v_i - \mu} \paren{v_i - \mu}^\top.
    \end{align*}
\end{fact}
\begin{proof}
We have
\begin{align*}
(v_i-\ov v)(v_i-\ov v)^\top
= (v_i-\mu)(v_i-\mu)^\top
- (v_i-\mu)(\ov v-\mu)^\top
- (\ov v-\mu)(v_i-\mu)^\top
+ (\ov v-\mu)(\ov v-\mu)^\top.
\end{align*}
Averaging over $i=1,\dots,N$ and using
\begin{align*}
    \frac{1}{N} \sum_{i=1}^N (v_i-\mu) = \ov v-\mu,
\end{align*}
we get
\begin{align*}
\frac{1}{N} \sum_{i=1}^N (v_i-\ov v)(v_i-\ov v)^\top
= \frac{1}{N} \sum_{i=1}^N (v_i-\mu)(v_i-\mu)^\top
- (\ov v-\mu)(\ov v-\mu)^\top.
\end{align*}
Therefore,
\begin{align*}
\frac{1}{N}\sum_{i=1}^N (v_i-\ov v)(v_i-\ov v)^\top
- \frac{1}{N}\sum_{i=1}^N (v_i-\mu)(v_i-\mu)^\top
= & ~ -(\ov v-\mu)(\ov v-\mu)^\top \\
\preceq & ~ 0,
\end{align*}
which implies
\[
    \frac{1}{N} \sum_{i = 1}^N (v_i - \ov v)(v_i - \ov v)^\top \preceq \frac{1}{N} \sum_{i = 1}^N (v_i - \mu)(v_i - \mu)^\top. \qedhere
\]
\end{proof}

%% file: 0_refinement.tex
\section{Background and Refinements to Robust Mean Estimation Via SoS}
\label{sec:rme_SoS}

Before delving into our main results, we revisit the standard robust mean
estimation problem.

\begin{problem}[Robust Mean Estimation]
\label{prob:rme}
    Let $\alpha\in(0,0.5)$ and let $P$ be a distribution over $\R^d$ with mean
    $\mu$ and covariance $\Sigma\preceq \mathbb I_d$. First, $n$ clean samples
    are drawn \iid from $P$. Then an adversary inspects these samples and
    replaces at most an $\alpha$-fraction of them by arbitrary points. Given only
    the final, possibly corrupted, samples, the goal is to estimate $\mu$ in
    $\ell_2$ distance.
\end{problem}

\paragraph{Overview of SoS-Based Robust Mean Estimation}
The main idea behind high-dimensional robust mean estimation is the following.
If one can identify a large collection of points with sufficiently small
empirical covariance, then the empirical mean of this collection cannot be
pulled too far in any direction. This collection does not need to be perfectly
clean: it may contain some corrupted points, as long as these points are not
harmful enough to significantly increase the empirical covariance. This gives a
\emph{verifiable criterion}: instead of knowing which samples are clean, it is
enough to find a large set of samples satisfying a low-covariance condition.

A brute-force search over all large subsets is computationally infeasible:
checking all subsets of size $(1-\alpha)n$ requires time at least
$\binom{n}{(1-\alpha)n}$, which is exponential in $n$. The Sum-of-Squares (SoS)
framework provides a polynomial-time relaxation of this subset-selection
problem. It introduces formal variables representing a candidate cleaned dataset
and imposes polynomial constraints requiring that many candidate points agree
with the observed data and that the candidate dataset has small empirical
covariance.

Given observed, possibly corrupted samples
$\{x_j\}_{j=1}^n\subset \mathbb R^d$, the standard SoS relaxation
can be written using variables $\{\widetilde Y_j\}_{j=1}^n\subset\mathbb R^d$, $\{W_j\}_{j=1}^n\subset\mathbb R$, $B\in\mathbb R^{d\times d}$.
Here $\widetilde Y_j$ represents a candidate cleaned version of the observed
sample $x_j$, and $W_j$ is a Boolean selection variable indicating whether the
candidate cleaned sample is required to agree with the observed sample. The
polynomial constraints are
\begin{align}
    W_j^2 &= W_j, && \forall j\in[n], \label{eq:old_1}\\
    \sum_{j=1}^n W_j &\ge (1-\alpha)n, \label{eq:old_2}\\
    W_j(\widetilde Y_j-x_j) &= 0, && \forall j\in[n], \label{eq:old_3}\\
    \frac1n\sum_{j=1}^n
    (\widetilde Y_j-\widetilde Y)
    (\widetilde Y_j-\widetilde Y)^\top
    &= C\mathbb I_d - BB^\top, \label{eq:old_4}
\end{align}
where $\widetilde Y := \frac1n\sum_{j=1}^n \widetilde Y_j$.
The first constraint forces each $W_j$ to be Boolean, and the second constraint
requires that at least $(1-\alpha)n$ candidate points agree with the observed
data. The third constraint says that whenever $W_j=1$, the cleaned sample must
agree with the observed sample, i.e. $\widetilde Y_j=x_j$. The final constraint
encodes the requirement that the candidate cleaned dataset has empirical
covariance at most $C\mathbb I_d$.
This relaxation leads to two basic questions:
\begin{enumerate}
    \item With high probability over the clean samples,
    does there exist an assignment to the variables
    $\{\wt Y_j\}_{j=1}^n$, $\{W_j\}_{j=1}^n$, and $B$ satisfying all of the
    polynomial constraints?

    \item If a pseudoexpectation satisfies these
    polynomial constraints, why must the resulting estimate
    $\widetilde{\mathbb E}[\wt Y]$ be close to the true mean $\mu$?
\end{enumerate}

The SoS analysis has two corresponding parts. The first part is
\emph{satisfiability}: one shows that, with high probability over the clean
samples, there exists an assignment to the SoS variables satisfying all
polynomial constraints. The second part is \emph{identifiability}: one shows
that any pseudoexpectation satisfying the polynomial constraints produces a mean
estimate close to the clean empirical mean.

\paragraph{Prior Work and Our Refinements}
The SoS-based analysis of robust mean estimation in
Problem~\ref{prob:rme} can be traced back to
\cite{hopkins2018mixture,kothari2018robust}. These works develop SoS
relaxations for robust mean and moment estimation under bounded covariance
assumptions, as well as under stronger certifiably bounded higher-moment
assumptions. The higher-moment assumptions rule out the most problematic heavy
tails and imply that, with high probability, the uncorrupted empirical
distribution satisfies the required low-degree empirical moment constraints.
Thus, in those settings, the clean empirical distribution itself provides a
feasible assignment to the polynomial constraints, establishing satisfiability.

More concretely, if $z_j$ denotes the realized value of the clean sample
$Z_j$, then the standard non-truncated satisfiability proof assigns $z_j$ to $\widetilde Y_j$ for all $j \in [n]$, and sets $W_j=1$ on a large subset of uncorrupted indices. For such indices,
$x_j=z_j$, and therefore $\widetilde Y_j=x_j$. Hence the agreement constraint
$W_j(\widetilde Y_j-x_j)=0$ is satisfied. The remaining covariance constraint is
satisfied provided the raw clean empirical distribution has sufficiently small
empirical covariance.

We revisit the bounded-covariance setting while allowing
arbitrary distributions satisfying only $\Sigma\preceq \mathbb I_d$, including distributions that may not satisfy any bounded higher-moment
condition. For such heavy-tailed distributions, the raw clean samples need not
directly satisfy the empirical covariance constraint. 
We therefore use a
truncation argument \cite{diakonikolas2023algorithmic} to construct a large
well-behaved clean set whose empirical covariance is bounded while changing the
mean by only the target error.

Specifically, we show the following:
\begin{itemize}
    \item With high probability, truncation changes at most an $\alpha$-fraction
    of the clean samples. Equivalently, the realized truncated samples
    $\{Z_j^{\mathrm{tr}}\}_{j=1}^n$ differ from the realized clean samples
    $\{Z_j\}_{j=1}^n$ on at most $\alpha n$ indices
    (Claim~\ref{cla:covariance-condition_part_1}).

    \item The truncation shifts the population mean by at most the target error:
    if
        $\mu^{\mathrm{tr}} := \mathbb E[Z_j^{\mathrm{tr}}]$,
    then
        $\|\mu^{\mathrm{tr}}-\mu\|_2 \le O(\sqrt{\alpha})$
    (Claim~\ref{cla:covariance-condition_part_3}).

    \item The realized truncated samples satisfy the empirical covariance
    constraint required by the polynomial system
    (Claim~\ref{cla:covariance-condition_part_2_6}).
\end{itemize}
Proving these claims is the goal of \Cref{sub:rme_SoS:true}.

Consequently, the observed samples $\{x_j\}_{j\in[n]}$ can be viewed as a
$2\alpha$-strong contamination of the realized truncated clean samples: an
$\alpha$-fraction of indices may differ due to the adversary, and another
$\alpha$-fraction may differ due to truncation. Therefore, in
\Cref{sub:rme_SoS:background}, we modify the polynomial constraints accordingly.
For example, the constraint
    $\sum_{j=1}^n W_j \ge (1-\alpha)n$
is replaced by
    $\sum_{j=1}^n W_j \ge (1-2\alpha)n$.
Moreover, the SoS variables should be interpreted as candidate cleaned
\emph{truncated} samples. Thus, in our polynomial system, we reserve the variables
    $\{Y_j\}_{j=1}^n\subset\mathbb R^d$
for candidate cleaned truncated samples. In the satisfiability proof, these
variables are assigned by the realized truncated samples:
    $Y_j = Z_j^{\mathrm{tr}}$, for all $j\in[n]$.
Thus, $\widetilde Y_j$ denotes the non-truncated candidate cleaned sample in the
standard SoS relaxation, whereas $Y_j$ denotes the truncated candidate cleaned
sample used in our analysis.

The identifiability part builds upon
\cite{hopkins2018mixture,kothari2018robust}: any pseudoexpectation satisfying
the polynomial constraints must output a mean close to the mean of the feasible
clean object. In the non-truncated setting, this feasible object is the raw
clean empirical distribution. In our setting, it is the truncated clean
empirical distribution. It remains to relate the empirical mean of the truncated
clean samples back to the original population mean $\mu$. Under
$\Sigma\preceq \mathbb I_d$, the clean empirical mean has squared error at most
$d/n$ in expectation. Since truncation makes the relevant random variables
bounded, we can further use a concentration inequality \cite{mkmp25} to obtain
the desired high-probability, nearly optimal sample-complexity guarantee. This
step is carried out in \Cref{sub:rme_SoS:empirical}.

Combining the truncation-based satisfiability argument (Lemma~\ref{lem:covariance-condition}), the SoS identifiability
argument (Lemma~\ref{lem:sos_mean}), and the concentration of the truncated empirical mean (Lemma~\ref{lem:sos_mean}), we obtain the
following result:
\begin{theorem}[SoS robust mean estimation under bounded covariance]
\label{thm:rme_sos_bounded_covariance}
    Let $c\in(0,0.5)$ be a universal constant. Let $\alpha \in (0, c)$. Under the setup of Problem~\ref{prob:rme}, there
    exists a polynomial-time algorithm (Algorithm~\ref{alg:sos_rme}) that outputs $\widehat \mu\in\mathbb R^d$
    such that
        $\|\widehat\mu-\mu\|_2 = O(\sqrt{\alpha})$
    with probability at least $1-\delta$, provided
        $n = \Omega\left(\frac{d}{\alpha}\log\frac d\delta\right)$.
    The error rate is minimax optimal, and the sample complexity is nearly
    optimal up to logarithmic factors.
\end{theorem}

\subsection{Background and Refinements to Truncation and Satisfiability Tools}
\label{sub:rme_SoS:true}

In this section, we adapt the truncation argument from
\cite{diakonikolas2023algorithmic}. Specifically, in
Claim~\ref{cla:covariance-condition_part_1},
Claim~\ref{cla:covariance-condition_part_3}, and
Claim~\ref{cla:covariance-condition_part_2_6}, we show that modifying at most
an $\alpha$-fraction of the clean samples by truncation is sufficient to ensure
the desired bounded empirical covariance condition, while changing the
population mean by only the target error.

We first explain the random variables used in the truncation argument. The
vectors $Z_j$ are the clean samples that would have been observed without
adversarial corruption. They are not available to the algorithm, but they exist
in the probabilistic model and are used only in the analysis. Since bounded
covariance does not imply that the empirical covariance of the raw clean samples
is bounded with high probability, we introduce truncated clean samples
$Z_j^{\mathrm{tr}}$. These truncated samples differ from $Z_j$ on only a small
fraction of indices (Claim~\ref{cla:covariance-condition_part_1}), their mean remains close to $\mu$ (Claim~\ref{cla:covariance-condition_part_3}), and their empirical
covariance is bounded with high probability (Claim~\ref{cla:covariance-condition_part_2_6}).

\begin{claim}[Few samples are modified by truncation]
\label{cla:covariance-condition_part_1}
Let $\gamma > 0$ be the scaling factor of the covariance bound\footnote{In the setting of Problem~\ref{prob:rme}, we choose $\gamma = 1$.}. Let $Z_1,\ldots,Z_n\in\R^d$ be \iid random vectors with common mean
$\mu=\E[Z_j]$ and covariance $\cov[Z_j]\preceq \gamma \mathbb I_d$. Define
\begin{align*}
    Z_j^{\mathrm{tr}}
    :=
    \begin{cases}
    Z_j, & \norm{Z_j-\mu}_2\le 2\sqrt{\gamma d/\alpha},\\
    \mu, & \text{otherwise.}
    \end{cases}
\end{align*}
Then, we have
\begin{align*}
    \Pr\left[\frac1n\sum_{j=1}^n \indic{Z_j\ne Z_j^{\mathrm{tr}}}\ge \alpha\right]
    \le \exp\left(-\frac{\alpha n}{4}\right).
\end{align*}
\end{claim}

\begin{proof}
Note that since $\mu=\E[Z_j]$, we have 
\begin{align}\label{eq:cov_X_4}
    \mathrm{Cov}\bracket{Z_j} = \E\bracket{\paren{Z_j - \mu} \paren{Z_j - \mu}^\top} \preceq \gamma \mathbb{I}_d,
\end{align}
so we can further get
\begin{align}\label{eq:expectation_x_mu_4}
    \E\bracket{\norm{Z_j - \mu}_2^2}
    = & ~ \E\bracket{\left \langle Z_j - \mu, Z_j - \mu \right \rangle} \notag\\
    = & ~ \mathrm{tr}\bracket{\E\bracket{\paren{Z_j - \mu} \paren{Z_j - \mu}^\top}} \notag\\
    \leq & ~ \mathrm{tr}\bracket{\gamma \mathbb{I}_d} \notag\\
    = & ~ \gamma d,
\end{align}
where the first step follows from the definition of the $\ell_2$ norm and the third step follows from Eq.~\eqref{eq:cov_X_4}.
Therefore, we can see that
\begin{align*}
    \Pr\bracket{Z_j^{\mathrm{tr}} \neq Z_j} 
    = & ~ \Pr\bracket{\norm{Z_j - \mu}_2^2 > 4 \frac{\gamma d}{\alpha}}\\
    \leq & ~ \frac{\E\bracket{\norm{Z_j - \mu}_2^2}}{4 \gamma d / \alpha} \\
    \leq & ~ \frac{\alpha}{4},
\end{align*}
where the first step follows from the definition of $Z_j^{\mathrm{tr}}$, the second step follows from the Markov inequality (see Fact~\ref{fac:markov}), and the third step follows from Eq.~\eqref{eq:expectation_x_mu_4}.
For all $j \in [n]$, $\indic{Z_j \neq Z_j^{\mathrm{tr}}}$ are \iid Bernoulli with $\E\bracket{\indic{Z_j \neq Z_j^{\mathrm{tr}}}} \leq \frac{\alpha}{4}$. Therefore, we have
\begin{align}\label{eq:bound_np_4}
    \E\bracket{\sum_{j=1}^n \indic{Z_j \neq Z_j^{\mathrm{tr}}}} \leq \frac{\alpha}{4} n.
\end{align}
If $\E\bracket{\sum_{j=1}^n \indic{Z_j \neq Z_j^{\mathrm{tr}}}} = 0$, then the proof is
immediate. If $\E\bracket{\sum_{j=1}^n \indic{Z_j \neq Z_j^{\mathrm{tr}}}} \neq 0$, then we define 
\begin{align*}
    \Delta 
    := & ~ \frac{\alpha n}{\E\bracket{\sum_{j=1}^n \indic{Z_j \neq Z_j^{\mathrm{tr}}}}} - 1.
\end{align*}
Using the multiplicative Chernoff bound (see Fact~\ref{fac:chernoff}), we have
\begin{align*}
    \Pr\bracket{\frac{1}{n} \sum_{j = 1}^n \indic{Z_j \neq Z_j^{\mathrm{tr}}} \geq \alpha}
    = & ~ \Pr\bracket{\sum_{j = 1}^n \indic{Z_j \neq Z_j^{\mathrm{tr}}} \geq n\alpha}\\
    = & ~ \Pr\bracket{\sum_{j = 1}^n \indic{Z_j \neq Z_j^{\mathrm{tr}}} \geq \paren{1 + \Delta}\E\bracket{\sum_{j = 1}^n \indic{Z_j \neq Z_j^{\mathrm{tr}}}}}\\
    \leq & ~ \exp \paren{- \frac{\Delta \E\bracket{\sum_{j = 1}^n \indic{Z_j \neq Z_j^{\mathrm{tr}}}}}{3}}\\
    = & ~ \exp \paren{- \frac{\alpha n - \E\bracket{\sum_{j = 1}^n \indic{Z_j \neq Z_j^{\mathrm{tr}}}}}{3}}\\
    \leq & ~ \exp \paren{- \frac{\alpha n}{4}}
\end{align*}
where the second and the last steps follow from Eq.~\eqref{eq:bound_np_4}.
\end{proof}

\begin{claim}[The truncation-induced mean shift is small]
\label{cla:covariance-condition_part_3}
Under the notation of Claim~\ref{cla:covariance-condition_part_1}, let
$\mu^{\mathrm{tr}}:=\E[Z_j^{\mathrm{tr}}]$. Then, we have
\[
    \norm{\mu^{\mathrm{tr}}-\mu}_2
    \le O(\sqrt{\gamma\alpha}).
\]
\end{claim}

\begin{proof}
Let $F_j := \{Z_j=Z_j^{\mathrm{tr}}\}$.
On $F_j$, truncation does not change the sample, while on $\overline F_j$ we
have $Z_j^{\mathrm{tr}}=\mu$. We define $b:=\mu^{\mathrm{tr}}-\mu$.
Since $\mu=\E[Z_j]$, we have
\begin{align*}
    b
    = & ~
    \E\bracket{Z_j^{\mathrm{tr}}} - \E\bracket{Z_j} \\
    = & ~
    \E\bracket{Z_j^{\mathrm{tr}}-Z_j}\\
    = & ~
    \E\left[\indic{\overline F_j}(\mu-Z_j)\right].
\end{align*}
Therefore, we have
\begin{align*}
    \norm b_2^4
    &=
    \langle b,b\rangle^2 \\
    &=
    \left(
    \E\left[\indic{\overline F_j}
    \left\langle \mu-Z_j,b\right\rangle\right]
    \right)^2 \\
    &\le
    \Pr[\overline F_j]\cdot
    \E\left[\left\langle Z_j-\mu,b\right\rangle^2\right] \\
    &\le
    \frac{\alpha}{4}\cdot
    b^\top \cov[Z_j] b \\
    &\le
    \frac{\gamma\alpha}{4}\norm b_2^2.
\end{align*}
where the third step follows from Cauchy--Schwarz inequality, the fourth step follows from $\Pr[\overline F_j]\le \alpha/4$
(proof of Claim~\ref{cla:covariance-condition_part_1}), and the last step follows from
$\cov[Z_j]\preceq \gamma\mathbb I_d$.
If $b={\bf 0}_d$, then the proof is done. Otherwise, dividing both sides by
$\norm b_2^2$ gives $\norm b_2^2 \le \frac{\gamma\alpha}{4}$.
Thus,
\[
    \norm{\mu^{\mathrm{tr}}-\mu}_2
    =
    \norm b_2
    \le
    \frac12\sqrt{\gamma\alpha}
    =
    O(\sqrt{\gamma\alpha}). \qedhere
\]
\end{proof}

\begin{claim}[Empirical covariance of the truncated variables]
\label{cla:covariance-condition_part_2_6}
Let $Z_1,\ldots,Z_n$ be \iid random vectors in $\R^d$ with mean $\E[Z_j] = \mu$
and covariance $\cov[Z_j]\preceq \gamma \mathbb I_d$, for all $j \in [n]$. Let
$Z_1^{\mathrm{tr}},\ldots,Z_n^{\mathrm{tr}}$ be the truncated variables from
Claim~\ref{cla:covariance-condition_part_1}. Let $\ov Z^{\mathrm{tr}}:=\frac1n\sum_{j=1}^n Z_j^{\mathrm{tr}}$. Then, with probability at least
$1-d\exp(-\Omega(n\alpha/d))$,
\[
    \frac1n\sum_{j=1}^n
    (Z_j^{\mathrm{tr}}-\overline Z^{\mathrm{tr}})
    (Z_j^{\mathrm{tr}}-\overline Z^{\mathrm{tr}})^\top
    \preceq
    C \gamma \mathbb I_d
\]
for a sufficiently large universal constant $C>0$.
\end{claim}

\begin{proof}
We define
\[
    M_j:=\frac1n
    (Z_j^{\mathrm{tr}}-\mu)(Z_j^{\mathrm{tr}}-\mu)^\top.
\]
We can get
\begin{align*}
    \norm{M_j}_2
    = & ~ \norm{\frac1n
    (Z_j^{\mathrm{tr}}-\mu)(Z_j^{\mathrm{tr}}-\mu)^\top} \\
    = & ~ \frac1n \norm{(Z_j^{\mathrm{tr}}-\mu)(Z_j^{\mathrm{tr}}-\mu)^\top}_2 \\
    = & ~ \frac1n \norm{Z_j^{\mathrm{tr}}-\mu}_2^2 \\
    \leq & ~ \frac{4\gamma d}{\alpha n},
\end{align*}
where the third step follows from Fact~\ref{fac:spectral_ell_2} and the last step follows from Claim~\ref{cla:covariance-condition_part_1}.
Hence, we can see that for all $j \in [n]$, with $R=O\left(\frac{\gamma d}{\alpha n}\right)$,
\[
    {\bf 0}_{d \times d} \preceq M_j \preceq R\mathbb I_d.
\]
Since truncation is performed around $\mu$, we have
\[
    Z_j^{\mathrm{tr}}-\mu
    =
    (Z_j-\mu)\indic{\|Z_j-\mu\|_2\le 2\sqrt{\gamma d/\alpha}}.
\]
Therefore, for every fixed $u\in\R^d$,
\[
    \langle Z_j^{\mathrm{tr}}-\mu,u\rangle^2
    =
    \indic{\|Z_j-\mu\|_2\le 2\sqrt{\gamma d/\alpha}}
    \langle Z_j-\mu,u\rangle^2
    \le
    \langle Z_j-\mu,u\rangle^2.
\]
Taking expectation gives
\[
    \E\left[\langle Z_j^{\mathrm{tr}}-\mu,u\rangle^2\right]
    \le
    \E\left[\langle Z_j-\mu,u\rangle^2\right].
\]
Moreover, since $\cov[Z_j]\preceq \gamma \mathbb{I}_d$ and $\mu=\E[Z_j]$,
\[
    \E\left[\langle Z_j-\mu,u\rangle^2\right]
    =
    u^\top \cov[Z_j]u
    \le
    \gamma \|u\|_2^2.
\] 
Moreover, using the covariance bound of the truncated variables, we have
\begin{align*}
    \lambda_{\max}
    := & ~ \left\|\E\left[\sum_{j=1}^n M_j\right]\right\|_2 \\
    = & ~ \left\|
    \frac1n \sum_{j=1}^n
    \E\left[
        (Z_j^{\mathrm{tr}}-\mu)(Z_j^{\mathrm{tr}}-\mu)^\top
    \right]
    \right\|_2 \\
    \leq & ~ \frac1n \sum_{j=1}^n \left\|
    \E\left[
        (Z_j^{\mathrm{tr}}-\mu)(Z_j^{\mathrm{tr}}-\mu)^\top
    \right]
    \right\|_2 \\
    \leq & ~ \gamma.
\end{align*}
If $\lambda_{\max}=0$, the proof is done. Otherwise, choose
$\Delta=C\gamma/\lambda_{\max}-1$ for a sufficiently large universal constant $C$,
so that $\Delta>0$ and $(1+\Delta)\lambda_{\max}=C\gamma$. Applying the matrix
Chernoff inequality (Fact~\ref{fac:matrix_Chernoff}) gives
\begin{align*}
    \Pr\left[
    \left\|\sum_{j=1}^n M_j\right\|_2>C\gamma 
    \right]
    \le & ~ d\exp\left(-\Omega\left(\frac{\Delta \lambda_{\max}}{R}\right)\right)\\
    \le & ~ d\exp\left(-\Omega\left(\frac{\gamma}{R}\right)\right)\\
    \le & ~
    d\exp\left(-\Omega\left(\frac{n\alpha}{d}\right)\right),
\end{align*}
where the second step follows from $\Delta\lambda_{\max}=C\gamma-\lambda_{\max}$ and
$\lambda_{\max}\le \gamma$, so we have
$\Delta\lambda_{\max}\ge (C-1)\gamma=\Omega(\gamma)$.
Combining with Fact~\ref{fac:from_mu_to_zbar}, we have, with probability at least $1 - d\exp\left(-\Omega\left(\frac{n\alpha}{d}\right)\right)$,
\begin{align*}
    \frac1n\sum_{j=1}^n
    (Z_j^{\mathrm{tr}}-\overline Z^{\mathrm{tr}})
    (Z_j^{\mathrm{tr}}-\overline Z^{\mathrm{tr}})^\top \preceq \frac1n
    \sum_{j=1}^n (Z_j^{\mathrm{tr}}-\mu)(Z_j^{\mathrm{tr}}-\mu)^\top \preceq C \gamma  \mathbb{I}_d,
\end{align*}
which completes the proof.
\end{proof}
\subsection{Refinements to SoS Variables and Polynomial Constraints}
\label{sub:rme_SoS:background}

We combine the truncation argument from the previous section
(\Cref{sub:rme_SoS:true}) with the SoS polynomial constraints
(Eqs.~\eqref{eq:old_1}, \eqref{eq:old_2}, \eqref{eq:old_3}, and
\eqref{eq:old_4}) to formalize the refined polynomial system.
Since the observed samples $\{x_j\}_{j=1}^n$ can be viewed as a
$2\alpha$-strong contamination of the truncated clean samples
$\{Z_j^{\mathrm{tr}}\}_{j=1}^n$, the SoS relaxation should only require
agreement with a $(1-2\alpha)$-fraction of the observed points. Moreover, in the
satisfiability proof, the candidate cleaned variables $Y_j$ will be assigned
by the truncated clean samples $Z_j^{\mathrm{tr}}$, rather than by the raw clean
samples $Z_j$.

\begin{definition}[Polynomial system $\mathsf C$]
\label{def:C}
Let $C>0$ be a sufficiently large universal constant. Given observed samples
$\{x_j\}_{j=1}^n\subset\R^d$, the polynomial system $\mathsf C$ consists of the
following constraints over variables $\{Y_j\}_{j=1}^n$, $\{W_j\}_{j=1}^n$, and
$B\in\R^{d\times d}$:
\begin{align}
    W_j^2 &= W_j, && \forall j\in[n], \label{eq:constraint_W2_W}\\
    \sum_{j=1}^n W_j &\ge (1-2\alpha)n, \label{eq:constraint_sum_W}\\
    W_j(Y_j-x_j) &= 0, && \forall j\in[n], \label{eq:constraint_wzv}\\
    \frac1n\sum_{j=1}^n (Y_j-\overline Y)(Y_j-\overline Y)^\top
    &= C\mathbb I_d - BB^\top . \label{eq:constraint_cov}
\end{align}
\end{definition}

The constraint $W_j^2=W_j$ makes $W_j$ a Boolean variable. The cardinality
constraint says that the candidate set retains at least $(1-2\alpha)n$ points. The
agreement constraint says that if $W_j=1$, then the candidate cleaned point
$Y_j$ must agree with the observed point $x_j$. The covariance constraint says
that the candidate cleaned points have empirical covariance at most
$C\mathbb I_d$; the slack matrix $B$ converts this semidefinite inequality into
polynomial equalities.

\begin{algorithm}[!ht]
\caption{SoS for robust mean estimation (Problem~\ref{prob:rme}).}
\label{alg:sos_rme}
\begin{algorithmic}[1]
\State \textbf{Input}: observed samples $\{x_j\}_{j=1}^n$, corruption parameter $\alpha$, and polynomial system $\mathsf C$ (Definition~\ref{def:C}).
\State Compute a pseudoexpectation $\wt{\mathbb E}$ satisfying $\mathsf C$.
\State \textbf{Output}: $\widehat\mu=\wt{\mathbb E}[\overline Y]\in\R^d$.
\end{algorithmic}
\end{algorithm}

We now formally show that our refined polynomial system (polynomial system $\mathsf C$, Definition~\ref{def:C}) is satisfiable for all arbitrary distribution $P$ with $\cov\bracket{P} \preceq \mathbb{I}_d$.

\begin{lemma}[Satisfiability of $\mathsf{C}$]
\label{lem:covariance-condition}
Under the setup of Problem~\ref{prob:rme}, with probability at least
$1-\delta$, there exists an assignment to the variables
\[
    \{Y_j\}_{j=1}^n\subset\R^d,\qquad
    \{W_j\}_{j=1}^n\subset\R,\qquad
    B\in\R^{d\times d}
\]
such that all constraints in $\mathsf C$ are satisfied, as long as
\[
    n \geq \Omega\left(\frac{d}{\alpha}\log\frac{d}{\delta}\right).
\]
\end{lemma}

\begin{proof}
Let $\{Z_j\}_{j=1}^n$ be the clean samples drawn before the adversary acts, and
let $\{Z_j^{\mathrm{tr}}\}_{j=1}^n$ be their truncated versions. We denote their
realized values by $\{z_j\}_{j=1}^n$ and
$\{z_j^{\mathrm{tr}}\}_{j=1}^n$, respectively.

By Claim~\ref{cla:covariance-condition_part_1}, with probability at least
$1-\exp(-\alpha n/4)$, at most $\alpha n$ clean samples are modified by truncation.
Since at most $\alpha n$ samples are corrupted by the adversary, at least
$(1-2\alpha)n$ indices satisfy both $x_j=z_j$ and $z_j=z_j^{\mathrm{tr}}$. 
Let $S:=\{j\in[n]: x_j=z_j=z_j^{\mathrm{tr}}\}$.
On this event, we have $|S|\ge (1-2\alpha)n$. Choose any subset
$T\subseteq S$ with
\[
    |T|=(1-2\alpha)n.
\]

We now define an assignment to the variables of $\mathsf C$. For every
$j\in[n]$, set
    $Y_j=z_j^{\mathrm{tr}}$,
and set
    $W_j=\indic{j\in T}$.
Then $W_j^2=W_j$ for every $j\in[n]$, and
\[
    \sum_{j=1}^n W_j=|T|=(1-2\alpha)n.
\]
Thus, the Boolean constraint and the cardinality constraint are satisfied.
Next, we verify the agreement constraint. If $W_j=0$, then
$W_j(Y_j-x_j)=0$ holds trivially. If $W_j=1$, then $j\in T\subseteq S$, so
    $x_j=z_j=z_j^{\mathrm{tr}}=Y_j$.
Hence,
\[
    W_j(Y_j-x_j)=0
\]
for every $j\in[n]$.
It remains to verify the covariance constraint. By
Claim~\ref{cla:covariance-condition_part_2_6}, with probability at least
    $1-d\exp\left(-\Omega\left(\frac{n\alpha}{d}\right)\right)$,
we have
\[
    \frac1n\sum_{j=1}^n
    (z_j^{\mathrm{tr}}-\overline z^{\mathrm{tr}})
    (z_j^{\mathrm{tr}}-\overline z^{\mathrm{tr}})^\top
    \preceq C\mathbb I_d,
\]
where
    $\overline z^{\mathrm{tr}}
    :=
    \frac1n\sum_{j=1}^n z_j^{\mathrm{tr}}$.
Since $Y_j=z_j^{\mathrm{tr}}$ for every $j\in[n]$, we have
    $\overline Y
    =
    \frac1n\sum_{j=1}^n Y_j
    =
    \overline z^{\mathrm{tr}}$.
Therefore,
\[
    \frac1n\sum_{j=1}^n
    (Y_j-\overline Y)(Y_j-\overline Y)^\top
    \preceq
    C\mathbb I_d.
\]
Taking the constant $C$ in Definition~\ref{def:C} to be sufficiently large,
there exists a matrix $B\in\R^{d\times d}$ such that
\[
    \frac1n\sum_{j=1}^n
    (Y_j-\overline Y)(Y_j-\overline Y)^\top
    =
    C\mathbb I_d-BB^\top.
\]
For example, $B$ can be chosen so that
    $BB^\top
    =
    C\mathbb I_d
    -
    \frac1n\sum_{j=1}^n
    (Y_j-\overline Y)(Y_j-\overline Y)^\top$.
Thus, the covariance constraint is satisfied.
Combining the two high-probability events by a union bound, the polynomial
system $\mathsf C$ is satisfiable with probability at least
\[
    1-\exp\left(-\frac{\alpha n}{4}\right)
    -d\exp\left(-\Omega\left(\frac{n\alpha}{d}\right)\right).
\]

\textbf{Proof of sample complexity.}
It remains to check the sample complexity. If
    $n \geq \Omega\left(\frac{d}{\alpha}\log\frac{d}{\delta}\right)$,
then
\[
    d\exp\left(-\Omega\left(\frac{n\alpha}{d}\right)\right)
    \le \frac{\delta}{2}.
\]
Moreover, by increasing the universal constant in the sample complexity bound if
necessary, we also have
\[
    \exp\left(-\frac{\alpha n}{4}\right)
    \le \frac{\delta}{2}.
\]
Hence the total failure probability is at most $\delta$. Therefore, with
probability at least $1-\delta$, there exists an assignment satisfying all
constraints in $\mathsf C$.
\end{proof}

\subsection{Background and Refinements to Identifiability of the SoS Relaxation}
\label{sub:rme_SoS:empirical}

We now prove identifiability for the refined polynomial system $\mathsf C$.
The argument follows the standard SoS identifiability framework for robust mean
estimation \cite{Schramm2022ProofsToAlgs}, but with one key refinement: the
reference clean object is no longer the raw clean empirical distribution.
Instead, because we first truncate the clean samples to handle heavy tails, the
SoS solution is compared to the empirical mean of the truncated clean samples
$\overline z^{\mathrm{tr}}$. We first show, using the agreement and covariance
constraints of $\mathsf C$, that any feasible pseudoexpectation must satisfy
    $\left\|\widetilde{\mathbb E}[\overline Y]
    - \overline z^{\mathrm{tr}}\right\|_2
    \le O(\sqrt{\alpha})$.
We then use the truncation guarantees from
\Cref{sub:rme_SoS:true} to relate $\overline z^{\mathrm{tr}}$ back to the true
population mean $\mu$. We further use a recent vector Bernstein inequality from \cite{mkmp25}
(Fact~\ref{fac:vec_bernstein}) to bound the deviation between the empirical
truncated mean and the population truncated mean. This allows us to obtain the
desired high-probability guarantee directly, and may save a
$\log(1/\delta)$ factor in the sample complexity compared with first proving a
constant-probability bound and then applying confidence amplification
\cite{hs16}.

\begin{lemma}[Identifiability of $\mathsf C$]
\label{lem:sos_mean}
Under the setting of Problem~\ref{prob:rme}, if
\[
    n \geq \Omega\left(\frac{d}{\alpha}\log\frac{d}{\delta}\right),
\]
then, with probability at least $1-\delta$, every pseudoexpectation
$\wt \E$ satisfying the SoS constraints $\mathsf C$ yields
\[
    \norm{\wt\E\bracket{\overline Y} - \mu}_2
    \leq O\paren{\sqrt{\alpha}}.
\]
Moreover, such a pseudoexpectation can be computed in time
\[
    \paren{nd+d^2}^{O(1)}.
\]
\end{lemma}
\begin{proof}
Let $\{Z_j\}_{j=1}^n$ be the clean samples before adversarial corruption.
Let $\{Z_j^{\mathrm{tr}}\}_{j=1}^n$ be the truncated clean samples, and we denote
their realized truncated values by $\{z_j^{\mathrm{tr}}\}_{j=1}^n$. Define
    $\overline z^{\mathrm{tr}}
    :=
    \frac1n\sum_{j=1}^n z_j^{\mathrm{tr}}$.
For each $j\in[n]$, we define the indicator
    $A_j := \indic{x_j = z_j^{\mathrm{tr}}}$.
On the high-probability event from Claim~\ref{cla:covariance-condition_part_1}
and the strong contamination model, the observed samples
$\{x_j\}_{j=1}^n$ differ from the truncated clean samples
$\{z_j^{\mathrm{tr}}\}_{j=1}^n$ on at most a $2\alpha$-fraction of indices. Hence,
\[
    \frac1n\sum_{j=1}^n (1-A_j)\le 2\alpha.
\]
Moreover, by Claim~\ref{cla:covariance-condition_part_2_6}, with probability
at least $1-d\exp(-\Omega(n\alpha/d))$,
\[
    \Sigma_{\mathrm{tr}}
    :=
    \frac1n\sum_{j=1}^n
    (z_j^{\mathrm{tr}}-\overline z^{\mathrm{tr}})
    (z_j^{\mathrm{tr}}-\overline z^{\mathrm{tr}})^\top
    \preceq C\mathbb I_d.
\]

We first prove that the SoS output is close to
$\overline z^{\mathrm{tr}}$. Define
    $b := \overline z^{\mathrm{tr}}-\overline Y$
and
    $\rho
    :=
    \frac1n\sum_{j=1}^n (1-W_jA_j)$.
Using the cardinality constraint $\sum_{j=1}^n W_j=(1-2\alpha)n$ and the bound
$\frac1n\sum_{j=1}^n(1-A_j)\le 2\alpha$, we have
\[
\mathsf C\sststile{2}{Y_1, \dots, Y_n, W_1, \dots, W_n, B}\left\{
\begin{aligned}
    \rho
    &=
    \frac1n\sum_{j=1}^n (1-W_j)
    +
    \frac1n\sum_{j=1}^n W_j(1-A_j) \\
    &\le
    2\alpha
    +
    \frac1n\sum_{j=1}^n (1-A_j) \\
    &\le
    4\alpha.
\end{aligned}
\right\}
\]
Here we use that $W_j\in \{0, 1\}$ has a degree-$2$ SoS proof from
$W_j^2=W_j$. Since $\alpha<0.1$, this also implies
$1-\rho\ge 1-4\alpha=\Omega(1)$.
Now observe that whenever $W_jA_j=1$, we have both $W_j=1$ and
$x_j=z_j^{\mathrm{tr}}$. By the agreement constraint
$W_j(Y_j-x_j)=0$, this implies
    $Y_j=x_j=z_j^{\mathrm{tr}}$.
Thus, for all $j\in[n]$,
\[
\mathsf C\sststile{2}{Y_1, \dots, Y_n, W_1, \dots, W_n, B}
\left\{
    W_jA_j(Y_j-z_j^{\mathrm{tr}})=0
\right\}.
\]
Consequently,
\[
\mathsf C\sststile{2}{Y_1, \dots, Y_n, W_1, \dots, W_n, B}
\left\{
\begin{aligned}
    b
    &=
    \frac1n\sum_{j=1}^n (z_j^{\mathrm{tr}}-Y_j) \\
    &=
    \frac1n\sum_{j=1}^n
    (1-W_jA_j)(z_j^{\mathrm{tr}}-Y_j).
\end{aligned}
\right\}
\]
Using
    $z_j^{\mathrm{tr}}-Y_j
    =
    (z_j^{\mathrm{tr}}-\overline z^{\mathrm{tr}})
    -
    (Y_j-\overline Y)
    +
    b$,
we get
\[
\mathsf C\sststile{2}{Y_1, \dots, Y_n, W_1, \dots, W_n, B}
\left\{
\begin{aligned}
    (1-\rho)b
    =
    \frac1n\sum_{j=1}^n
    (1-W_jA_j)(z_j^{\mathrm{tr}}-\overline z^{\mathrm{tr}})
    -
    \frac1n\sum_{j=1}^n
    (1-W_jA_j)(Y_j-\overline Y).
\end{aligned}
\right\}
\]
Taking inner product with $b = \overline z^{\mathrm{tr}}-\overline Y$ and applying SoS Cauchy--Schwarz (Fact~\ref{fac:cauchy_schwarz}) gives
\[
\mathsf C\sststile{6}{Y_1, \dots, Y_n, W_1, \dots, W_n, B}
\left\{
\begin{aligned}
    \langle (1-\rho)b, b \rangle^2
    & = (1-\rho)^2\norm b_2^4 \\
    & \leq
    2
    \left(
    \frac1n\sum_{j=1}^n
    (1-W_jA_j)
    \left\langle z_j^{\mathrm{tr}}-\overline z^{\mathrm{tr}},b\right\rangle
    \right)^2 \\
    &+
    2
    \left(
    \frac1n\sum_{j=1}^n
    (1-W_jA_j)
    \left\langle Y_j-\overline Y,b\right\rangle
    \right)^2 .
\end{aligned}
\right\}
\]

We bound the two terms separately. For the first term, recall that $\rho = \frac1n\sum_{j=1}^n (1-W_jA_j)$ and $\Sigma_{\mathrm{tr}} =
    \frac1n\sum_{j=1}^n
    (z_j^{\mathrm{tr}}-\overline z^{\mathrm{tr}})
    (z_j^{\mathrm{tr}}-\overline z^{\mathrm{tr}})^\top$, so another application of
Cauchy--Schwarz gives
\[
\mathsf C\sststile{6}{Y_1, \dots, Y_n, W_1, \dots, W_n, B}
\left\{
\begin{aligned}
    \left(
    \frac1n\sum_{j=1}^n
    (1-W_jA_j)
    \left\langle z_j^{\mathrm{tr}}-\overline z^{\mathrm{tr}},b\right\rangle
    \right)^2
    \leq & ~
    \rho\cdot
    \frac1n\sum_{j=1}^n
    \left\langle z_j^{\mathrm{tr}}-\overline z^{\mathrm{tr}},b\right\rangle^2 \\
    = & ~
    \rho\cdot b^\top \Sigma_{\mathrm{tr}} b \\
    \leq & ~
    C\rho\norm b_2^2,
\end{aligned}
\right\}
\]
where the first step follows from $(1-W_jA_j)^2 = 1 + W_j^2A_j^2 - 2W_jA_j = 1 + W_jA_j - 2W_jA_j = 1 - W_jA_j$, as $W_j, A_j \in \{0, 1\}$.
For the second term, the covariance constraint in $\mathsf C$ implies, by combining the polynomial constraint Eq.~\eqref{eq:constraint_cov} with
Fact~\ref{fac:sos_operator_norm}, that
\[
\mathsf C\sststile{4}{Y_1, \dots, Y_n, W_1, \dots, W_n, B}
\left\{
    \frac1n\sum_{j=1}^n
    \left\langle Y_j-\overline Y,b\right\rangle^2
    \leq
    C\norm b_2^2
\right\}.
\]
Hence, again by SoS Cauchy--Schwarz,
\[
\mathsf C\sststile{6}{Y_1, \dots, Y_n, W_1, \dots, W_n, B}
\left\{
\begin{aligned}
    \left(
    \frac1n\sum_{j=1}^n
    (1-W_jA_j)
    \left\langle Y_j-\overline Y,b\right\rangle
    \right)^2
    \leq & ~
    \rho\cdot
    \frac1n\sum_{j=1}^n
    \left\langle Y_j-\overline Y,b\right\rangle^2 \\
    \leq & ~
    C\rho\norm b_2^2 .
\end{aligned}
\right\}
\]
Combining the two bounds on the right-hand side, we get
\[
\mathsf C\sststile{6}{Y_1,\dots,Y_n,W_1,\dots,W_n,B}
\left\{
    (1-\rho)^2\norm b_2^4
    \le
    4C\rho\norm b_2^2
\right\}.
\]
Since $\rho\le 4\alpha$, this implies
\[
\mathsf C\sststile{6}{Y_1,\dots,Y_n,W_1,\dots,W_n,B}
\left\{
    (1-\rho)^2\norm b_2^4
    \le
    16C\alpha\norm b_2^2
\right\}.
\]
Moreover, since $\alpha<0.1$, we have
    $1-\rho\ge 1-4\alpha\ge 0.6$. Thus,
\[
\mathsf C\sststile{6}{Y_1,\dots,Y_n,W_1,\dots,W_n,B}
\left\{
    \norm b_2^4
    \le
    O(\alpha)\norm b_2^2
\right\}.
\]
Since $b=\overline z^{\mathrm{tr}}-\overline Y$, this is exactly
\[
\mathsf C\sststile{6}{Y_1,\dots,Y_n,W_1,\dots,W_n,B}
\left\{
    \norm{\overline z^{\mathrm{tr}}-\overline Y}_2^4
    \le
    O(\alpha)
    \norm{\overline z^{\mathrm{tr}}-\overline Y}_2^2
\right\}.
\]

Applying the pseudoexpectation $\wt \E$ to both sides gives
\[
    \wt\E\left[
        \norm{\overline z^{\mathrm{tr}}-\overline Y}_2^4
    \right]
    \leq
    O(\alpha)\,
    \wt\E\left[
        \norm{\overline z^{\mathrm{tr}}-\overline Y}_2^2
    \right].
\]
On the other hand, by \textbf{Part 1} of Fact~\ref{fac:pseudoexpectation}, we have
\[
    \wt\E\left[
        \norm{\overline z^{\mathrm{tr}}-\overline Y}_2^2
    \right]^2
    \leq
    \wt\E\left[
        \norm{\overline z^{\mathrm{tr}}-\overline Y}_2^4
    \right].
\]
Combining the previous two inequalities yields
\[
    \wt\E\left[
        \norm{\overline z^{\mathrm{tr}}-\overline Y}_2^2
    \right]
    \leq
    O(\alpha).
\]
By \textbf{Part 3} of Fact~\ref{fac:pseudoexpectation}
\[
    \wt\E\left[
        \norm{\overline z^{\mathrm{tr}}-\overline Y}_2^2
    \right]
    \geq
    \norm{\overline z^{\mathrm{tr}}-\wt\E[\overline Y]}_2^2,
\]
we conclude that
\[
    \norm{\overline z^{\mathrm{tr}}-\wt\E[\overline Y]}_2
    \leq
    O(\sqrt{\alpha}).
\]
It remains to relate $\overline z^{\mathrm{tr}}$ to the true mean $\mu$.
Recall that
\[
    \mu^{\mathrm{tr}}:=\E[Z_j^{\mathrm{tr}}].
\]
By Claim~\ref{cla:covariance-condition_part_3}, the truncation-induced
population mean shift satisfies
\[
    \norm{\mu^{\mathrm{tr}}-\mu}_2\le O(\sqrt{\alpha}).
\]
Thus, it remains to control the concentration of the empirical mean
$\overline z^{\mathrm{tr}}$ around $\mu^{\mathrm{tr}}$.
We apply Fact~\ref{fac:vec_bernstein} to the independent mean-zero random
vectors, for all $j\in[n]$,
\[
    X_j := \frac1n\left(Z_j^{\mathrm{tr}}-\mu^{\mathrm{tr}}\right).
\]
Then,
    $\sum_{j=1}^n X_j
    =
    \overline Z^{\mathrm{tr}}-\mu^{\mathrm{tr}}$.
We first bound the variance term:
\begin{align*}
    \E\left[\norm{Z_j^{\mathrm{tr}}-\mu^{\mathrm{tr}}}_2^2\right]
    &\le
    \E\left[\norm{Z_j^{\mathrm{tr}}-\mu}_2^2\right] \\
    &\le
    \E\left[\norm{Z_j-\mu}_2^2\right] \\
    &\le d.
\end{align*}
Therefore,
\[
    \sum_{j=1}^n
    \E\left[\norm{X_j}_2^2\right]
    =
    \frac1{n^2}
    \sum_{j=1}^n
    \E\left[\norm{Z_j^{\mathrm{tr}}-\mu^{\mathrm{tr}}}_2^2\right]
    \le
    \frac dn.
\]
Since
    $\norm{Z_j^{\mathrm{tr}}-\mu}_2
    \le 2\sqrt{d/\alpha}$
by the definition of truncation (Claim~\ref{cla:covariance-condition_part_1}), and
$\norm{\mu^{\mathrm{tr}}-\mu}_2\le O(\sqrt{\alpha})$, we have
\[
    \norm{Z_j^{\mathrm{tr}}-\mu^{\mathrm{tr}}}_2
    \le
    O\left(\sqrt{d/\alpha}\right).
\]
Hence
    $\norm{X_j}_2
    \le
    O\left(\frac1n\sqrt{\frac d\alpha}\right)$.
Applying Fact~\ref{fac:vec_bernstein}, with probability at least $1-\delta$,
\begin{align*}
    \norm{\overline Z^{\mathrm{tr}}-\mu^{\mathrm{tr}}}_2
    \le
    \sqrt{\frac dn}\left(1+\sqrt{2\log(1/\delta)}\right)
    +
    O\left(\frac1n\sqrt{\frac d\alpha}\log(1/\delta)\right).
\end{align*}
Therefore, as long as
    $n \ge
    \Omega\left(\frac d\alpha \log\frac d\delta\right)$,
the right-hand side is at most $O(\sqrt{\alpha})$. Hence, with probability at
least $1-\delta$,
\[
    \norm{\overline z^{\mathrm{tr}}-\mu^{\mathrm{tr}}}_2
    \le
    O(\sqrt{\alpha}).
\]
Combining this concentration bound with the truncation-induced mean shift, we
obtain
\[
    \norm{\overline z^{\mathrm{tr}}-\mu}_2
    \le
    \norm{\overline z^{\mathrm{tr}}-\mu^{\mathrm{tr}}}_2
    +
    \norm{\mu^{\mathrm{tr}}-\mu}_2
    \le
    O(\sqrt{\alpha}).
\]
Finally, by the triangle inequality,
\[
    \norm{\wt\E[\overline Y]-\mu}_2
    \le
    \norm{\wt\E[\overline Y]-\overline z^{\mathrm{tr}}}_2
    +
    \norm{\overline z^{\mathrm{tr}}-\mu}_2
    \le
    O(\sqrt{\alpha}).
\]

\textbf{Proof of running time.}
The running time follows from the standard SoS SDP algorithm. The polynomial
system $\mathsf C$ has $nd+n+d^2$ variables and constant degree. Hence a
constant-degree pseudoexpectation satisfying $\mathsf C$ can be found in time
    $(nd+d^2)^{O(1)}$.
This completes the proof. 
\end{proof}

%% file: 1_collaborate_with_close_mean.tex
\section{Warm-Up: Learning from Corrupted Batches With Heterogeneous Distributions}
\label{sec:bounded_mean}

In this section, we study the mean-shift model. There are $N$ users, each
contributing a batch of $n$ data points. At least a $(1-\eps)$-fraction of the
users are good. Each good user $i$ draws its clean samples from a distribution
$P_i$ whose mean $\mu_i$ is close to the target mean $\mu$, while the remaining
users may be fully adversarial.

\begin{restatable}[Mean shift]{problem}{ProbBoundedMean}\label{prob:bounded_mean} Consider parameters $\eps \in (0,0.5)$, $\beta > 0$, a target mean $\mu \in \mathbb{R}^d$, and $N$ users, each providing a batch of $n$ data points in $\mathbb{R}^d$. We are guaranteed that at least a $(1-\eps)$-fraction of these users are good; each good user $i \in \mathcal{G}$ provides $n$ i.i.d.\ samples drawn from an unknown distribution $P_i$ with mean $\mu_i$ and covariance $\Sigma_i \preceq \mathbb{I}_d$. The relevance of a good user's data is quantified by its proximity to the target: $\|\mu_i - \mu\|_2 \leq \sqrt{\beta}$. The remaining $\eps$-fraction of users are adversarial and may contribute entirely arbitrary batches. Our goal is to robustly estimate the true mean $\mu$ in $\ell_2$ distance.
\end{restatable}

Equivalently, the $N$ users are subject to $\eps$-strong contamination at the user level. One may imagine
that all users first draw their clean batches from their respective
distributions $P_i$, and then an adversary inspects all the data and replaces
the entire batches of an $\eps$-fraction of users by arbitrary data. The
algorithm only observes the final, possibly corrupted, batches.

\begin{theorem}
\label{thm:bounded_mean_upper}
Let $c\in(0,0.5)$ be a universal constant. Given $\eps\in(0,c)$, $\beta\ge0$, and $n=o(\dims)$ samples from each of $N\in\Z_+$ users satisfying
\[
nN \geq
\min\left\{
\Omega\left(\frac{d}{\eps/n}\log\left(\frac d\delta\right)\right),
\Omega\left(\frac{d}{\beta}\log\left(\frac d\delta\right)\right)
\right\}.
\]
Then there exists a polynomial-time algorithm that solves
Problem~\ref{prob:bounded_mean}. With probability at least $1-\delta$, the algorithm outputs
$\widehat \mu \in \R^d$ satisfying
\[
\norm{\widehat \mu-\mu}_2
=
O\left(\sqrt{\frac{\eps}{n}}+\sqrt{\beta}\right).
\]
\end{theorem}

\begin{remark}
    For simplicity
of presentation, we assume \(\beta\in(0,c)\) for a sufficiently small universal
constant \(c>0\). This is without loss of generality up to a rescaling argument:
if necessary, we may multiply all samples by a common scalar factor so that the
rescaled heterogeneity parameter is smaller than \(c\), run the algorithm on
the rescaled instance, and then rescale the output back to the original scale.
\end{remark}

\begin{algorithm}[!ht]
\caption{SoS for solving mean shift (Problem~\ref{prob:bounded_mean}).}
\label{alg:sos_bounded}
\begin{algorithmic}[1]
\State \textbf{Input}: $n$ samples from each user, $\beta \in(0, 0.1)$, $\eps \in(0, 0.1)$, polynomial system $\mathsf{A}$ (Definition~\ref{def:A}).
\State Compute a pseudoexpectation $\wt{\mathbb{E}}$ satisfying the constraints of the polynomial system $\mathsf{A}$.
\State \textbf{Output:} $\wh \mu = \wt{\mathbb{E}}\bracket{\ov Z} \in \R^d$.
\end{algorithmic}
\end{algorithm}

The recent work \cite{pp25} studies robust estimation under combined global and local corruptions using stability-based algorithms. Their framework gives better running time and sub-Gaussian sample-complexity rates for Problem~\ref{prob:bounded_mean} compared with our SoS method. However, our goal is not to improve upon \cite{pp25}. Instead, our contribution is different. To the best of our knowledge, this is the first SoS formulation for the untrusted-batch mean-shift problem studied here. The SoS perspective is important for two reasons. First, it provides a modular polynomial system whose satisfiability and identifiability can be analyzed separately. Second, this section serves as a building block for the more complicated two-level heterogeneous corruption model (Problem~\ref{prob:heterogeneous_two_level}), where user-level corruption, sample-level corruption, and heterogeneous user means must be handled simultaneously. The stability-based techniques of \cite{pp25}, while has better running time and sample complexity, cannot be directly applied to the more complicated problem that we aim to study. We will specify the reason in \Cref{sec:combined}.

Recall the standard SoS robust mean
estimation principle discussed in \Cref{sec:rme_SoS}: if one can identify a
large truncated set of points whose empirical covariance is small, and whose
truncation does not significantly shift the mean, then the empirical mean of
this truncated set serves as a stable proxy for the original clean mean. In
the present mean-shift model, we apply this principle at the user level.
Instead of working with individual samples, we work with the empirical mean of
each user's batch. For a good user, this empirical mean has covariance smaller
by a factor of $1/n$. Although its expectation may be shifted from the target
mean $\mu$ by as much as $\sqrt{\beta}$, the resulting covariance scale is
still only
$\frac1n+\beta$.
Thus, the natural goal is to find a large collection of user-level empirical
means whose empirical covariance is bounded at this scale.

To implement this idea, we introduce a polynomial system $\mathsf A$
(Definition~\ref{def:A}). Let $x_i\in\R^d$ denote the observed empirical mean
of user $i$, obtained by averaging the $n$ submitted samples. For each clean
user, let $Z_i\in\R^d$ denote the clean user-level empirical mean before
user-level corruption, and let $Z_i^{\mathrm{tr}}\in\R^d$ denote the truncated
version of $Z_i$ used in the satisfiability proof. In the SoS relaxation, we
introduce variables $Y_i\in\R^d$, representing candidate cleaned user-level
empirical means, and Boolean variables $W_i\in \{0,1\}$, certifying whether
$Y_i$ agrees with the observed empirical mean $x_i$. We write
$\overline Y := \frac1N\sum_{i=1}^N Y_i$.
In the satisfiability, we can treat $Y_i$ to be the realized
truncated clean empirical mean $Z_i^{\mathrm{tr}}$ to show that there exists a solution to our polynomial system.

\begin{definition}[Polynomial system $\mathsf A$]
\label{def:A}
Let $C_A>0$ be a sufficiently large universal constant. Let $\eps' := \min\left\{ \max\{\eps,n\beta\},
        \frac{1}{10}
    \right\}$. The
polynomial system $\mathsf A$ consists of the following constraints over
variables $\{W_i\}_{i=1}^N$, $\{Y_i\}_{i=1}^N$, and $B \in \R^{d \times d}$ (a matrix of ``slack'' variable):
\begin{align*}
    W_i(Y_i-x_i) &= 0, && \forall i\in[N],\\
    W_i^2 &= W_i, && \forall i\in[N],\\
    \sum_{i=1}^N W_i &= (1-2\eps')N,\\
    \frac1N \sum_{i=1}^N (Y_i-\overline Y)(Y_i-\overline Y)^\top
    &=
    C_A\left(\frac1n+\beta\right)\mathbb I_d - BB^\top.
\end{align*}
\end{definition}

The main technical point is that we do not force the SoS system to retain a
$(1-\eps)$-fraction of users. Instead, we use the relaxed retention parameter
\[
\eps'
:=
\min\left\{\max\{\eps,n\beta\},\frac{1}{10}\right\}.
\]
This relaxation allows fewer $Y_i$ to satisfy $W_i(Y_i-x_i) = 0$, allowing us to replace more extreme samples $z_{i, j}$ to $\mu_i$ via truncation. This makes the empirical covariance bound $O\left(\frac1n+\beta\right)\mathbb I_d$ easier to be satisfied. Therefore, this is the key device that gives the desired sample complexity. If
the target error is dominated by $\sqrt{\eps/n}$, then $\eps'\approx \eps$,
and the system behaves like the usual user-level robust mean estimation
relaxation. If the target error is dominated by $\sqrt{\beta}$, then we may
allow the SoS system to discard up to an $O(n\beta)$-fraction of user-level
empirical means without worsening the final error. This additional slack makes
the truncation and covariance concentration step succeed under
\[
nN \geq
\min\left\{
\Omega\left(\frac{d}{\eps/n}\log\left(\frac d\delta\right)\right),
\Omega\left(\frac{d}{\beta}\log\left(\frac d\delta\right)\right)
\right\},
\]
rather than under the worse sample complexity obtained by taking the maximum
of the two requirements.

For the SoS algorithm to work, we prove two statements.

\begin{enumerate}
\item \textbf{Satisfiability.} With probability at least $1-\delta/2$,
there exists an assignment to all variables satisfying the polynomial
system $\mathsf A$. The witness is obtained by setting
$Y_i=Z_i^{\mathrm{tr}}$ and selecting a large set of indices on which the
observed empirical means agree with the truncated clean empirical means.
We show that the user-level truncation argument with the relaxed
parameter $\eps'$, together with a covariance bound at scale
$\frac1n+\beta$ despite the fact that the clean user means $\mu_i$ need
not be identical. See Lemma~\ref{lem:bounded_covariance} in
\Cref{sub:bounded_mean:satisfy}.

\item \textbf{Identifiability.} With probability at least $1-\delta/2$,
every pseudoexpectation satisfying $\mathsf A$ yields an accurate estimate
of the true mean $\mu$. We show that the SoS proof remains valid
even after relaxing the retained fraction from $(1-\eps)$ to
$(1-2\eps')$. The resulting SoS error term is
\[
    O\left(\sqrt{\eps'\left(\frac1n+\beta\right)}\right),
\]
which is still bounded by
\[
    O\left(\sqrt{\frac{\eps}{n}}+\sqrt{\beta}\right).
\]
The remaining terms are controlled by the fluctuation of the truncated
user-level empirical means, the truncation bias, and the allowed
heterogeneity $\|\mu_i-\mu\|_2\le \sqrt{\beta}$. See
Lemma~\ref{lem:mean_close} in \Cref{sub:bounded_mean:identify}.

\end{enumerate}

By assuming that {satisfiability} (Lemma~\ref{lem:bounded_covariance}) and {identifiability} (Lemma~\ref{lem:mean_close}) hold, we prove Theorem~\ref{thm:bounded_mean_upper}.

\begin{proof}[Proof of Theorem~\ref{thm:bounded_mean_upper}]
We apply
Lemma~\ref{lem:bounded_covariance} and Lemma~\ref{lem:mean_close} with failure
probability $\delta/2$. Since
\[
    \log\left(\frac{d}{\delta/2}\right)
    =
    O\left(\log\left(\frac d\delta\right)\right),
\]
the assumed sample complexity still implies the hypotheses of both lemmas.

By Lemma~\ref{lem:bounded_covariance}, with probability at least
$1-\delta/2$, the polynomial system $\mathsf A$ is satisfiable. Hence the SoS
relaxation has a feasible pseudoexpectation. By Lemma~\ref{lem:mean_close}, with
probability at least $1-\delta/2$, every pseudoexpectation satisfying
$\mathsf A$ yields
\[
    \left\|
    \widetilde{\mathbb E}[\overline Y]-\mu
    \right\|_2
    =
    O\left(\sqrt{\frac{\eps}{n}}+\sqrt{\beta}\right).
\]
Therefore, by a union bound, with probability at least $1-\delta$, both events
hold simultaneously. 
\end{proof}

\subsection{Satisfiability}
\label{sub:bounded_mean:satisfy}

\begin{lemma}[Satisfiability of $\mathsf A$]
\label{lem:bounded_covariance}
Under the setting of Problem~\ref{prob:bounded_mean}, with probability at least
$1-\delta$, there exists an assignment to the variables
\[
    \{W_i\}_{i=1}^N \subset \R,\qquad
    \{Y_i\}_{i=1}^N \subset \R^d,\qquad
    B\in\R^{d\times d}
\]
such that all constraints in $\mathsf A$ are satisfied, as long as
\begin{align*}
    nN \geq
    \min\left\{
    \Omega\left(\frac{d}{\eps/n}\log\left(\frac d\delta\right)\right),
    \Omega\left(\frac{d}{\beta}\log\left(\frac d\delta\right)\right)
    \right\}.
\end{align*}
\end{lemma}

\paragraph{Overview}
Now, we give a proof overview, and then give a formal proof. Recall that in the standard SoS-based robust mean
estimation (as discussed in \Cref{sec:rme_SoS}), the clean samples may be
heavy-tailed under only a bounded covariance assumption, so their empirical covariance may not be bounded with high
probability. The standard way is to truncate the clean random variables:
after modifying only a small fraction of them, the truncated variables are
bounded enough for matrix concentration, while the observed corrupted samples
can still be viewed as a slightly more corrupted version of the truncated clean
samples.

In our setting, the objects entering the polynomial system $\mathsf A$ are not
individual samples, but user-level empirical means. For each good user $i$, let
$Z_i$ be the empirical mean of its $n$ clean samples. Since $Z_i$ is centered at
$\mu_i$ and satisfies $\cov\bracket{Z_i} \preceq \frac1n \mathbb I_d$, we
truncate $Z_i$ around $\mu_i$, obtaining $Z_i^{\mathrm{tr}}$. The shift from
$\mu_i$ to the global mean $\mu$ is then handled separately using
$\|\mu_i-\mu\|_2\le \sqrt{\beta}$. Thus the natural covariance scale is
    $\frac1n+\beta$.

We set the truncation parameter to be $\rho=\eps'$. With high probability,
truncation modifies at most $\eps' N$ clean user-level empirical means (Claim~\ref{cla:covariance-condition_part_1_4}). Since
the adversary corrupts at most $\eps N\le \eps' N$ users, the observed empirical
means $\{x_i\}_{i=1}^N$ and the truncated clean empirical means
$\{Z_i^{\mathrm{tr}}\}_{i=1}^N$ agree on at least $(1-2\eps')N$ indices. The $\ell_2$ distance of the truncated mean $\mu_i^{\mathrm{tr}} = \E\bracket{Z_i^{\mathrm{tr}}}$ and the user mean $\mu_i$ can always be bounded by $O\paren{\sqrt{\frac{\eps}{n}}}$, for all $i \in [N]$ (Claim~\ref{cla:covariance-condition_part_3_4}), meaning that this truncation will not negatively influence our error guarantee $O\paren{\sqrt{\frac{\eps}{n}} + \sqrt{\beta}}$.

On this high-probability event, we construct a feasible assignment as follows.
Set $Y_i=Z_i^{\mathrm{tr}}$ for every $i$. Then choose $(1-2\eps')N$ indices
where $x_i=Z_i^{\mathrm{tr}}$ and set $W_i=1$ on these indices and $W_i=0$ on
the remaining indices. This immediately satisfies the agreement constraint
$W_i(Y_i-x_i)=0$, the Boolean constraint $W_i^2=W_i$, and the cardinality
constraint $\sum_i W_i=(1-2\eps')N$. 

It remains to verify the covariance constraint. By the truncation construction
and matrix Chernoff, with high probability,
\[
    \frac1N\sum_{i=1}^N
    (Z_i^{\mathrm{tr}}-\mu)(Z_i^{\mathrm{tr}}-\mu)^\top
    \preceq
    C_A\left(\frac1n+\beta\right)\mathbb I_d ,
\]
where the universal constant $C_A$ absorbs the losses from truncation and
concentration. Since centering at the empirical average can only decrease
empirical covariance, and since $Y_i=Z_i^{\mathrm{tr}}$, we get
\[
    \frac1N\sum_{i=1}^N
    (Y_i-\overline Y)(Y_i-\overline Y)^\top
    \preceq
    C_A\left(\frac1n+\beta\right)\mathbb I_d .
\]
Therefore the covariance constraint in $\mathsf A$ is satisfied (Claim~\ref{cla:covariance-condition_part_2}) by choosing a
slack matrix $B$ such that
\[
    C_A\left(\frac1n+\beta\right)\mathbb I_d
    -
    \frac1N\sum_{i=1}^N
    (Y_i-\overline Y)(Y_i-\overline Y)^\top
    =
    BB^\top .
\]
Thus all constraints in $\mathsf A$ are satisfiable with high probability. The
sample complexity follows from the matrix Chernoff bound applied to the
truncated variables, with the choice
$\eps'=\min\{\max\{\eps,n\beta\},1/10\}$.

To prove satisfiability, we need to show that the aggregated empirical covariance of the truncated user-level empirical means is bounded with high probability:
\begin{align}\label{eq:B_bounded_covariance_mu}
    \frac{1}{N} \sum_{i=1}^N \paren{Z_i^{\mathrm{tr}} - \mu}\paren{Z_i^{\mathrm{tr}} - \mu}^\top \preceq C_A\paren{\frac{1}{n} + \beta} \mathbb{I}_d.
\end{align}

\paragraph{Formal proof} We now present the formal proof.

\begin{remark}
    Claim~\ref{cla:covariance-condition_part_1} and
    Claim~\ref{cla:covariance-condition_part_3} are stated for \iid samples
    with a common mean. In this section, the user-level random variables
    $Z_i$ need not have the same mean: each $Z_i$ has mean $\mu_i$, and the
    means satisfy $\|\mu_i-\mu\|_2\le \sqrt{\beta}$. Nevertheless, the proofs of
    Claim~\ref{cla:covariance-condition_part_1_4} and
    Claim~\ref{cla:covariance-condition_part_3_4} are essentially identical to
    those of Claim~\ref{cla:covariance-condition_part_1} and
    Claim~\ref{cla:covariance-condition_part_3}, respectively. The only changes
    are that the covariance scale is now $\gamma=1/n$, and the truncation
    parameter $\eps$ is replaced by $\rho$. Indeed, both arguments are
    pointwise in the underlying mean: for each user $i$, truncation is performed
    around its own mean $\mu_i$. We will include the short proof verification under Claim~\ref{cla:covariance-condition_part_1_4} and
    Claim~\ref{cla:covariance-condition_part_3_4} for
    completeness.
\end{remark}
\begin{claim}\label{cla:covariance-condition_part_1_4}
Let $\rho \in (0,1)$ be a truncation parameter and $\gamma > 0$\footnote{In this section, we choose $\gamma = \frac{1}{n}$.}. For each user $i \in [N]$, let $Z_i \in \R^d$ be an independent random vector
with mean $\mu_i \eqdef \E[Z_i]$ and covariance $\cov[Z_i] \preceq \gamma \mathbb{I}_d$. We define the truncated random variable
\begin{align*}
    Z_i^{\mathrm{tr}} \eqdef
\begin{cases}
Z_i, & \norm{Z_i - \mu_i}_2 \leq 2\sqrt{\frac{\gamma d}{\rho}},\\
\mu_i, & \text{otherwise.}
\end{cases}
\end{align*}
Then, we have
\begin{align*}
    \Pr\bracket{\frac{1}{N}\sum_{i=1}^N \indic{Z_i \neq Z_i^{\mathrm{tr}}} \ge \rho} \leq \exp\paren{-\frac{\rho N}{4}}.
\end{align*}
\end{claim}

\begin{proof}
For each $i\in[N]$,
\[
    \E\left[\norm{Z_i-\mu_i}_2^2\right]
    =
    \mathrm{tr}\bracket{\cov[Z_i]}
    \le
    \gamma d.
\]
Therefore, by Markov's inequality,
\[
    \Pr[Z_i\neq Z_i^{\mathrm{tr}}]
    =
    \Pr\left[
        \norm{Z_i-\mu_i}_2^2>
        4\frac{\gamma d}{\rho}
    \right]
    \le
    \frac{\E[\norm{Z_i-\mu_i}_2^2]}{4d\gamma/\rho}
    \le
    \frac{\rho}{4}.
\]
The indicators
\[
    \indic{Z_i\neq Z_i^{\mathrm{tr}}},\qquad i\in[N],
\]
are independent Bernoulli random variables with expectations at most
$\rho/4$. Applying the multiplicative Chernoff bound gives
\[
    \Pr\left[
        \frac1N\sum_{i=1}^N
        \indic{Z_i\neq Z_i^{\mathrm{tr}}}
        \ge \rho
    \right]
    \le
    \exp\left(-\frac{\rho N}{4}\right). \qedhere
\]
\end{proof}

To proceed, we need to ensure that the truncation step—used to bound the covariance—does not alter the mean by too much. The following claim quantifies this effect, showing that the difference between the original and truncated means is small and scales as $O\paren{\sqrt{\rho/n}}$.

\begin{claim}\label{cla:covariance-condition_part_3_4}
    Under the same notation as in Claim~\ref{cla:covariance-condition_part_1_4}, for each user $i \in [N]$, let $\mu_i^{\mathrm{tr}} \eqdef \E\bracket{Z_i^{\mathrm{tr}}}$.
    Then, we have
    \begin{align*}
        \norm{\mu_i^{\mathrm{tr}} - \mu_i}_2 \leq O\paren{\sqrt{\gamma \rho}}.
    \end{align*}
\end{claim}

\begin{proof}
The proof is the same as Claim~\ref{cla:covariance-condition_part_3}.
\end{proof}

Having established in Claim~\ref{cla:covariance-condition_part_3_4} that the means of the truncated and original variables are close, we now extend this result to show that the empirical covariance of the truncated variables remains well-behaved. 
\begin{remark}
    The genuinely different part is the empirical covariance bound. In
    Claim~\ref{cla:covariance-condition_part_2_6}, all truncated variables have
    a common reference mean $\mu$, and the goal is to show that their empirical
    covariance is bounded by a constant multiple of $\mathbb I_d$. In the
    present mean-shift setting, however, the truncated user-level variables may
    have different means $\mu_i^{\mathrm{tr}}$, and these means are only known
    to be close to the target mean $\mu$. Thus, the empirical second moment
    around $\mu$ contains two contributions: the within-user fluctuation, of
    order $1/n$, and the between-user mean shift, of order $\beta$ (together
    with the smaller truncation-induced shift $\rho/n$). Consequently,
    Claim~\ref{cla:covariance-condition_part_2} must prove the stronger bound
    \[
        \frac1N\sum_{i=1}^N
        (Z_i^{\mathrm{tr}}-\mu)(Z_i^{\mathrm{tr}}-\mu)^\top
        \preceq
        C\left(\frac1n+\beta\right)\mathbb I_d,
    \]
    rather than merely a constant covariance bound. 
\end{remark}

\begin{claim}\label{cla:covariance-condition_part_2}
Let $Z_1^{\mathrm{tr}},\dots,Z_N^{\mathrm{tr}} \in \R^d$ be independent truncated random vectors defined as in
Claim~\ref{cla:covariance-condition_part_1_4} with truncation parameter $\rho$. For each $i\in[N]$, let $\mu_i^{\mathrm{tr}} \eqdef \E\bracket{Z_i^{\mathrm{tr}}}$.
Suppose that for all $i\in[N]$, $\cov\bracket{Z_i^{\mathrm{tr}}} \preceq \frac{1}{n} \mathbb{I}_d$ and $\norm{\mu_i^{\mathrm{tr}}-\mu}_2^2 \leq O\paren{\beta + \rho/n}$.
Then, with probability at least $1-\delta$,
\begin{align*}
    \frac{1}{N}\sum_{k=1}^N (Z_k^{\mathrm{tr}}-\mu)(Z_k^{\mathrm{tr}}-\mu)^\top
    \preceq
    C\left(\frac1n+\beta\right) \mathbb{I}_d,
\end{align*}
where
\begin{align*}
    \delta =
\begin{cases}
d\exp\left(-\Omega\left(\frac{\rho N}{d}\right)\right), & \beta n \le 1,\\
d\exp\left(-\Omega\left(\frac{\beta n \rho N}{d}\right)\right), & 1 < \beta n \le \frac{d}{\rho}.
\end{cases}
\end{align*}
\end{claim}

\begin{proof}
We define
\begin{align}\label{eq:M_k}
    M_k := \frac{1}{N} (Z_k^{\mathrm{tr}} - \mu)(Z_k^{\mathrm{tr}} - \mu)^\top.
\end{align}
By the definition of $Z_k^{\mathrm{tr}}$, Claim~\ref{cla:covariance-condition_part_1_4}, and Claim~\ref{cla:covariance-condition_part_3_4}, we have
\begin{align}\label{eq:norm_zktr_mu}
    \norm{Z_k^{\mathrm{tr}} - \mu}_2
    \leq &~ \norm{Z_k^{\mathrm{tr}} - \mu_k^{\mathrm{tr}}}_2 + \norm{\mu_k^{\mathrm{tr}} - \mu_k}_2 + \norm{\mu_k - \mu}_2 \notag\\
    \leq &~ O\paren{\sqrt{\frac{d}{n\rho}} + \sqrt{\beta}}.
\end{align}
Therefore, we have
\begin{align*}
    \max_{k \in [N]} \norm{M_k}_2
    = &~ \max_{k \in [N]} \norm{\frac{1}{N} (Z_k^{\mathrm{tr}} - \mu)(Z_k^{\mathrm{tr}} - \mu)^\top}_2 \\
    = &~ \frac{1}{N} \max_{k \in [N]} \norm{(Z_k^{\mathrm{tr}} - \mu)(Z_k^{\mathrm{tr}} - \mu)^\top}_2 \\
    = &~ \frac{1}{N} \max_{k \in [N]} \norm{Z_k^{\mathrm{tr}} - \mu}_2^2 \\
    \leq &~ O\paren{\frac{d}{\rho nN} + \frac{\beta}{N}},
\end{align*}
where the first step follows from Eq.~\eqref{eq:M_k}, the third step follows from Fact~\ref{fac:spectral_ell_2}, and the last step follows from Eq.~\eqref{eq:norm_zktr_mu}. 
This implies that
\begin{align*}
    \mathbf{0}_{d \times d} \preceq M_k \preceq R \cdot \mathbb{I}_d,
    \qquad
    R = O\paren{\frac{d}{\rho nN} + \frac{\beta}{N}} .
\end{align*}
Additionally, by Fact~\ref{fac:low_covariance}, we have
\begin{align*}
    \mathrm{Cov}\bracket{Z_k^{\mathrm{tr}}} \preceq \frac{1}{n} \mathbb{I}_d.
\end{align*}
Also, Claim~\ref{cla:covariance-condition_part_3_4} and the assumption $\norm{\mu_k-\mu}_2^2\leq \beta$ imply
\begin{align*}
    \norm{\mu_k^{\mathrm{tr}} - \mu}_2^2 \leq O\paren{\beta + \rho/n},
\end{align*}
and hence
\begin{align*}
    \paren{\mu_k^{\mathrm{tr}} - \mu} \paren{\mu_k^{\mathrm{tr}} - \mu}^\top
    \preceq O\paren{\beta + \rho/n} \mathbb{I}_d.
\end{align*}
Therefore,
\begin{align*}
    \E\bracket{\frac{1}{N}\sum_{k = 1}^{N} \paren{Z_k^{\mathrm{tr}} - \mu} \paren{Z_k^{\mathrm{tr}} - \mu}^\top}
    = &~ \frac{1}{N}\sum_{k = 1}^{N} \E\bracket{\paren{Z_k^{\mathrm{tr}} - \mu} \paren{Z_k^{\mathrm{tr}} - \mu}^\top}\\
    = &~ \frac{1}{N}\sum_{k = 1}^{N} \paren{\E\bracket{\paren{Z_k^{\mathrm{tr}} - \mu_k^{\mathrm{tr}}} \paren{Z_k^{\mathrm{tr}} - \mu_k^{\mathrm{tr}}}^\top} + \paren{\mu_k^{\mathrm{tr}} - \mu} \paren{\mu_k^{\mathrm{tr}} - \mu}^\top}\\
    \preceq &~ O\paren{\frac{1}{n} + \beta}\mathbb{I}_d,
\end{align*}
where we used $\rho\leq 1$ in the last line. Thus
\begin{align*}
    \lambda_{\max}
    \eqdef
    \norm{\E\bracket{\sum_{k=1}^N M_k}}_2
    \leq C_0\paren{\frac{1}{n}+\beta}
\end{align*}
for a universal constant $C_0>0$.
Let $L \eqdef \frac{1}{n}+\beta$. Choose a sufficiently large universal constant $C_1>C_0$. Applying the matrix Chernoff inequality with threshold $C_1L$ gives
\begin{align}\label{eq:failure}
    \Pr\bracket{ \left\lVert \sum_k M_k \right\rVert_2 > C_1 L}
    \leq &~ d \exp\paren{-\Omega\paren{\frac{L}{R}}} \notag\\
    \leq &~ d \exp\paren{-\Omega \paren{\frac{N \paren{1 + \beta n}}{\frac{d}{\rho} + \beta n}}}.
\end{align}
It remains to simplify the exponent. We consider the two regimes that are needed later, namely $\beta n \leq 1$ and $1<\beta n\leq d/\rho$.

\textbf{Case 1 ($\beta n \leq 1$).}
Suppose $\beta n \leq 1$. Then Eq.~\eqref{eq:failure} gives
\begin{align*}
    d \exp\paren{-\Omega \paren{\frac{N \paren{1 + \beta n}}{\frac{d}{\rho} + \beta n}}}
    \leq &~ d \exp\paren{-\Omega \paren{\frac{N}{d/\rho}}} \\
    = &~ d \exp\paren{-\Omega \paren{\frac{\rho N}{d}}}.
\end{align*}

\textbf{Case 2 ($1 < \beta n \leq \frac{d}{\rho}$).}
Suppose $1 < \beta n \leq \frac{d}{\rho}$. Then Eq.~\eqref{eq:failure} gives
\begin{align*}
    d \exp\paren{-\Omega \paren{\frac{N \paren{1 + \beta n}}{\frac{d}{\rho} + \beta n}}}
    \leq &~ d \exp\paren{-\Omega \paren{\frac{\beta n N}{d/\rho}}} \\
    = &~ d \exp\paren{-\Omega \paren{\frac{\beta \rho n N}{d}}}.
\end{align*}
Finally, choosing $C\geq C_1$ gives
\begin{align*}
    \frac{1}{N}\sum_{k = 1}^{N} \paren{Z_k^{\mathrm{tr}} - \mu} \paren{Z_k^{\mathrm{tr}} - \mu}^\top
    \preceq C\paren{\frac{1}{n}+\beta}\mathbb I_d
\end{align*}
with probability at least $1-\delta$, for $\delta$ as stated.
\end{proof}

Now, we combine everything together to prove Lemma~\ref{lem:bounded_covariance}.

\begin{proof}[Proof of Lemma~\ref{lem:bounded_covariance}]
We apply Claims~\ref{cla:covariance-condition_part_1_4},
\ref{cla:covariance-condition_part_3_4}, and
\ref{cla:covariance-condition_part_2} with truncation parameter
$\rho=\eps'$.

By Claim~\ref{cla:covariance-condition_part_1_4}, with probability at least
$1-\exp(-\Omega(\eps' N))$, at most $\eps' N$ clean empirical means are
modified by truncation. Since at most $\eps N\leq \eps' N$ users are corrupted
at the user level, at least $(1-2\eps')N$ indices satisfy $x_i = Z_i^{\mathrm{tr}}$,
after the clean empirical means are sampled and truncated. Throughout the rest
of this proof, we treat $\{Z_i^{\mathrm{tr}}\}_{i=1}^N$ as fixed realized
vectors.

At least $(1-2\eps')N$ indices satisfy $x_i=Z_i^{\mathrm{tr}}$. On this event, set $Y_i = Z_i^{\mathrm{tr}}$, for all $i\in[N]$.
Let $S = \{i\in[N] : x_i = Z_i^{\mathrm{tr}}\}$.
Choose any subset $T\subseteq S$ with $|T|=(1-2\eps')N$, and set $W_i = \indic{i\in T}$.
Therefore, the constraints
\begin{align*}
    W_i(Y_i-x_i)=0,\qquad W_i^2=W_i,\qquad
    \sum_{i=1}^N W_i=(1-2\eps')N
\end{align*}
all hold. It remains to verify the covariance constraint. By
Claim~\ref{cla:covariance-condition_part_2}, with probability at least
$1-\delta$,
\begin{align*}
    \frac{1}{N} \sum_{i=1}^N
    \paren{Z_i^{\mathrm{tr}}-\mu}
    \paren{Z_i^{\mathrm{tr}}-\mu}^\top
    \preceq
    C\paren{\frac1n+\beta}\mathbb I_d .
\end{align*}
Combining this with Fact~\ref{fac:from_mu_to_zbar}, applied to the fixed
vectors $\{Z_i^{\mathrm{tr}}\}_{i=1}^N$, gives
\begin{align*}
    \frac{1}{N} \sum_{i=1}^N
    \paren{Z_i^{\mathrm{tr}}-\overline Z^{\mathrm{tr}}}
    \paren{Z_i^{\mathrm{tr}}-\overline Z^{\mathrm{tr}}}^\top
    \preceq
    C\paren{\frac1n+\beta}\mathbb I_d,
\end{align*}
where
\[
    \overline Z^{\mathrm{tr}}
    =
    \frac1N\sum_{i=1}^N Z_i^{\mathrm{tr}}.
\]
Since $Y_i=Z_i^{\mathrm{tr}}$ for every $i$, we have
$\overline Y=\overline Z^{\mathrm{tr}}$. Therefore,
\begin{align*}
    \frac{1}{N} \sum_{i=1}^N
    \paren{Y_i-\overline Y}
    \paren{Y_i-\overline Y}^\top
    \preceq
    C\paren{\frac1n+\beta}\mathbb I_d.
\end{align*}
Taking $C_A$ in Definition~\ref{def:A} to be a sufficiently large universal
constant, there exists a matrix $B\in\R^{d\times d}$ such that
\begin{align*}
    \frac{1}{N} \sum_{i=1}^N
    \paren{Y_i-\overline Y}
    \paren{Y_i-\overline Y}^\top
    =
    C_A\paren{\frac1n+\beta}\mathbb I_d - BB^\top.
\end{align*}
Thus the polynomial system $\mathsf A$ is satisfiable on this high-probability
event.

\textbf{Proof of sample complexity.}
We now set $\rho=\eps'$ in the failure probabilities of
Claim~\ref{cla:covariance-condition_part_2}. If $\beta n\leq 1$, then
\[
    \delta
    =
    d\exp\paren{-\Omega\paren{\frac{\eps' N}{d}}}.
\]
Thus it is enough that
\[
    N \geq \Omega\paren{\frac{d}{\eps'}\log\paren{d/\delta}}.
\]
Since
\[
    \eps'=\min\{\max\{\eps,n\beta\},1/10\},
\]
this requirement is implied, up to universal constants, by
\begin{align*}
    nN
    \geq
    \min\left\{
        \Omega\paren{\frac{d}{\eps/n}\log\paren{d/\delta}},
        \Omega\paren{\frac{d}{\beta}\log\paren{d/\delta}}
    \right\}.
\end{align*}

If $1<\beta n\leq d/\eps'$, then
\[
    \delta
    =
    d\exp\paren{-\Omega\paren{\frac{\beta n\eps' N}{d}}}.
\]
In this regime, $\beta n>1$, so $\eps'=1/10$. Hence it suffices that
\[
    nN
    \geq
    \Omega\paren{\frac{d}{\beta}\log\paren{d/\delta}}.
\]
Combining the two regimes gives the claimed sample-complexity condition.
\end{proof}

\subsection{Identifiability}

\label{sub:bounded_mean:identify}

\begin{lemma}[Identifiability of $\mathsf A$]
\label{lem:mean_close}
Under the setting of Problem~\ref{prob:bounded_mean}, if 
\begin{align*}
    nN \geq
    \min\left\{
    \Omega\left(\frac{d}{\eps/n}\log\left(\frac d\delta\right)\right),
    \Omega\left(\frac{d}{\beta}\log\left(\frac d\delta\right)\right)
    \right\},
\end{align*}
then, with probability at least
$1-\delta$, every pseudoexpectation $\widetilde{\mathbb E}$ satisfying the SoS
constraints $\mathsf A$ yields
\[
    \left\|
    \widetilde{\mathbb E}[\overline Y]-\mu
    \right\|_2
    =
    O\left(\sqrt{\frac{\eps}{n}}+\sqrt{\beta}\right).
\]
Moreover, such a pseudoexpectation can be computed in time
\[
    O(Nnd)+(Nd+d^2)^{O(1)}.
\]
\end{lemma}

\paragraph{Overview}

Now, we give a
proof overview, and then, we will present the formal proof. Our goal is to show that every feasible pseudoexpectation
satisfying $\mathsf A$ yields a mean estimate close to the true mean $\mu$,
namely
\[
    \left\|\wt \E[\overline Y]-\mu\right\|_2
    =
    O\left(\sqrt{\frac{\eps}{n}}+\sqrt{\beta}\right).
\]
With $\overline Z^{\mathrm{tr}}
    :=
    \frac1N\sum_{i=1}^N Z_i^{\mathrm{tr}}$, we decompose the error as the sum of $\left\|\wt \E[\overline Y]-\overline Z^{\mathrm{tr}}\right\|_2$, $\left\|
        \frac1N\sum_{i=1}^N
        \left(Z_i^{\mathrm{tr}}-\mu_i^{\mathrm{tr}}\right)
    \right\|_2$, $\left\|
        \frac1N\sum_{i=1}^N
        \left(\mu_i^{\mathrm{tr}}-\mu_i\right)
    \right\|_2$, and $\left\|
        \frac1N\sum_{i=1}^N
        \left(\mu_i-\mu\right)
    \right\|_2$.
We bound these four terms separately.

The first term measures how far an arbitrary feasible pseudo-solution can be from the
truncated clean empirical mean. This is exactly the role of our improved SoS
robust mean estimation lemma (Lemma~\ref{lem:sos_mean}). Since the observed
empirical means and the truncated clean empirical means differ on at most a
$2\eps'$-fraction of indices, and since both collections satisfy the covariance
constraint at scale $\frac1n+\beta$, Lemma~\ref{lem:sos_mean} gives
\[
    \left\|\wt \E[\overline Y]-\overline Z^{\mathrm{tr}}\right\|_2
    \leq
    O\left(
        \sqrt{\eps'\left(\frac1n+\beta\right)}
    \right).
\]
By the definition
$\eps'=\min\{\max\{\eps,n\beta\},1/10\}$, this term is bounded by $O\left(\sqrt{\frac{\eps}{n}}+\sqrt{\beta}\right)$.

The second term is the fluctuation of the truncated user-level empirical means around their own
expectations. Since each user-level empirical mean has covariance at scale
$\frac1n$, standard concentration gives, with high probability, $\norm{
        \frac1N\sum_{i=1}^N
        \left(Z_i^{\mathrm{tr}}-\mu_i^{\mathrm{tr}}\right)}_2
    \leq
    O\left(\sqrt{\frac{d}{nN}}\right)$.
Under the assumed sample complexity, this term is dominated by
\[
    O\left(\sqrt{\frac{\eps}{n}}+\sqrt{\beta}\right).
\]

The third term is the bias introduced by truncation. From our satisfiability proof (Claim~\ref{cla:covariance-condition_part_2}), we have already shown that for each
$i$, $\left\|\mu_i^{\mathrm{tr}}-\mu_i\right\|_2
    \leq
    O\left(\sqrt{\frac{\eps'}{n}}\right)$.
Therefore,
\[
    \left\|
        \frac1N\sum_{i=1}^N
        \left(\mu_i^{\mathrm{tr}}-\mu_i\right)
    \right\|_2
    \leq
    O\left(\sqrt{\frac{\eps'}{n}}\right)
    =
    O\left(\sqrt{\frac{\eps}{n}}+\sqrt{\beta}\right).
\]

The fourth term comes from the allowed user-level mean shift. Since every good user satisfies
$\|\mu_i-\mu\|_2\leq \sqrt{\beta}$, we have
\[
    \left\|
        \frac1N\sum_{i=1}^N
        \left(\mu_i-\mu\right)
    \right\|_2
    \leq
    \frac1N\sum_{i=1}^N
        \|\mu_i-\mu\|_2
    \leq
    \sqrt{\beta}.
\]

Combining the four bounds gives our desired error bound $\left\|\wt \E[\overline Y]-\mu\right\|_2
    =
    O\left(\sqrt{\frac{\eps}{n}}+\sqrt{\beta}\right)$.
This establishes identifiability: any feasible pseudoexpectation satisfying
$\mathsf A$ must output a mean close to the true mean. Finally, the running time
follows from the standard constant-degree SoS optimization guarantee applied to
the variables $\{W_i,Y_i\}_{i=1}^N$ and the slack matrix $B$.

\paragraph{Formal proof} Now, we present the formal proof.

\begin{proof}[Proof of Lemma~\ref{lem:mean_close}]
\textbf{Proof of Correctness.}
Condition on a realization of the truncated random variables $Z_i^{\mathrm{tr}}$, and define
\begin{align*}
    \overline Z^{\mathrm{tr}} = \frac{1}{N}\sum_{i=1}^N Z_i^{\mathrm{tr}}.
\end{align*}
By Lemma~\ref{lem:sos_mean}, applied to the polynomial system $\mathsf A$, the SoS algorithm outputs $\wt\E[\overline Y]$ satisfying
\begin{align}\label{eq:ovz-wtWovZ}
    \norm{\overline Z^{\mathrm{tr}} - \wt\E\bracket{\overline Y}}_2
    \leq O\paren{\sqrt{\eps'\paren{\frac{1}{n}+\beta}}}.
\end{align}
$C_A$ in Definition~\ref{def:A} only affect the universal constant hidden in the $O(\cdot)$ notation.
It remains to bound $\norm{\overline Z^{\mathrm{tr}}-\mu}_2$. We decompose
\begin{align}\label{eq:ovz-ovmu-clean-decomp}
    \norm{\overline Z^{\mathrm{tr}}-\mu}_2
    \leq & ~ \norm{\frac1N\sum_{i=1}^N \paren{Z_i^{\mathrm{tr}}-\mu_i^{\mathrm{tr}}}}_2 + \norm{\frac1N\sum_{i=1}^N \paren{\mu_i^{\mathrm{tr}}-\mu_i}}_2 + \norm{\frac1N\sum_{i=1}^N \paren{\mu_i-\mu}}_2.
\end{align}
For the first term of Eq.~\eqref{eq:ovz-ovmu-clean-decomp}, we define $U_i \eqdef Z_i^{\mathrm{tr}}-\mu_i^{\mathrm{tr}}$ and $X_i \eqdef \frac1N U_i$.
Then $\{X_i\}_{i=1}^N$ are independent mean-zero random vectors, and
\begin{align*}
    \sum_{i=1}^N X_i = \frac1N\sum_{i=1}^N \paren{Z_i^{\mathrm{tr}}-\mu_i^{\mathrm{tr}}}.
\end{align*}
Since $\cov\bracket{Z_i^{\mathrm{tr}}} \preceq \frac1n \mathbb I_d$, we have
\begin{align*}
    \E\bracket{\norm{U_i}_2^2}
    = & ~ \E\bracket{\norm{Z_i^{\mathrm{tr}} -\mu_i^{\mathrm{tr}}}_2^2}\\
    = & ~ \E\bracket{\mathrm{tr}\bracket{\paren{Z_i^{\mathrm{tr}} -\mu_i^{\mathrm{tr}}} \paren{Z_i^{\mathrm{tr}} -\mu_i^{\mathrm{tr}}}^\top}}\\
    = & ~ \mathrm{tr}\bracket{\E\bracket{\paren{Z_i^{\mathrm{tr}} -\mu_i^{\mathrm{tr}}} \paren{Z_i^{\mathrm{tr}} -\mu_i^{\mathrm{tr}}}^\top}}\\
    = & ~ \mathrm{tr}\bracket{ \cov\bracket{Z_i^{\mathrm{tr}}} } \\
    \leq & ~ \frac dn.
\end{align*}
Therefore, we have
\begin{align*}
    \sum_{i=1}^N \E\bracket{\|X_i\|_2^2}
    = & ~ \frac1{N^2}\sum_{i=1}^N \E\bracket{\|U_i\|_2^2} \\
    \leq & ~ \frac{d}{nN}.
\end{align*}
Moreover, by Claim~\ref{cla:covariance-condition_part_1_4} and Claim~\ref{cla:covariance-condition_part_3_4},
\begin{align*}
    \|U_i\|_2
    = & ~ \|Z_i^{\mathrm{tr}}-\mu_i^{\mathrm{tr}}\|_2 \\
    \leq & ~ \|Z_i^{\mathrm{tr}}-\mu_i\|_2 + \|\mu_i-\mu_i^{\mathrm{tr}}\|_2 \\
    \leq & ~ O\paren{\sqrt{\frac{d}{n\rho}}}.
\end{align*}
Thus, for all $i \in [N]$, we have
\begin{align*}
    \|X_i\|_2 \leq O\paren{\frac1N \sqrt{\frac{d}{n\rho}}}.
\end{align*}
Applying Fact~\ref{fac:vec_bernstein}, with probability at least $1-\delta$, we have
\begin{align*}
    \left\|
    \frac1N\sum_{i=1}^N
    \paren{Z_i^{\mathrm{tr}}-\mu_i^{\mathrm{tr}}}
    \right\|_2
    \leq
    O\paren{
    \sqrt{\frac{d\log(1/\delta)}{nN}}
    +
    \frac1N\sqrt{\frac{d}{n\rho}}\log(1/\delta)
    }.
\end{align*}
After setting $\rho=\eps'$, our sample-complexity condition $N \geq \Omega\paren{\frac{d}{\eps'}\log\paren{d/\delta}} > \log(1/\delta)/\rho$. Hence $\frac1N\sqrt{\frac{d}{n\rho}}\log(1/\delta)$ is absorbed by $\sqrt{\frac{d\log(1/\delta)}{nN}}$, and therefore
\begin{align}\label{eq:ovz-ovmu_1_simplified}
    \left\|
    \frac1N\sum_{i=1}^N
    \paren{Z_i^{\mathrm{tr}}-\mu_i^{\mathrm{tr}}}
    \right\|_2
    \leq
    O\paren{\sqrt{\frac{d\log(1/\delta)}{nN}}}.
\end{align}
For the second term of Eq.~\eqref{eq:ovz-ovmu-clean-decomp}, Claim~\ref{cla:covariance-condition_part_3_4}, applied with $\rho=\eps'$, gives
\begin{align}\label{eq:ovz-ovmu_2}
    \norm{\frac{1}{N}\sum_{i=1}^N
    \paren{\mu_i^{\mathrm{tr}}-\mu_i}}_2
    \leq O\paren{\sqrt{\frac{\eps'}{n}}}.
\end{align}
For the third term of Eq.~\eqref{eq:ovz-ovmu-clean-decomp}, the heterogeneity assumption gives
\begin{align}\label{eq:ovz-ovmu_3}
    \norm{\frac1N\sum_{i=1}^N
    \paren{\mu_i-\mu}}_2
    \leq \frac1N\sum_{i=1}^N \norm{\mu_i-\mu}_2
    \leq \sqrt{\beta}.
\end{align}
Combining Eq.~\eqref{eq:ovz-ovmu-clean-decomp}, Eq.~\eqref{eq:ovz-ovmu_1_simplified}, Eq.~\eqref{eq:ovz-ovmu_2}, and Eq.~\eqref{eq:ovz-ovmu_3}, we obtain
\begin{align*}
    \norm{\overline Z^{\mathrm{tr}}-\mu}_2
    \leq O\paren{\sqrt{\frac{d\log(1/\delta)}{nN}} + \sqrt{\frac{\eps'}{n}} + \sqrt{\beta}}.
\end{align*}
Again, by our sample complexity condition $nN \geq \Omega\paren{\frac{d}{\eps'/n}\log\paren{d/\delta}}$, we have
\begin{align*}
    \sqrt{\frac{d\log(1/\delta)}{nN}}
    \leq & ~ \sqrt{\frac{d\log(1/\delta)}{\frac{d}{\eps'/n}\log\paren{d/\delta}}} \\
    \leq & ~ \sqrt{\frac{1}{\frac{1}{\eps'/n}}} \\
    = & ~ \sqrt{\eps'/n},
\end{align*}
which implies
\begin{align*}
    \norm{\overline Z^{\mathrm{tr}}-\mu}_2
    \leq O\paren{\sqrt{\frac{\eps'}{n}} + \sqrt{\beta}}.
\end{align*}
Therefore, we have
\begin{align*}
    \norm{\mu - \wt\E\bracket{\overline Y}}_2
    \leq &~ \norm{\mu-\overline Z^{\mathrm{tr}}}_2
    + \norm{\overline Z^{\mathrm{tr}}-\wt\E\bracket{\overline Y}}_2\\
    \leq &~ O\paren{\sqrt{\frac{\eps'}{n}}+\sqrt{\beta}} + O\paren{\sqrt{\eps'\paren{\frac1n+\beta}}}\\
    \leq &~ O\paren{\sqrt{\frac{\eps'}{n}}+\sqrt{\beta}}.
\end{align*}
In the last line we used $\eps'\leq 1/10$, so $\sqrt{\eps'\beta}\leq O(\sqrt\beta)$.
Moreover,
\begin{align*}
    \sqrt{\frac{\eps'}{n}}
    \leq
    O\paren{\sqrt{\frac{\eps}{n}}+\sqrt\beta},
\end{align*}
because $\eps'\leq \max\{\eps,n\beta\}$. 
Hence, we have
\begin{align*}
    \norm{\mu - \wt\E\bracket{\overline Y}}_2
    \leq O\paren{\sqrt{\frac{\eps}{n}}+\sqrt\beta},
\end{align*}
which completes the correctness proof.

    \textbf{Proof of running time.} Regarding the running time, it takes $O(Nnd)$ time to compute the
empirical mean $x_i \in \R^d$ of each user $i \in [N]$, each of which is
computed by averaging $n$ samples. We then solve the SoS relaxation associated
with the polynomial system $\mathsf A$ in Definition~\ref{def:A}. By
\cite{Schramm2022ProofsToAlgs}, if a polynomial system has at most $q$
variables and $a$ polynomial axioms, then a degree-$k$ SoS relaxation can be
solved in time $(qa)^{O(k)}$.

We now count these quantities for $\mathsf A$. The variables are $Y_1,\ldots,Y_N \in \R^d$, $W_1,\ldots,W_N \in \R$, and $B \in \R^{d\times d}$. Thus, the number of scalar variables is $q = Nd + N + d^2$.

The scalar polynomial axioms are as follows. The constraint
$W_i(Y_i-x_i)=0$ contributes $O(Nd)$ scalar equality constraints, the Boolean
constraints $W_i^2=W_i$ contribute $O(N)$ scalar equality constraints, the
cardinality constraint $\sum_{i=1}^N W_i=(1-2\eps')N$ contributes one scalar
equality constraint, and the matrix constraint $\frac{1}{N}\sum_{i=1}^N (Y_i-\ov Y)(Y_i-\ov Y)^\top = C_A\paren{\frac{1}{n}+\beta}\mathbb{I}_d - BB^\top$ contributes at most $O(d^2)$ scalar equality constraints. Therefore the total
number of scalar polynomial axioms is at most $a = O(Nd + d^2)$.

Our SoS degree $k = 6$
used in the algorithm is a fixed constant. Hence the SoS step runs in time
\begin{align*}
    (qa)^{O(k)} = (Nd + d^2)^{O(1)}
\end{align*}

Combining everything together, the total running time is
\begin{align*}
    O(Nnd) + (Nd + d^2)^{O(1)}.
\end{align*}
Thus, Algorithm~\ref{alg:sos_bounded} runs in polynomial time.
\end{proof}

%% file: 2_collaborate_with_arbitrary_corruption.tex
\section{Learning From Doubly Corrupted Batches}
\label{sec:arbitrary_alpha}

In this section, we focus on the adversarial model, where
the goal is again to estimate an underlying mean parameter from data contributed
by $N$ users, each of whom provides a batch of $n$ samples. At least a
$(1-\eps)$-fraction of the users are considered ``good''. For each good user,
the clean samples are drawn i.i.d.\ from the target distribution, but an
$\alpha$-fraction of the samples in that user's batch may be adversarially
corrupted. The remaining $\eps$-fraction of users may be entirely adversarial and
may provide arbitrary batches.

\begin{restatable}[Adversarial]{problem}{ProbArbitraryAlpha}\label{prob:arbitrary_alpha}
Consider parameters $\alpha, \eps \in (0, 0.5)$, a target distribution $P$ over $\mathbb{R}^d$ with mean $\mu$ and covariance $\Sigma \preceq \mathbb{I}_d$, and $N$ users, each providing a batch of $n$ data points in $\mathbb{R}^d$. We are guaranteed that at least a $(1-\eps)$-fraction of these users are good; each good user $i \in \mathcal{G}$ draws $n$ i.i.d. samples from $P$. Then an $\alpha$-fraction of the $n$ samples are replaced by arbitrary, adversarially chosen points. This is in addition to the 
entirely adversarial datasets of the remaining $\eps$-fraction of users. Our goal is to robustly estimate the true mean $\mu$ in $\ell_2$ distance.
\end{restatable}

Equivalently, the observed data are generated through a two-level strong
contamination process. One may imagine that all samples of all good users are
first drawn independently from the target distribution $P$ with mean $\mu$ and
covariance bounded by $\mathbb I_d$. Then an adversary inspects the entire
collection of samples. The adversary may replace the entire batches of an
$\eps$-fraction of users, and, for every remaining good user, may additionally
replace up to an $\alpha$-fraction of that user's samples by arbitrary points.
We can only observe the final, possibly corrupted, batches.

Given these observed batches as input, we design a novel SoS-based algorithm
solving Problem~\ref{prob:arbitrary_alpha}. We show that the algorithm runs in
polynomial time and achieves the optimal $\ell_2$-error rate
    $O\left(\sqrt{\alpha}\right)$ when $\alpha n
    \ge
    4\log\paren{\frac{4}{\eps}}$.

\begin{theorem}
\label{thm:arbitrary_alpha_upper}
    Let $c\in(0,0.5)$ be a universal constant. Given $\eps\in(0,c)$, $\alpha\in(0,c)$, and $n=o(\dims)$ samples from each of $N\in\Z_+$ users satisfying
    \[
        nN \geq \Omega\paren{\frac{d}{\alpha + \frac{\eps}{n}}\log\paren{d/\delta}},
    \]
    there exists an algorithm with running time
    \[
        (Nnd+d^2)^{O(1)}
    \]
    that solves Problem~\ref{prob:arbitrary_alpha}. The algorithm outputs
    $\wh \mu \in \R^d$ satisfying
    \[
        \norm{\wh\mu-\mu}_2
        =
        O\paren{\sqrt{\alpha + \frac{\eps}{n}}}
    \]
    with probability at least $1-\delta$.
\end{theorem}

We first explain the main idea before stating the two key lemmas used in the proof. A naive approach is to ignore the
two-level structure of the corruption and treat all $Nn$ samples as jointly
$(\eps+\alpha)$-corrupted. This would allow one to apply a standard robust mean
estimation algorithm (Theorem~\ref{thm:rme_sos_bounded_covariance}), but it would only give the weaker error
$O(\sqrt{\eps+\alpha})$. This approach is suboptimal because it does not use
the fact that most users are only partially corrupted: for a good user, only an
$\alpha$-fraction of its samples are corrupted, while only an $\eps$-fraction of
users are completely unreliable.

Our algorithm exploits this hierarchical structure and seeks ``good'' samples from ``good'' users. At the sample level, we
first seek a large set of samples whose empirical
covariance is bounded. This step removes gross outliers caused by sample-level
corruption. At the user level, we then aggregate the cleaned user-level means
and enforce a second bounded-covariance condition across users. The second step
is crucial because, after removing the sample-level corruptions, the cleaned
mean of a good user's batch has covariance on the order of $1/n$. Thus the
effect of the $\eps$-fraction of entirely corrupted users is reduced by the
batch size, leading to the term $\sqrt{\eps/n}$, while the within-user
corruption contributes the term $\sqrt{\alpha}$.

To solve this problem, we use an SoS-based algorithm. Instead of directly
formalizing the problem as a combinatorial two-level subset-selection problem,
we propose a system of polynomial constraints (which will be defined in Definition~\ref{def:B}) whose purpose is to encode both
sample-level and user-level refinement. The polynomial system introduces
SoS indeterminates representing candidate cleaned samples and candidate
cleaned user-level means. In the satisfiability proof, these variables are
assigned by truncated clean random variables.
We need to show:
\begin{enumerate}
    \item \textbf{Satisfiability}
    (Lemma~\ref{lem:bounded_covariance_5}): With probability at least
    $1-\delta/2$, there exists an assignment of all the SoS indeterminates
    such that all constraints are
    satisfied.

    \item \textbf{Identifiability}
    (Lemma~\ref{lem:mean_close_5}): With probability at least $1-\delta/2$,
    any feasible pseudoexpectation satisfying the constraints yields an accurate estimate
    $\widehat \mu \in \R^d$ of the true mean $\mu \in \R^d$.
\end{enumerate}

\paragraph{Organization of the Remainder of This Section}
In the remainder of this section, we make this outline formal. In
\Cref{sub:arbitrary_alpha:alg}, we define the SoS variables and polynomial
constraints, state the two key lemmas---{satisfiability}
(Lemma~\ref{lem:bounded_covariance_5}) and {identifiability}
(Lemma~\ref{lem:mean_close_5})---and provide high-level proof overviews for
both. Assuming these two lemmas, we then prove the main result of this section
(Theorem~\ref{thm:arbitrary_alpha_upper}). In
\Cref{sub:arbitrary_alpha:satisfy,sub:arbitrary_alpha:identify}, we prove
{satisfiability} (Lemma~\ref{lem:bounded_covariance_5}) and {identifiability}
(Lemma~\ref{lem:mean_close_5}), respectively.

\subsection{Notation}
\label{sub:adversarial:notation}

Let $\{x_{i,j}\}_{i\in[N],j\in[n]} \subset \R^d$ be the observed samples after
the two-level corruption process. Lowercase symbols such as $x_{i,j}$ denote samples. For each user $i$ and sample index $j$, let
$Z_{i,j}\in\R^d$ denote the clean sample random variable before any
sample-level or user-level corruption, and let $z_{i,j}$ denote its realized
value. The clean user-level empirical mean random variable is
    $Z_i := \frac1n \sum_{j=1}^n Z_{i,j}$.
$Z_{i,j}^{\mathrm{tr}}$ denotes the truncated version of the clean sample random
variable $Z_{i,j}$. We write
    $Z_i^{\mathrm{tr}}
    :=
    \frac1n\sum_{j=1}^n Z_{i,j}^{\mathrm{tr}}$
for the empirical mean of the sample-level truncated random variables. If a
second truncation is applied at the user level, we denote the resulting random
variable by $Z_i^{\trtwo}$, so $\ov Z^{\trtwo} = \frac{1}{N}\sum_{i = 1}^N Z_i^{\trtwo}$. Lowercase symbols such as
$z_{i,j}^{\mathrm{tr}}$, $z_i^{\mathrm{tr}}$, and
$z_i^{\trtwo}$ denote the corresponding realized values.

Let
$\{\mathsf{Z}_{i,j}\}_{i\in[N],j\in[n]} \subset \R^d$ be the SoS variables
representing candidate cleaned samples, and define the SoS user-level empirical
mean
    $\mathsf{Z}_i := \frac1n\sum_{j=1}^n \mathsf{Z}_{i,j}$.
Let $\{\mathsf{Y}_i\}_{i=1}^N \subset \R^d$ be the SoS variables representing
candidate cleaned user-level means. The variables $W_{i,j}$ and $U_i$ are also
SoS indeterminates: $W_{i,j}$ indicates whether sample $(i,j)$ is selected as
clean, and $U_i$ indicates whether user $i$ is selected as globally
uncorrupted. The SoS variables $\mathsf{Z}_{i,j}$,
$\mathsf{Z}_i$, and $\mathsf{Y}_i$ are assigned by the
corresponding truncated clean random variables.

\subsection{Algorithmic Design and Proof Overview}
\label{sub:arbitrary_alpha:alg}

In \Cref{sub:arbitrary_alpha:alg:def}, we define the SoS variables and polynomial
constraints. In \Cref{sub:arbitrary_alpha:alg:satisfy}, we state the
satisfiability result (Lemma~\ref{lem:bounded_covariance_5}) and give a proof
overview. In \Cref{sub:arbitrary_alpha:alg:identify}, we state the
identifiability result (Lemma~\ref{lem:mean_close_5}) and give a proof overview.
Finally, in \Cref{sub:arbitrary_alpha:alg:proof}, we combine these results to
prove Theorem~\ref{thm:arbitrary_alpha_upper}.

\subsubsection{SoS Variables and Polynomial Constraints}
\label{sub:arbitrary_alpha:alg:def}

Our algorithm is based on the following two-level polynomial system. 

\begin{definition}[Polynomial system $\mathsf B$]
\label{def:B}
Let the notation be defined as in \Cref{sub:adversarial:notation}.
Let $\tau := \alpha/\eps$. Let $C_3, C_4 > 0$ be sufficiently large universal constants. The polynomial system $\mathsf B$ consists of the
following constraints over SoS indeterminates
\[
    \{\mathsf Z_{i,j}\}_{i\in[N],j\in[n]} \subset \R^d,\qquad
    \{\mathsf Z_i\}_{i=1}^N \subset \R^d,\qquad
    \{\mathsf Y_i\}_{i=1}^N \subset \R^d,
\]
\[
    \{W_{i,j}\}_{i\in[N],j\in[n]} \subset \R,\qquad
    \{U_i\}_{i=1}^N \subset \R,\qquad
    B_1,B_2\in\R^{d\times d}.
\]
\begin{align}
    \intertext{\textbf{Crude sample-level refinement:}}
    \mathsf Z_i
    &=
    \frac1n\sum_{j=1}^n \mathsf Z_{i,j},
    && \forall i\in[N],
    \label{eq:C_5}
    \\
    \overline{\mathsf Z}
    &=
    \frac1{Nn}\sum_{i=1}^N\sum_{j=1}^n \mathsf Z_{i,j},
    \label{eq:C_2}
    \\
    W_{i,j}^2
    &=
    W_{i,j},
    && \forall i\in[N],\ j\in[n],
    \label{eq:C_1}
    \\
    \sum_{i=1}^N\sum_{j=1}^n U_iW_{i,j}
&=
(1-\Theta\paren{\alpha})n\sum_{i=1}^N U_i,
\label{eq:C_8}
    \\
    W_{i,j}(\mathsf Z_{i,j}-x_{i,j})
    &=
    0,
    && \forall i\in[N],\ j\in[n],
    \label{eq:C_3}
    \\
    \frac1{Nn}\sum_{i=1}^N\sum_{j=1}^n
    (\mathsf Z_{i,j}-\overline{\mathsf Z})
    (\mathsf Z_{i,j}-\overline{\mathsf Z})^\top
    &=
    C_3\mathbb I_d-B_1B_1^\top,
    \label{eq:C_7}
    \\
    \intertext{\textbf{User-level refinement:}}
    U_i^2
    &=
    U_i,
    && \forall i\in[N],
    \label{eq:C_4}
    \\
    \overline{\mathsf Y}
    &=
    \frac1N\sum_{i=1}^N \mathsf Y_i,
    \label{eq:C_10}
    \\
    \sum_{i=1}^N U_i
    &=
    (1 - \Theta\paren{\eps})N,
    \label{eq:C_9}
    \\
    U_i(\mathsf Y_i-\mathsf Z_i)
    &=
    0,
    && \forall i\in[N],
    \label{eq:C_6}
    \\
    \frac1N\sum_{i=1}^N
    (\mathsf Y_i-\overline{\mathsf Y})
    (\mathsf Y_i-\overline{\mathsf Y})^\top
    &=
    C_4 \left(\frac1n+\tau\right)\mathbb I_d-B_2B_2^\top.
    \label{eq:C_11}
\end{align}
\end{definition}

Generally speaking, the crude sample-level refinement removes sample-level
outliers within selected users. The variable $W_{i,j}$ indicates whether sample
$(i,j)$ is selected as clean. The agreement constraint Eq.~\eqref{eq:C_3}
says that, if $W_{i,j}=1$, then the candidate cleaned sample
$\mathsf Z_{i,j}$ must agree with the observed sample $x_{i,j}$.
The global cardinality constraint Eq.~\eqref{eq:C_8} says that, among all
samples belonging to users selected at the user level, the system certifies a
\((1-\Theta(\alpha))\)-fraction of sample--user pairs. The
\(\Theta(\alpha)\) loss accounts for both the samples changed by the adversary
and the clean samples modified by the sample-level truncation. Importantly,
this constraint controls the number of certified samples globally across the
selected users; it does not require every selected user to certify the same
fraction of its samples. 
The sample-level covariance constraint
Eq.~\eqref{eq:C_7}
encodes the spectral bound
\[
    \frac1{Nn}\sum_{i=1}^N\sum_{j=1}^n
    (\mathsf Z_{i,j}-\overline{\mathsf Z})
    (\mathsf Z_{i,j}-\overline{\mathsf Z})^\top
    \preceq
    C_3\mathbb I_d .
\]
Here $C_3>0$ is a sufficiently large universal constant. It absorbs the
constant slack coming from the matrix Chernoff bound in the satisfiability
proof.

The user-level refinement works with the cleaned user-level means
$\mathsf Y_i$. The constraint
Eq.~\eqref{eq:C_6}
says that, for a selected user, the cleaned user-level mean $\mathsf Y_i$ must
agree with the average $\mathsf Z_i$ of its candidate cleaned samples. The
constraint
Eq.~\eqref{eq:C_9}
forces the system to select a large collection of users. The $\Theta(\eps)$ loss accounts for both the $\eps$-fraction of entirely adversarial users and the additional $O(\eps)$-fraction of users that may be removed by the user-level truncation.
Finally, the user-level covariance constraint
Eq.~\eqref{eq:C_11}
encodes the spectral bound
\[
    \frac1N\sum_{i=1}^N
    (\mathsf Y_i-\overline{\mathsf Y})
    (\mathsf Y_i-\overline{\mathsf Y})^\top
    \preceq
    C_4\left(\frac1n+\tau\right)\mathbb I_d .
\]
Here $C_4>0$ is a sufficiently large universal constant. The term $1/n$
reflects the variance reduction from averaging $n$ samples within each user,
while the slack parameter $\tau=\alpha/\eps$ is chosen so that the matrix
Chernoff failure probability is small under the desired sample complexity,
without changing the final error rate.

\begin{algorithm}[!ht]
\caption{SoS for solving adversarial (Problem~\ref{prob:arbitrary_alpha}).}
\label{alg:sos_arbitrary_alpha}
\begin{algorithmic}[1]
\State \textbf{Input}: $n$ samples from each user, $\eps \in \paren{0, \frac{1}{18}}$, $\alpha \in \paren{0, \frac{1}{90}}$ satisfying $\eps + 5\alpha < \frac{1}{18}$, and polynomial system $\mathsf{B}$ (crude and user-level refinement, see the full definition in Definition~\ref{def:B}).
\State Compute a pseudoexpectation $\wt{\mathbb{E}}$ satisfying the constraints of the polynomial system $\mathsf{B}$.
\State \textbf{Output:} $\wh \mu := \wt{\mathbb{E}}\bracket{\ov Y} \in \R^d$.
\end{algorithmic}
\end{algorithm}

The SoS algorithm (Algorithm~\ref{alg:sos_arbitrary_alpha}) is now simple:
compute a pseudoexpectation satisfying $\mathsf B$, and output
    $\widehat \mu := \widetilde{\mathbb E}\left[\overline{\mathsf Y}\right]$.
To prove correctness, sample complexity, and running time, we analyze the
satisfiability and identifiability of the polynomial system $\mathsf B$.

\subsubsection{Proof Overview of Satisfiability}
\label{sub:arbitrary_alpha:alg:satisfy}

\begin{lemma}[Satisfiability of $\mathsf B$]
\label{lem:bounded_covariance_5}
Under the setting of Problem~\ref{prob:arbitrary_alpha}, with probability at
least $1-\delta$, there exists an assignment to the variables
\[
    \{\mathsf Z_{i,j}\}_{i\in[N],j\in[n]} \subset \R^d,\qquad
    \{\mathsf Z_i\}_{i=1}^N \subset \R^d,\qquad
    \{\mathsf Y_i\}_{i=1}^N \subset \R^d,
\]
\[
    \{W_{i,j}\}_{i\in[N],j\in[n]} \subset \R,\qquad
    \{U_i\}_{i=1}^N \subset \R,\qquad
    B_1,B_2\in\R^{d\times d}
\]
such that all constraints in $\mathsf B$ are satisfied, as long as
\[
    nN
    \geq
    \Omega\left(\frac d\alpha\log\left(\frac d\delta\right)\right).
\]
\end{lemma}

The formal proof is deferred to \Cref{sub:arbitrary_alpha:satisfy}. We construct a feasible assignment using suitable
truncated clean random variables. We first truncate $Z_{i,j}$ around the true mean
$\mu$, obtaining $Z_{i,j}^{\mathrm{tr}}$. This sample-level truncation is just like the truncation step in \Cref{sub:rme_SoS:true}, where we modify
only a small fraction of the clean samples to ensure the satisfiability of \textbf{crude sample-level refinement} in Definition~\ref{def:B}. 

The technical novelty here is our second truncation step. 
We form the truncated empirical mean
    $Z_i^{\mathrm{tr}}
    =
    \frac1n\sum_{j=1}^n Z_{i,j}^{\mathrm{tr}}$.
Since $Z_i^{\mathrm{tr}}$ is an average of $n$ bounded-covariance random
variables, its covariance is at scale $1/n$. To ensure the satisfiability of \textbf{user-level refinement} in Definition~\ref{def:B}, we apply a second truncation step. Namely, we truncate $Z_i^{\mathrm{tr}}$ around its expectation,
obtaining $Z_i^{\trtwo}$. 
However, this second truncation step cannot be immediately derived from our prior approach, such as Claim~\ref{cla:covariance-condition_part_2_6} under \Cref{sub:rme_SoS:true}. To be more specific, using Claim~\ref{cla:covariance-condition_part_2_6}, we can get with probability at least
$1-d\exp(-\Omega(N\eps/d))$, 
\begin{align*}
    \frac1N\sum_{i = 1}^N
    \paren{Z_i^{\trtwo} - \barZtrtwo}
    \paren{Z_i^{\trtwo} -\barZtrtwo}^\top
    \preceq
    C \frac{1}{n} \mathbb I_d.
\end{align*}
Although it helps with satisfying Eq.~\eqref{eq:C_11} by assigning $Z_i^{\trtwo}$ to the SoS variable $\mathsf{Y}_i$, we need too many samples to ensure $d\exp(-\Omega(N\eps/d)) \leq \delta$: $N \geq \Omega\paren{d / \eps \log\paren{d / \delta}}$. Thus, the total sample complexity is 
\begin{align*}
    nN \geq \Omega\paren{\frac{d}{\eps / n} \log\paren{d / \delta}}
\end{align*}
Our target accuracy is $\norm{\wh \mu - \mu}_2 \leq O\paren{\sqrt{\alpha}}$. Therefore, this sample complexity is sub-optimal as $\alpha \geq \frac{\eps}{n}$. 
To ensure optimal sample complexity up to logarithmic factors, we make some compromise to the tightness of the bound:
\[
    \frac1N\sum_{i = 1}^N
    \paren{Z_i^{\trtwo} - \barZtrtwo}
    \paren{Z_i^{\trtwo} -\barZtrtwo}^\top
    \preceq
    C_4 \left(\frac1n+\tau\right)\mathbb I_d .
\]
The choice $\tau=\alpha/\eps$ makes the matrix Chernoff failure probability
small under the desired total sample complexity
$nN\ge \Omega((d/\alpha)\log(d/\delta))$, and it will not negatively influence the accuracy $\norm{\wh \mu - \mu}_2 \leq O\paren{\sqrt{\alpha}}$ (which will be informally explained in \Cref{sub:arbitrary_alpha:alg:identify} and formally proved in \Cref{sub:arbitrary_alpha:identify}).

\subsubsection{Proof Overview of Identifiability}
\label{sub:arbitrary_alpha:alg:identify}

\begin{lemma}[Identifiability of $\mathsf B$]
\label{lem:mean_close_5}
Under the setting of Problem~\ref{prob:arbitrary_alpha}, if
\[
    nN
    \geq
    \Omega\left(\frac d\alpha\log\left(\frac d\delta\right)\right),
\]
then, with probability at least $1-\delta$, every pseudoexpectation
$\widetilde{\mathbb E}$ satisfying the SoS constraints $\mathsf B$ yields
\[
    \left\|
    \widetilde{\mathbb E}[\overline{\mathsf Y}]-\mu
    \right\|_2
    =
    O\left(\sqrt{\alpha}\right).
\]
Moreover, such a pseudoexpectation can be computed in polynomial time.
\end{lemma}

The full proof is given in \Cref{sub:arbitrary_alpha:identify}. Now, we give a
proof overview. 
Our goal is to show that every feasible pseudoexpectation
satisfying $\mathsf B$ yields a mean estimate close to the true mean $\mu$,
namely
\[
    \left\|
    \widetilde{\mathbb E}[\overline{\mathsf Y}]-\mu
    \right\|_2
    =
    O\left(\sqrt{\frac{\eps}{n}}+\sqrt{\alpha}\right).
\]

Let
\[
    \widetilde \mu_i^{\trtwo}
    :=
    z_i^{\trtwo},
    \qquad
    \widetilde \mu^{\trtwo}
    :=
    \frac1N\sum_{i=1}^N \widetilde \mu_i^{\trtwo}
    =
    \frac1N\sum_{i=1}^N z_i^{\trtwo} .
\]
Here $z_i^{\trtwo}$ is the realized value of the twice-truncated
user-level random variable $Z_i^{\trtwo}$. We first prove that
$\overline{\mathsf Y}$ is close to $\widetilde \mu^{\trtwo}$ in the
SoS sense. The distance is decomposed according to whether a user is both
selected by the SoS solution and belongs to the deterministic reference set of
good, untruncated users. Let \(\mathcal I_i\) be the deterministic indicator that user \(i\) is
genuinely good and satisfies
\(z_i^{\trtwo}=z_i^{\mathrm{tr}}\). Then we decompose
    $\overline{\mathsf Y}-\widetilde \mu^{\trtwo}$
into two contributions:
\[
    \frac1N\sum_{i=1}^N
    (1-U_i\mathcal I_i)(\mathsf Y_i-\widetilde\mu_i^{\trtwo})
    \qquad\text{and}\qquad
    \frac1N\sum_{i=1}^N
    U_i\mathcal I_i(\mathsf Y_i-\widetilde\mu_i^{\trtwo}).
\]

The first term corresponds to users that are either not certified by the SoS
solution or are not truly good. This term is controlled by the user-level
bounded covariance constraint. Since at most an $\eps$-fraction of users are
bad, and since the selected user-level means have empirical covariance at scale
$\frac1n+\tau$, the SoS Cauchy--Schwarz inequality gives a contribution of
order
\[
    O\left(\sqrt{\eps\left(\frac1n+\tau\right)}\right)
    =
    O\left(\sqrt{\frac{\eps}{n}}+\sqrt{\alpha}\right),
\]
where we used $\tau=\alpha/\eps$.

The second term is the main new difficulty compared with the bounded
mean-shift model. Even if a user is good, an $\alpha$-fraction of its samples
may still be corrupted. Therefore $\mathsf Y_i$ need not agree with the clean
empirical mean $\widetilde\mu_i$ sample by sample. The sample-level constraints
are designed precisely to control this contribution. 
The global sample-count constraint, together with the global bound on the number of disagreements between $z_{i,j}^{\tr}$ and $x_{i,j}$ among reference users, implies that the total surviving bad mass is $O\paren{\alpha}$.

Moreover, the agreement constraint Eq.~\eqref{eq:C_3} implies that whenever
\(W_{i,j}=1\) and \(z_{i,j}^{\tr}=x_{i,j}\), we have
\(\mathsf Z_{i,j}=z_{i,j}^{\tr}\). Thus, the discrepancy
\(\mathsf Z_{i,j}-z_{i,j}^{\tr}\) can be nonzero only on pairs that are
either not selected by \(W_{i,j}\) or for which the observed sample disagrees
with the truncated clean reference.

Applying the SoS Cauchy--Schwarz inequality, this second contribution is
bounded by the product of two factors. The first factor is the surviving bad
mass, which is at most $O(\alpha)$. The second factor is controlled using the
sample-level bounded covariance constraint
Eq.~\eqref{eq:C_7}
together with the user-level covariance constraint and the concentration of the
clean empirical averages. This yields a contribution of order
$O(\sqrt{\alpha})$.

Combining these two bounds gives
\[
    \left\|
    \overline{\mathsf Y}-\widetilde \mu^{\trtwo}
    \right\|_2
    =
    O\left(\sqrt{\frac{\eps}{n}}+\sqrt{\alpha}\right)
\]
in the SoS proof system. Finally, the twice-truncated empirical mean
$\widetilde \mu^{\trtwo}$ is close to the true mean $\mu$ by the
two truncation-bias bounds and vector Bernstein applied to the bounded
variables $Z_i^{\trtwo}-\mu^{\trtwo}$. Under the assumed
sample complexity $nN\ge \Omega((d/\alpha)\log(d/\delta))$, this gives
\[
    \|\widetilde \mu^{\trtwo}-\mu\|_2
    =
    O(\sqrt{\alpha}).
\]
Therefore,
\[
    \left\|
    \widetilde{\mathbb E}[\overline{\mathsf Y}]-\mu
    \right\|_2
    =
    O\left(\sqrt{\frac{\eps}{n}}+\sqrt{\alpha}\right).
\]
This establishes identifiability: any feasible pseudoexpectation satisfying
$\mathsf B$ must output a mean close to the true mean. Finally, the running time
follows from the standard constant-degree SoS optimization guarantee applied to
the variables
\[
    \{\mathsf Z_{i,j},W_{i,j}\}_{i\in[N],j\in[n]},
    \qquad
    \{\mathsf Y_i,\mathsf Z_i,U_i\}_{i=1}^N,
    \qquad
    B_1,B_2.
\]

\subsubsection{Proof of Theorem~\ref{thm:arbitrary_alpha_upper}}
\label{sub:arbitrary_alpha:alg:proof}

By assuming that {satisfiability} (Lemma~\ref{lem:bounded_covariance_5}) and
{identifiability} (Lemma~\ref{lem:mean_close_5}) hold, we prove
Theorem~\ref{thm:arbitrary_alpha_upper}.

\begin{proof}[Proof of Theorem~\ref{thm:arbitrary_alpha_upper}]
We apply Lemma~\ref{lem:bounded_covariance_5} and
Lemma~\ref{lem:mean_close_5} with failure probability $\delta/2$. Since
\[
    \log\left(\frac d{\delta/2}\right)
    =
    O\left(\log\left(\frac d\delta\right)\right),
\]
the assumed sample complexity still implies the hypotheses of both lemmas.

By Lemma~\ref{lem:bounded_covariance_5}, with probability at least
$1-\delta/2$, the polynomial system $\mathsf B$ is satisfiable. Hence the SoS
relaxation has a feasible pseudoexpectation. By Lemma~\ref{lem:mean_close_5},
with probability at least $1-\delta/2$, every pseudoexpectation satisfying
$\mathsf B$ yields
\[
    \left\|
    \widetilde{\mathbb E}[\overline{\mathsf Y}]-\mu
    \right\|_2
    =
    O\left(\sqrt{\frac{\eps}{n}}+\sqrt{\alpha}\right).
\]
Therefore, by a union bound, with probability at least $1-\delta$, both events
hold simultaneously. The algorithm outputs
\[
    \widehat \mu
    :=
    \widetilde{\mathbb E}[\overline{\mathsf Y}],
\]
and hence satisfies the desired error guarantee. Under the parameter regime of
Problem~\ref{prob:arbitrary_alpha}, the term $\sqrt{\eps/n}$ is dominated by
$\sqrt{\alpha}$, so the bound becomes $O(\sqrt{\alpha})$.
\end{proof}

\subsection{Satisfiability}
\label{sub:arbitrary_alpha:satisfy}

In this section, we give a formal proof for Lemma~\ref{lem:bounded_covariance_5}, showing that our polynomial system $\mathsf{B}$ is satisfiable. We prove this under Lemma~\ref{lem:C_7} (the satisfiability of \textbf{crude sample-level refinement} in polynomial system $\mathsf{B}$) and Lemma~\ref{lem:C_11} (the satisfiability of \textbf{user-level refinement} in polynomial system $\mathsf{B}$).
\begin{lemma}\label{lem:C_7}
    Let the notation be defined as in \Cref{sub:adversarial:notation}. Let $\delta \in \paren{0, 0.1}$ be the failure probability. Then, with at least $nN 
    \geq \Omega\paren{\frac{d}{\alpha} \log\paren{d / \delta}}$ total number of samples, we can find a truncated random variable $Z_{i,j}^{\mathrm{tr}} \in \R^d$ such that
    \begin{enumerate}
        \item $\frac{1}{nN}\sum_{i = 1}^N \sum_{j = 1}^n \indic{Z_{i, j}\ne Z_{i, j}^{\mathrm{tr}}} < \alpha$,
        \item with $\mu^{\mathrm{tr}} = \E\bracket{Z_{i,j}^{\mathrm{tr}}}$ and $\mu = \E\bracket{Z_{i,j}}$, we have $\norm{\mu^{\mathrm{tr}} - \mu}_2 \leq O\paren{\sqrt{\alpha}}$, and
        \item the sample-level covariance constraint Eq.~\eqref{eq:C_7} holds if $\mathsf Z_{i,j}$ and $\mathsf Z_i$ are assigned the realized values of $Z_{i,j}^{\mathrm{tr}}$ and $Z_i^{\mathrm{tr}}$, respectively.
    \end{enumerate}
\end{lemma}
\begin{proof}
We define
\begin{align}
\label{eq:rv_C_7}
   Z_{i,j}^{\mathrm{tr}} :=
   \begin{cases}
   Z_{i,j}, & \text{if } \norm{Z_{i,j} - \mu}_2 \le 2\sqrt{\frac{d}{\alpha}},\\
   \mu, & \text{otherwise.}
   \end{cases}
\end{align}
\textbf{Parts 1--3} respectively follow from Claim~\ref{cla:covariance-condition_part_1}, Claim~\ref{cla:covariance-condition_part_3}, and Claim~\ref{cla:covariance-condition_part_2_6} by setting $\gamma = 1$, $\eps$ to be $\alpha$, $Z_j$ to be $Z_{i, j}$, $Z_j^{\mathrm{tr}}$ to be $Z_{i, j}^{\mathrm{tr}}$, and $n$ to be $nN$.
\end{proof}

\begin{lemma}\label{lem:C_11}
    Let the notation be defined as in Lemma~\ref{lem:C_7}. Then, with at least $nN 
    \geq \Omega\paren{\frac{d}{\alpha} \log\paren{d / \delta}}$ total number of samples, we can find a truncated random variable $Z_i^{\trtwo} \in \R^d$ such that
    \begin{enumerate}
        \item $\frac1N\sum_{i=1}^N \indic{Z_i^{\trtwo} \ne Z_i^{\mathrm{tr}}} < \eps$,
        \item with $\mu^{\mathrm{tr}} = \E\bracket{Z_{i,j}^{\mathrm{tr}}}$ and $\mu^{\trtwo} = \E\bracket{Z_i^{\trtwo}}$, we have $\norm{\mu^{\mathrm{tr}} - \mu^{\trtwo}}_2 \leq O\paren{\sqrt{\eps / n}}$, and
        \item the user-level covariance constraint Eq.~\eqref{eq:C_11} holds if $\mathsf Y_i$ is assigned the realized value of $Z_i^{\trtwo}$.
    \end{enumerate}
\end{lemma}
\begin{proof}
We further make use of the sample-level truncated random variables $Z_{i,j}^{\mathrm{tr}}$ from Lemma~\ref{lem:C_7} (see Eq.~\eqref{eq:rv_C_7}). 
For each user $i \in [N]$, we define
\begin{align*}
    Z_i^{\mathrm{tr}} := \frac{1}{n} \sum_{j=1}^n Z_{i,j}^{\mathrm{tr}}
    \qquad \text{and} \qquad
    \mu^{\mathrm{tr}} := \E\bracket{Z_i^{\mathrm{tr}}}.
\end{align*}
We then define the user-level truncated random variable
\begin{align}\label{eq:rv_C_7_2}
    Z_i^{\trtwo} := \begin{cases}
   Z_i^{\mathrm{tr}}, & \text{if } \norm{Z_i^{\mathrm{tr}} - \mu^{\mathrm{tr}}}_2 \le 2\sqrt{\frac{d}{n \eps}},\\
   \mu^{\mathrm{tr}}, & \text{otherwise,}
   \end{cases}
\end{align}
and write
\begin{align*}
    \mu^{\trtwo} := \E\bracket{Z_i^{\trtwo}}.
\end{align*}
\textbf{Part 1} and \textbf{Part 2} respectively follow from Claim~\ref{cla:covariance-condition_part_1} and Claim~\ref{cla:covariance-condition_part_3} by setting $\gamma = 1 / n$\footnote{Since $Z_i^{\mathrm{tr}} := \frac{1}{n} \sum_{j=1}^n Z_{i,j}^{\mathrm{tr}}$, we can see $\cov\bracket{Z_i^{\mathrm{tr}}}\preceq \frac{1}{n}\mathbb I_d$ by Fact~\ref{fac:low_covariance}.}, $Z_j$ to be $Z_i^{\mathrm{tr}}$, $Z_j^{\mathrm{tr}}$ to be $Z_i^{\trtwo}$, and $n$ to be $N$.
Below, we give a proof for \textbf{Part 3}.

We intend to apply the matrix Chernoff inequality (see Fact~\ref{fac:matrix_Chernoff}). Thus, we define
\begin{align}\label{eq:Xij'_2}
    M_i^{\trtwo} \eqdef \frac{1}{N}\paren{Z_i^{\trtwo}-\mu^{\trtwo}}\paren{Z_i^{\trtwo}-\mu^{\trtwo}}^{\top}.
\end{align}
By the definition of $Z_i^{\trtwo}$ (Eq.~\eqref{eq:rv_C_7_2}), we know that for all $i \in [N]$,
\begin{align}\label{eq:Xij'_muij'_22}
    \norm{Z_i^{\trtwo}-\mu^{\trtwo}}_2 
    \leq & ~ \norm{Z_i^{\trtwo}-\mu^{\mathrm{tr}}}_2 + \norm{\mu^{\mathrm{tr}} -\mu^{\trtwo}}_2 \notag\\
    \leq & ~ O\paren{\sqrt{\frac{d}{n \eps}} + \sqrt{\frac{\eps}{n}}} \notag\\
    \leq & ~ O\paren{\sqrt{\frac{d}{n \eps}}},
\end{align}
where the first step follows from the triangle inequality, the second step follows from Eq.~\eqref{eq:rv_C_7_2} and \textbf{Part 2}, and the last step follows from $\eps < 1$ and $d > 1$.
Additionally, we can get
\begin{align}\label{eq:Xij'_muij'_33}
    \max_{i \in [N]} \norm{M_i^{\trtwo}}_2
    = & ~ \max_{i \in [N]} \norm{\frac{1}{N}\paren{Z_i^{\trtwo}-\mu^{\trtwo}}\paren{Z_i^{\trtwo}-\mu^{\trtwo}}^{\top}}_2 \notag\\
    = & ~ \frac{1}{N}\max_{i \in [N]} \norm{\paren{Z_i^{\trtwo}-\mu^{\trtwo}}\paren{Z_i^{\trtwo}-\mu^{\trtwo}}^{\top}}_2 \notag\\
    = & ~ \frac{1}{N}\max_{i \in [N]} \norm{Z_i^{\trtwo}-\mu^{\trtwo}}_2^2 \notag\\
    \leq & ~ O\paren{\frac{d}{Nn\eps}},
\end{align}
where the first step follows from Eq.~\eqref{eq:Xij'_2}, the third step follows from Fact~\ref{fac:spectral_ell_2}, and the last step follows from Eq.~\eqref{eq:Xij'_muij'_22}.
This implies that $R = O\paren{\frac{d}{Nn\eps}}$ with 
    $\mathbf{0}_{d \times d} \preceq M_i^{\trtwo} \preceq R \cdot \mathbb{I}_d$.
Furthermore, we can get
\begin{align*}
    \E\bracket{\frac{1}{N}\sum_{i = 1}^N\paren{Z_i^{\trtwo}-\mu^{\trtwo}}\paren{Z_i^{\trtwo}-\mu^{\trtwo}}^{\top}} 
    = & ~ \frac{1}{N}\sum_{i = 1}^N \E\bracket{\paren{Z_i^{\trtwo}-\mu^{\trtwo}}\paren{Z_i^{\trtwo}-\mu^{\trtwo}}^{\top}}  \\
    = & ~ \frac{1}{N}\sum_{i = 1}^N \mathrm{Cov}\bracket{Z_i^{\trtwo}} \\
    \preceq & ~ \frac{1}{n}\mathbb{I}_d,
\end{align*}
where the first step follows from the linearity of expectation, the second step follows from the definition of covariance, and the last step follows from $\mathrm{Cov}\bracket{Z_i^{\trtwo}} \preceq \frac{1}{n}\mathbb{I}_d$.
Let
\[
    L:=\frac1n+\tau.
\]
The preceding expectation bound gives
\[
    \left\|\E\left[\sum_{i=1}^N M_i^{\trtwo}\right]\right\|_2
    \le \frac1n\le L.
\]
Applying the matrix Chernoff inequality with a sufficiently large universal constant \(C>0\) yields
\begin{align*}
    \Pr\left[
        \frac1N\sum_{i=1}^N
        (Z_i^{\trtwo}-\mu^{\trtwo})
        (Z_i^{\trtwo}-\mu^{\trtwo})^\top
        \npreceq C L\mathbb I_d
    \right]
    &\le
    d\exp\left(-\Omega\left(\frac{L}{R}\right)\right)\\
    &\le
    d\exp\left(-\Omega\left(
        \frac{Nn\eps}{d}\left(\frac1n+\tau\right)
    \right)\right)\\
    &\le
    d\exp\left(-\Omega\left(
        \frac{Nn\eps\tau}{d}
    \right)\right).
\end{align*}
Since \(\eps\tau=\alpha\), this failure probability is at most \(\delta\) whenever
\[
    nN
    \ge
    \Omega\paren{\frac d\alpha\log\paren{d/\delta}}.
\]
Centering at the empirical mean can only decrease the empirical second-moment matrix. Thus, after increasing the universal constant in Eq.~\eqref{eq:C_11} if necessary, the user-level covariance constraint holds with probability at least \(1-\delta\). This establishes Part~3.
\end{proof}

\begin{proof}[Proof of Lemma~\ref{lem:bounded_covariance_5}]
Work on the intersection of the high-probability events in Lemmas~\ref{lem:C_7} and~\ref{lem:C_11}. Let \(\mathcal G\subseteq[N]\) be the users whose batches are not corrupted at the user level, and define
\[
    \mathcal H
    :=
    \{i\in\mathcal G:z_i^{\trtwo}=z_i^{\mathrm{tr}}\}.
\]
Since at most an \(\eps\)-fraction of users are corrupted at the user level and at most an additional \(\eps\)-fraction are modified by the second truncation, we have
\[
    |\mathcal H|\ge (1-O(\eps))N.
\]
Choose a subset \(\mathcal S\subseteq\mathcal H\) of cardinality
\[
    |\mathcal S|=(1-\Theta(\eps))N,
\]
where the hidden constant is chosen sufficiently large. Set
\[
    U_i:=\indic{i\in\mathcal S},
    \qquad
    \mathsf Z_{i,j}:=z_{i,j}^{\mathrm{tr}},
    \qquad
    \mathsf Z_i:=z_i^{\mathrm{tr}},
    \qquad
    \mathsf Y_i:=z_i^{\trtwo}.
\]
Then Eqs.~\eqref{eq:C_5}, \eqref{eq:C_9}, and \eqref{eq:C_6} hold. Indeed, if \(U_i=1\), then \(i\in\mathcal S\subseteq\mathcal H\), so \(z_i^{\trtwo}=z_i^{\mathrm{tr}}\).

For each \(i\in\mathcal S\), define
\[
    \mathcal A_i
    :=
    \{j\in[n]:z_{i,j}^{\mathrm{tr}}=x_{i,j}\}.
\]
Because \(i\in\mathcal G\), at most \(\alpha n\) observations in batch \(i\) were modified by the sample-level adversary. Therefore,
\[
    |\mathcal A_i|
    \ge
    n-\alpha n-
    \sum_{j=1}^n\indic{z_{i,j}\neq z_{i,j}^{\mathrm{tr}}}.
\]
Summing over \(i\in\mathcal S\) and using Part~1 of Lemma~\ref{lem:C_7},
\begin{align*}
    \sum_{i\in\mathcal S}|\mathcal A_i|
    &\ge
    n|\mathcal S|-\alpha n|\mathcal S|-\alpha Nn\\
    &\ge
    (1-\Theta(\alpha))n|\mathcal S|,
\end{align*}
where the hidden constant is chosen sufficiently large and we used \(|\mathcal S|=(1-\Theta(\eps))N\). Select exactly \((1-\Theta(\alpha))n|\mathcal S|\) pairs \((i,j)\) with \(i\in\mathcal S\) and \(j\in\mathcal A_i\), and set \(W_{i,j}=1\) on these pairs and \(W_{i,j}=0\) elsewhere. Then Eqs.~\eqref{eq:C_1}, \eqref{eq:C_3}, and \eqref{eq:C_8} hold.

By Part~3 of Lemma~\ref{lem:C_7}, the sample-level empirical covariance is bounded by \(C_3\mathbb I_d\) after increasing \(C_3\) if necessary. Centering at \(\overline{\mathsf Z}\) can only decrease the empirical second-moment matrix, so Eq.~\eqref{eq:C_7} holds for a suitable positive-semidefinite slack matrix \(B_1B_1^\top\). Likewise, Part~3 of Lemma~\ref{lem:C_11} and the same centering argument imply Eq.~\eqref{eq:C_11} for a suitable slack matrix \(B_2B_2^\top\). Finally, Eqs.~\eqref{eq:C_2}, \eqref{eq:C_4}, and \eqref{eq:C_10} hold by construction. A union bound over the events in Lemmas~\ref{lem:C_7} and~\ref{lem:C_11} completes the proof.
\end{proof}

\subsection{Identifiability}
\label{sub:arbitrary_alpha:identify}

The goal of this section is to establish the identifiability of the true mean~$\mu$
under the two-level corruption model captured by the polynomial system~$\mathsf{B}$.
After showing in \Cref{sub:arbitrary_alpha:satisfy} that the system is satisfiable with high probability,
we now argue that any feasible pseudoexpectation consistent with the SoS constraints
must correspond to an accurate estimate of~$\mu$.

At a high level, our proof proceeds in two stages. First, the global sample-count and agreement constraints ensure that, across the selected reference users, only an $O(\alpha)$-fraction of sample--user pairs can contribute discrepancies from the truncated clean reference values. The sample-level covariance constraint then controls the aggregate effect of these discrepancies.

Second, an~$\eps$-fraction of users may be entirely corrupted,
    but the remaining users have bounded deviations.
    Using the bounded-covariance conditions enforced by the SoS system,
    we show that the pseudoexpectation output~$\wt{\E}[\overline{\mathsf{Y}}]$
    (and thus the estimator~$\wh{\mu}$) cannot deviate far from the mean
    of the clean users.
    The total deviation accumulates additively from the two levels of corruption:
    $\sqrt{\frac{\eps}{n}}$ from user-level corruption
    and $\sqrt{\alpha}$ from the corruption of samples within each uncorrupted user.

Together, these steps show that any feasible pseudoexpectation satisfying
the SoS constraints must yield a mean estimate~$\wh{\mu}$ satisfying
\begin{align*}
    \norm{\wh \mu - \mu}_2 = O\paren{\sqrt{\frac{\eps}{n}}+  \sqrt{\alpha}}
\end{align*}
thereby establishing identifiability of~$\mu$ under the hierarchical corruption model.

\begin{proof}[Proof of Lemma~\ref{lem:mean_close_5}]
For every user $i\in[N]$, define the realized sample-level truncated empirical
mean and the realized twice-truncated user-level empirical mean by
\begin{align}
    \wt\mu_i^{\mathrm{tr}}
    &:=
    \frac1n\sum_{j=1}^n z_{i,j}^{\mathrm{tr}},
    \label{eq:wt_mu_i_tr}
    \\
    \wt\mu_i^{\trtwo}
    &:=
    z_i^{\trtwo}.
    \label{eq:wt_mu_i_tr_2}
\end{align}
We also define their averages over users:
\begin{align}
    \wt\mu^{\mathrm{tr}}
    &:=
    \frac1N\sum_{i=1}^N \wt\mu_i^{\mathrm{tr}}
    =
    \frac1{Nn}\sum_{i=1}^N\sum_{j=1}^n z_{i,j}^{\mathrm{tr}},
    \label{eq:wt_mu_tr}
    \\
    \wt\mu^{\trtwo}
    &:=
    \frac1N\sum_{i=1}^N \wt\mu_i^{\trtwo}
    =
    \frac1N\sum_{i=1}^N z_i^{\trtwo}.
    \label{eq:wt_mu_tr_2}
\end{align}
Here $z_{i,j}^{\mathrm{tr}}$ is the realized value of
$Z_{i,j}^{\mathrm{tr}}$, and $z_i^{\trtwo}$ is the realized value of
$Z_i^{\trtwo}$.

We introduce a deterministic indicator
\begin{align}\label{eq:mathcal_I}
    \mathcal I_i
    :=
    \indic{\mathsf{GoodRef}_i}.
\end{align}
Here $\mathsf{GoodRef}_i$ is the deterministic event that user $i$ is genuinely
good and
\[
    z_i^{\trtwo}=z_i^{\mathrm{tr}},
    \qquad
    z_i^{\mathrm{tr}}
    :=
    \frac1n\sum_{j=1}^n z_{i,j}^{\mathrm{tr}}.
\]
Thus, $\mathcal I_i=1$ means that user $i$ is not corrupted at the user
level and is not modified by the second truncation.
By the batch-level corruption bound and the user-level truncation bound,
\begin{equation}
    \frac1N\sum_{i=1}^N(1-\mathcal I_i)
    \le O(\eps).
    \label{eq:bad_reference_users}
\end{equation}
Moreover, among users with $\mathcal I_i=1$, disagreement between
$z_{i,j}^{\mathrm{tr}}$ and $x_{i,j}$ can arise only from sample-level
corruption or sample-level truncation. Therefore,
\begin{equation}
    \frac1{Nn}
    \sum_{i=1}^N\sum_{j=1}^n
    \mathcal I_i
    \left(
        1-\indic{z_{i,j}^{\mathrm{tr}}=x_{i,j}}
    \right)
    \le O(\alpha).
    \label{eq:global_reference_disagreement}
\end{equation}
By the definition of the $\ell_2$ norm, we have
\[
\mathsf{B}\sststile{4}{\mathsf Z_{i,j}, \mathsf Y_i, W_{i,j}, U_i, B_1, B_2}\left\{
\begin{aligned}
    \norm{\overline{\mathsf{Y}} - \wt\mu^{\trtwo}}_2^4
    = \left \langle \overline{\mathsf{Y}} - \wt\mu^{\trtwo}, \overline{\mathsf{Y}} - \wt\mu^{\trtwo} \right\rangle^2
\end{aligned}
\right\}
\]
Next, by combining Eq.~\eqref{eq:C_10} and Eq.~\eqref{eq:wt_mu_tr_2}, we can get
\[
\mathsf{B}\sststile{4}{\mathsf Z_{i,j}, \mathsf Y_i, W_{i,j}, U_i, B_1, B_2}\left\{
\begin{aligned}
    & ~ \left\langle \overline{\mathsf{Y}} - \wt\mu^{\trtwo}, \overline{\mathsf{Y}} - \wt\mu^{\trtwo} \right\rangle^2 \\
    = & ~ \Bigg( \frac{1}{N} \sum_{i=1}^N (1 - U_i \mathcal{I}_i) \left \langle \mathsf{Y}_i - \wt\mu_i^{\trtwo}, \overline{\mathsf{Y}} - \wt\mu^{\trtwo} \right \rangle \\
    + & ~ \frac{1}{N} \sum_{i=1}^N U_i \mathcal{I}_i \left \langle \mathsf{Y}_i - \wt\mu_i^{\trtwo}, \overline{\mathsf{Y}} - \wt\mu^{\trtwo} \right \rangle \Bigg)^2
\end{aligned}
\right\}
\]
Applying the SoS inequality $(a+b)^2\le 2a^2+2b^2$, which follows from
\begin{align*}
    2a^2+2b^2-(a+b)^2=(a-b)^2\ge 0,
\end{align*}
we obtain
\begin{align}\label{eq:bounding_ov_Y_wt_mu}
\mathsf{B}\sststile{4}{\mathsf Z_{i,j}, \mathsf Y_i, W_{i,j}, U_i, B_1, B_2}\left\{
\begin{aligned}
    \norm{\overline{\mathsf{Y}} - \wt\mu^{\trtwo}}_2^4
    \leq & ~ 2\left( \frac{1}{N} \sum_{i=1}^N (1 - U_i \mathcal{I}_i) \left \langle \mathsf{Y}_i - \wt\mu_i^{\trtwo}, \overline{\mathsf{Y}} - \wt\mu^{\trtwo} \right \rangle \right)^2 \\
    + & ~ 2\paren{\frac{1}{N} \sum_{i=1}^N U_i \mathcal{I}_i \left \langle \mathsf{Y}_i - \wt\mu_i^{\trtwo}, \overline{\mathsf{Y}} - \wt\mu^{\trtwo} \right \rangle}^2
\end{aligned}
\right\}
\end{align}

Considering the first term of Eq.~\eqref{eq:bounding_ov_Y_wt_mu}, we use the SoS Cauchy-Schwarz inequality (Fact~\ref{fac:cauchy_schwarz}) to obtain
\begin{align}\label{eq:sep}
\mathsf{B}\sststile{6}{\mathsf Z_{i,j}, \mathsf Y_i, W_{i,j}, U_i, B_1, B_2}\left\{
\begin{aligned}
    & ~ \left( \frac{1}{N} \sum_{i=1}^N (1 - U_i \mathcal{I}_i) \left \langle \mathsf{Y}_i - \wt\mu_i^{\trtwo}, \overline{\mathsf{Y}} - \wt\mu^{\trtwo} \right \rangle \right)^2 \\
    \leq & ~ \left( \frac{1}{N} \sum_{i=1}^N (1 - U_i \mathcal{I}_i)^2 \right) \left( \frac{1}{N} \sum_{i=1}^N \left \langle \mathsf{Y}_i - \wt\mu_i^{\trtwo}, \overline{\mathsf{Y}} - \wt\mu^{\trtwo} \right\rangle^2 \right) 
\end{aligned}
\right\}.
\end{align}

Regarding $ \frac{1}{N} \sum_{i=1}^N (1 - U_i \mathcal{I}_i)^2$, we have
\begin{align}\label{eq:bound_the_first_part}
    \mathsf{B}\sststile{2}{\mathsf Z_{i,j}, \mathsf Y_i, W_{i,j}, U_i, B_1, B_2}\left\{
    \begin{aligned}
    \frac1N\sum_{i=1}^N(1-U_i\mathcal I_i)^2
    &=
    \frac1N\sum_{i=1}^N(1-U_i\mathcal I_i)\\
    &=
    \frac1N\sum_{i=1}^N
    \bigl(1-U_i+U_i(1-\mathcal I_i)\bigr)\\
    &\le
    \frac1N\sum_{i=1}^N(1-U_i)
    +
    \frac1N\sum_{i=1}^N(1-\mathcal I_i)\\
    &\le O(\eps).
    \end{aligned}
    \right\},
\end{align}
where the first equality follows from $U_i^2=U_i$ and the fact that $\mathcal I_i$ is a deterministic indicator. The final inequality follows from Eq.~\eqref{eq:C_9} and Eq.~\eqref{eq:bad_reference_users}.
Furthermore, defining $b := \overline{\mathsf{Y}} - \wt\mu^{\trtwo}$, we have
\begin{align}\label{eq:bound_on_22}
\mathsf{B}\sststile{4}{\mathsf Z_{i,j}, \mathsf Y_i, W_{i,j}, U_i, B_1, B_2}\left\{
\begin{aligned}
    & ~ \frac{1}{N} \sum_{i=1}^N \left \langle \mathsf{Y}_i - \wt\mu_i^{\trtwo}, b \right \rangle^2 \\ 
= & ~  \frac{1}{N} \sum_{i=1}^N \left \langle \mathsf{Y}_i - \wt\mu_i^{\trtwo} + b - b, b \right \rangle^2 \\
= & ~  \frac{1}{N} \sum_{i=1}^N \paren{\left \langle \wt\mu^{\trtwo} - \wt\mu_i^{\trtwo}, b \right \rangle - \left \langle \overline{\mathsf{Y}} - \mathsf{Y}_i, b \right \rangle + \norm{b}_2^2}^2 \\
\leq & ~ \frac{3}{N} \sum_{i=1}^N \left \langle \wt\mu^{\trtwo} - \wt\mu_i^{\trtwo}, b \right \rangle^2 + \left \langle \overline{\mathsf{Y}} - \mathsf{Y}_i, b \right \rangle^2 + \norm{b}_2^4 \\
= & ~ 3 \paren{b^\top \Sigma_{\wt\mu^{\trtwo}} b + b^\top \Sigma_{\mathsf Y} b + \norm{b}_2^4}
\end{aligned} 
\right\},
\end{align}
where the second step follows from the linearity of inner product, the third step follows from Fact~\ref{fac:xyz_square}, and the last step follows from the definition of the (empirical) covariance. In addition, by the user-level covariance bound for the twice-truncated reference means, the empirical covariance $\Sigma_{\wt\mu^{\trtwo}}$ satisfies:
\begin{align}\label{eq:bound_on_221}
\mathsf{B}\sststile{2}{\mathsf Z_{i,j}, \mathsf Y_i, W_{i,j}, U_i, B_1, B_2}\left\{
    b^\top \Sigma_{\wt\mu^{\trtwo}} b \leq \paren{\frac{1}{n}+\tau} \norm{b}_2^2
\right\}
\end{align}
hold with high probability.
Similarly, by Eq.~\eqref{eq:C_11} and Fact~\ref{fac:sos_operator_norm}, we have
\begin{align}\label{eq:bound_on_222}
\mathsf{B}\sststile{2}{\mathsf Z_{i,j}, \mathsf Y_i, W_{i,j}, U_i, B_1, B_2}\left\{
    b^\top \Sigma_{\mathsf Y} b \leq \paren{\frac{1}{n} + \tau} \norm{b}_2^2
\right\}.
\end{align}

Combining Eq.~\eqref{eq:bound_on_22}, Eq.~\eqref{eq:bound_on_221}, and Eq.~\eqref{eq:bound_on_222}, we have
\begin{align}\label{eq:bound_on_223}
    \mathsf{B}\sststile{4}{\mathsf Z_{i,j}, \mathsf Y_i, W_{i,j}, U_i, B_1, B_2}\left\{
    \begin{aligned}
        & ~ \frac{1}{N} \sum_{i=1}^N \left \langle \mathsf{Y}_i - \wt\mu_i^{\trtwo}, b \right \rangle^2  \\
        \leq & ~ 3\paren{4\paren{\frac{1}{n} + \tau} \norm{\overline{\mathsf{Y}} - \wt\mu^{\trtwo}}_2^2 + \norm{\overline{\mathsf{Y}} - \wt\mu^{\trtwo}}_2^4}
    \end{aligned}
\right\}.
\end{align}

We further combine Eq.~\eqref{eq:sep}, Eq.~\eqref{eq:bound_the_first_part}, and Eq.~\eqref{eq:bound_on_223}:
\begin{align}
\label{eq:first_term}
\mathsf B\sststile{6}{\mathsf Z_{i,j},\mathsf Y_i,W_{i,j},U_i,B_1,B_2}
\left\{
\begin{aligned}
    &2\left(
        \frac1N\sum_{i=1}^N
        (1-U_i\mathcal I_i)
        \left\langle
            \mathsf Y_i-\wt\mu_i^{\trtwo},
            \overline{\mathsf Y}-\wt\mu^{\trtwo}
        \right\rangle
    \right)^2
    \\
    &\le
    O(\eps)
    \left(
        \left(\frac1n+\tau\right)
        \norm{\overline{\mathsf Y}-\wt\mu^{\trtwo}}_2^2
        +
        \norm{\overline{\mathsf Y}-\wt\mu^{\trtwo}}_2^4
    \right).
\end{aligned}
\right\}.
\end{align}
This completes the bound for the first term of Eq.~\eqref{eq:bounding_ov_Y_wt_mu}.
Now, we consider the second term of Eq.~\eqref{eq:bounding_ov_Y_wt_mu}. We have
\begin{align}\label{eq:sec_term}
\mathsf{B}\sststile{4}{\mathsf Z_{i,j}, \mathsf Y_i, W_{i,j}, U_i, B_1, B_2}\left\{
\begin{aligned}
    & ~ \frac{1}{N} \sum_{i = 1}^N U_i \mathcal{I}_i \left \langle \mathsf{Y}_i - \wt\mu_i^{\trtwo}, \overline{\mathsf{Y}} - \wt\mu^{\trtwo} \right \rangle \\
    = & ~ \frac{1}{Nn} \sum_{i = 1}^N U_i \mathcal{I}_i \paren{\sum_{j = 1}^n \left \langle \mathsf{Z}_{i,j} - z_{i,j}^{\mathrm{tr}}, \overline{\mathsf{Y}} - \wt\mu^{\trtwo} \right \rangle}
\end{aligned}
    \right\},
\end{align}
where the first step follows from Eq.~\eqref{eq:wt_mu_i_tr_2}, Eq.~\eqref{eq:C_5}, and the fact that $\mathcal I_i=1$ implies $z_i^{\trtwo}=z_i^{\mathrm{tr}}=\frac1n\sum_{j=1}^n z_{i,j}^{\mathrm{tr}}$.
Now, with any arbitrary $i, j$, we consider $\left \langle \mathsf{Z}_{i,j} - z_{i,j}^{\mathrm{tr}}, \overline{\mathsf{Y}} - \wt\mu^{\trtwo} \right \rangle$. We give the following claim:
\begin{claim}\label{cla:sec_term_sub}
    We define $A \eqdef \overline{\mathsf{Y}} - \wt\mu^{\trtwo}$. Then, we have
    \begin{align*}
    \mathsf{B}\sststile{2}{\mathsf Z_{i,j}, \mathsf Y_i, W_{i,j}, U_i, B_1, B_2}\left\{
        \left \langle \mathsf{Z}_{i,j} - z_{i,j}^{\mathrm{tr}}, A \right \rangle = \paren{1 - W_{i, j} \indic{z_{i,j}^{\mathrm{tr}} = x_{i, j}}} \cdot \left \langle \mathsf{Z}_{i,j} - z_{i,j}^{\mathrm{tr}}, A \right \rangle
    \right\}.
    \end{align*}
\end{claim}
\begin{proof}
It suffices to show
\begin{align}\label{eq:want_to_show}
\mathsf{B}\sststile{2}{\mathsf Z_{i,j}, \mathsf Y_i, W_{i,j}, U_i, B_1, B_2}\left\{
    0 = - W_{i, j} \indic{z_{i,j}^{\mathrm{tr}} = x_{i, j}} \cdot \left \langle \mathsf{Z}_{i,j} - z_{i,j}^{\mathrm{tr}}, A \right \rangle
\right\}.
\end{align}
If $\indic{z_{i,j}^{\mathrm{tr}} = x_{i, j}} = 0$, then Eq.~\eqref{eq:want_to_show} holds. If $\indic{z_{i,j}^{\mathrm{tr}} = x_{i, j}} = 1$, then we have $z_{i,j}^{\mathrm{tr}} = x_{i, j}$. By our polynomial constraint (Eq.~\eqref{eq:C_3}),
we have $W_{i,j}(\mathsf{Z}_{i,j}-x_{i,j})=0$. Therefore, we have $W_{i,j}(\mathsf{Z}_{i,j}-z_{i,j}^{\mathrm{tr}})=0$. Taking inner product
with $A$ gives us $W_{i,j}\langle \mathsf{Z}_{i,j}-z_{i,j}^{\mathrm{tr}},A\rangle = 0$.
\end{proof}

Therefore, combining Eq.~\eqref{eq:sec_term} and Claim~\ref{cla:sec_term_sub}, we have
\begin{align}\label{eq:sec_term_more}
\mathsf{B}\sststile{6}{\mathsf Z_{i,j}, \mathsf Y_i, W_{i,j}, U_i, B_1, B_2}\left\{
\begin{aligned}
    & ~ \paren{\frac{1}{N} \sum_{i = 1}^N U_i \mathcal{I}_i \left \langle \mathsf{Y}_i - \wt\mu_i^{\trtwo}, \overline{\mathsf{Y}} - \wt\mu^{\trtwo} \right \rangle}^2  \\
    \leq & ~ \paren{\frac{1}{Nn} \sum_{i=1}^N\sum_{j = 1}^n U_i\mathcal{I}_i { \paren{1- W_{i, j} \indic{z_{i,j}^{\mathrm{tr}} = x_{i, j}}}\left \langle \mathsf{Z}_{i,j} - z_{i,j}^{\mathrm{tr}}, \overline{\mathsf{Y}} - \wt\mu^{\trtwo} \right \rangle}}^2 \\
    = & ~ \paren{\frac{1}{Nn} \sum_{i=1}^N\sum_{j = 1}^n \paren{U_i\mathcal{I}_i  \paren{1- W_{i, j} \indic{z_{i,j}^{\mathrm{tr}} = x_{i, j}}}}^2} \\
    \cdot & ~ \paren{\frac{1}{Nn} \sum_{i=1}^N\sum_{j = 1}^n\left \langle \mathsf{Z}_{i,j} - z_{i,j}^{\mathrm{tr}}, \overline{\mathsf{Y}} - \wt\mu^{\trtwo} \right \rangle^2} \\
    = & ~ \paren{\frac{1}{Nn} \sum_{i=1}^N\sum_{j = 1}^n U_i\mathcal{I}_i  \paren{1- W_{i, j} \indic{z_{i,j}^{\mathrm{tr}} = x_{i, j}}}^2} \\
    \cdot & ~ \paren{\frac{1}{Nn} \sum_{i=1}^N\sum_{j = 1}^n\left \langle \mathsf{Z}_{i,j} - z_{i,j}^{\mathrm{tr}}, \overline{\mathsf{Y}} - \wt\mu^{\trtwo} \right \rangle^2} \\
    = & ~ \paren{\frac{1}{Nn} \sum_{i=1}^N\sum_{j = 1}^n U_i\mathcal{I}_i  \paren{1- W_{i, j} \indic{z_{i,j}^{\mathrm{tr}} = x_{i, j}}}}\\
    \cdot & ~ \paren{\frac{1}{Nn} \sum_{i=1}^N\sum_{j = 1}^n\left \langle \mathsf{Z}_{i,j} - z_{i,j}^{\mathrm{tr}}, \overline{\mathsf{Y}} - \wt\mu^{\trtwo} \right \rangle^2}
\end{aligned}
\right\},
\end{align}
where the second step follows from the Cauchy-Schwarz inequality (see Fact~\ref{fac:cauchy_schwarz}), the third step follows from $U_i = U_i^2$ (see Eq.~\eqref{eq:C_4}), the definition of the indicator function $\mathcal{I}_i$, and the last step follows from
\begin{align*}
    \paren{1- W_{i, j} \indic{z_{i,j}^{\mathrm{tr}} = x_{i, j}}}^2
    = & ~ 1 + W_{i, j}^2 \indic{z_{i,j}^{\mathrm{tr}} = x_{i, j}}^2 - 2 W_{i, j} \indic{z_{i,j}^{\mathrm{tr}} = x_{i, j}}\\
    = & ~ 1 + W_{i, j} \indic{z_{i,j}^{\mathrm{tr}} = x_{i, j}} - 2 W_{i, j} \indic{z_{i,j}^{\mathrm{tr}} = x_{i, j}}\\
    = & ~ 1 - W_{i, j} \indic{z_{i,j}^{\mathrm{tr}} = x_{i, j}}.
\end{align*}

Considering the first factor
    $\frac1{Nn}
    \sum_{i=1}^N\sum_{j=1}^n
    U_i\mathcal I_i
    \left(
        1-W_{i,j}
        \indic{z_{i,j}^{\mathrm{tr}}=x_{i,j}}
    \right)$
in Eq.~\eqref{eq:sec_term_more}, we have
\begin{align}
\label{eq:sec_term_more_1}
\mathsf B\sststile{2}{\mathsf Z_{i,j},\mathsf Y_i,W_{i,j},U_i,B_1,B_2}
\left\{
\begin{aligned}
    &\frac1{Nn}
    \sum_{i=1}^N\sum_{j=1}^n
    U_i\mathcal I_i
    \left(
        1-W_{i,j}
        \indic{z_{i,j}^{\mathrm{tr}}=x_{i,j}}
    \right)
    \\
    &=
    \frac1{Nn}
    \sum_{i=1}^N\sum_{j=1}^n
    U_i\mathcal I_i(1-W_{i,j})
    +
    \frac1{Nn}
    \sum_{i=1}^N\sum_{j=1}^n
    U_i\mathcal I_iW_{i,j}
    \left(
        1-\indic{z_{i,j}^{\mathrm{tr}}=x_{i,j}}
    \right)
    \\
    &\le
    \frac1{Nn}
    \sum_{i=1}^N\sum_{j=1}^n
    U_i(1-W_{i,j})
    +
    \frac1{Nn}
    \sum_{i=1}^N\sum_{j=1}^n
    \mathcal I_i
    \left(
        1-\indic{z_{i,j}^{\mathrm{tr}}=x_{i,j}}
    \right)
    \\
    &=
    \frac1{Nn}
    \left(
        n\sum_{i=1}^NU_i
        -
        \sum_{i=1}^N\sum_{j=1}^nU_iW_{i,j}
    \right)
    +
    \frac1{Nn}
    \sum_{i=1}^N\sum_{j=1}^n
    \mathcal I_i
    \left(
        1-\indic{z_{i,j}^{\mathrm{tr}}=x_{i,j}}
    \right)
    \\
    &=
    \frac1{Nn}
    \left(
        n\sum_{i=1}^NU_i
        -
        (1-\Theta(\alpha))n\sum_{i=1}^NU_i
    \right)
    +
    \frac1{Nn}
    \sum_{i=1}^N\sum_{j=1}^n
    \mathcal I_i
    \left(
        1-\indic{z_{i,j}^{\mathrm{tr}}=x_{i,j}}
    \right)
    \\
    &=
    \Theta(\alpha)\frac1N\sum_{i=1}^NU_i
    +
    \frac1{Nn}
    \sum_{i=1}^N\sum_{j=1}^n
    \mathcal I_i
    \left(
        1-\indic{z_{i,j}^{\mathrm{tr}}=x_{i,j}}
    \right)
    \\
    &=
    \Theta(\alpha)(1-\Theta(\eps))
    +
    \frac1{Nn}
    \sum_{i=1}^N\sum_{j=1}^n
    \mathcal I_i
    \left(
        1-\indic{z_{i,j}^{\mathrm{tr}}=x_{i,j}}
    \right) \\
    &\le
    O(\alpha)+O(\alpha)
    \\
    &=
    O(\alpha)
\end{aligned}
\right\},
\end{align}
where the second step follows from $\mathcal{I}_i, U_i, W_{i, j} \in \{0, 1\}$, the fourth step follows from $\sum_{i=1}^N\sum_{j=1}^n U_iW_{i,j}
=
(1-\Theta\paren{\alpha})n\sum_{i=1}^N U_i$ (Eq.~\eqref{eq:C_8}), the sixth step follows from $\sum_{i=1}^N U_i
    =
    (1 - \Theta\paren{\eps})N$ (Eq.~\eqref{eq:C_9}), and the seventh step follows from Eq.~\eqref{eq:global_reference_disagreement}.

Considering the second part $\frac{1}{Nn} \sum_{i=1}^N\sum_{j = 1}^n\left \langle \mathsf{Z}_{i,j} - z_{i,j}^{\mathrm{tr}}, \overline{\mathsf{Y}} - \wt\mu^{\trtwo} \right \rangle^2$ of Eq.~\eqref{eq:sec_term_more}, we define $c := \overline{\mathsf{Y}} - \wt\mu^{\trtwo}$ and have
\begin{align}\label{eq:sec_term_more_2}
\mathsf{B}\sststile{4}{\mathsf Z_{i,j}, \mathsf Y_i, W_{i,j}, U_i, B_1, B_2}\left\{
\begin{aligned}
    & ~ \frac{1}{Nn} \sum_{i=1}^N\sum_{j = 1}^n\left \langle \mathsf{Z}_{i,j} - z_{i,j}^{\mathrm{tr}}, c \right \rangle^2 \\
    = & ~ \frac{1}{Nn} \sum_{i=1}^N\sum_{j = 1}^n\left \langle \mathsf{Z}_{i,j} - z_{i,j}^{\mathrm{tr}} - c + c, c \right \rangle^2 \\
    = & ~ \frac{1}{Nn} \sum_{i=1}^N\sum_{j = 1}^n \paren{\left \langle \mathsf{Z}_{i,j} - z_{i,j}^{\mathrm{tr}} - c, c \right \rangle + \left \langle c, c \right \rangle}^2 \\
    = & ~ \frac{1}{Nn} \sum_{i=1}^N\sum_{j = 1}^n \paren{\left \langle \mathsf{Z}_{i,j} - \overline{\mathsf{Y}}, c \right \rangle - \left \langle z_{i,j}^{\mathrm{tr}} - \wt\mu^{\trtwo}, c \right \rangle + \norm{c}_2^2}^2 \\
    \leq & ~ \frac{3}{Nn} \sum_{i=1}^N\sum_{j = 1}^n \left \langle \mathsf{Z}_{i,j} - \overline{\mathsf{Y}}, c \right \rangle^2 + \frac{3}{Nn} \sum_{i=1}^N\sum_{j = 1}^n \left \langle z_{i,j}^{\mathrm{tr}} - \wt\mu^{\trtwo}, c \right \rangle^2 + 3 \norm{c}_2^4
\end{aligned}
\right\},
\end{align}
where the second and the third step follow from the linearity of inner product and the last step follows from Fact~\ref{fac:xyz_square}.
In particular, considering the first term of Eq.~\eqref{eq:sec_term_more_2}, we have
\begin{align}\label{eq:sec_term_more_2_first_term}
\mathsf{B}\sststile{6}{\mathsf Z_{i,j}, \mathsf Y_i, W_{i,j}, U_i, B_1, B_2}\left\{
\begin{aligned}
    & ~ \frac{3}{Nn} \sum_{i=1}^N\sum_{j = 1}^n \left \langle \mathsf{Z}_{i,j} - \overline{\mathsf{Y}}, c \right \rangle^2 \\
    \leq & ~ \frac{6}{Nn} \sum_{i=1}^N\sum_{j = 1}^n \left \langle \mathsf{Z}_{i,j} - \overline{\mathsf{Z}}, c \right \rangle^2 + \frac{6}{Nn} \sum_{i=1}^N\sum_{j = 1}^n \left \langle \overline{\mathsf{Z}} - \overline{\mathsf{Y}}, c \right \rangle^2 \\
    = & ~ 6 c^\top \frac{1}{Nn}\sum_{i=1}^N\sum_{j = 1}^n \paren{\mathsf{Z}_{i,j} - \overline{\mathsf{Z}}} \paren{\mathsf{Z}_{i,j} - \overline{\mathsf{Z}}}^\top c + \frac{6}{Nn} \sum_{i=1}^N\sum_{j = 1}^n \left \langle \overline{\mathsf{Z}} - \overline{\mathsf{Y}}, c \right \rangle^2 \\
    \leq & ~ 12 \norm{c}_2^2 + \frac{6}{Nn} \sum_{i=1}^N\sum_{j = 1}^n \left \langle \overline{\mathsf{Z}} - \overline{\mathsf{Y}}, c \right \rangle^2 \\
    \leq & ~ \norm{c}_2^2 \paren{12 + \frac{6}{Nn} \sum_{i=1}^N\sum_{j = 1}^n \norm{\overline{\mathsf{Z}} - \overline{\mathsf{Y}}}_2^2} \\
    \leq & ~ 12 \norm{c}_2^2 \paren{1 + \norm{\overline{\mathsf{Z}} - \wt\mu^{\trtwo}}_2^2 + \norm{c}_2^2}
\end{aligned}
\right\},
\end{align}
where the second step follows from the definition of the inner product, the third step follows from our polynomial constraint (see Eq.~\eqref{eq:C_7}) and Fact~\ref{fac:sos_operator_norm}, the fourth step follows from the Cauchy-Schwarz inequality (see Fact~\ref{fac:cauchy_schwarz}), and the last step follows from the triangle inequality.
Considering the second term of Eq.~\eqref{eq:sec_term_more_2}, we have
\begin{align}\label{eq:sec_term_more_2_second_term}
\mathsf{B}\sststile{2}{\mathsf Z_{i,j}, \mathsf Y_i, W_{i,j}, U_i, B_1, B_2}\left\{
\begin{aligned}
    \frac{3}{Nn} \sum_{i=1}^N\sum_{j = 1}^n \left \langle z_{i,j}^{\mathrm{tr}} - \wt\mu^{\trtwo}, c \right \rangle^2
    = & ~ 3 c^\top \Sigma_z^{\mathrm{tr}} c \\
    \leq & ~ 6 \norm{c}_2^2
\end{aligned}
\right\},
\end{align}
where $\Sigma_z^{\mathrm{tr}}$ denotes the empirical second-moment matrix of the realized sample-level truncated values $\{z_{i,j}^{\mathrm{tr}}\}_{i \in [N],j \in [n]}$ centered at $\wt\mu^{\trtwo}$. This matrix is controlled by the sample-level truncation/concentration bounds, up to the harmless shift between $\wt\mu^{\mathrm{tr}}$ and $\wt\mu^{\trtwo}$.
Combining Eq.~\eqref{eq:sec_term_more_2}, Eq.~\eqref{eq:sec_term_more_2_first_term}, and Eq.~\eqref{eq:sec_term_more_2_second_term}, we can bound the second part of Eq.~\eqref{eq:sec_term_more} as follows: 
\begin{align}\label{eq:sec_term_more_2_final}
\mathsf{B}\sststile{6}{\mathsf Z_{i,j}, \mathsf Y_i, W_{i,j}, U_i, B_1, B_2}\left\{
\begin{aligned}
    \frac{1}{Nn} \sum_{i=1}^N\sum_{j = 1}^n\left \langle \mathsf{Z}_{i,j} - z_{i,j}^{\mathrm{tr}}, c \right \rangle^2 
    \leq 12 \norm{c}_2^2 \paren{1.5 + \norm{\overline{\mathsf{Z}} - \wt\mu^{\trtwo}}_2^2 + 1.25 \norm{c}_2^2}
\end{aligned}
    \right\}.
\end{align}
Therefore, combining Eq.~\eqref{eq:sec_term_more_1} and Eq.~\eqref{eq:sec_term_more_2_final}, we can bound the Eq.~\eqref{eq:sec_term_more} as follows:
\begin{align}\label{eq:sec_term_final}
\mathsf{B}\sststile{6}{\mathsf Z_{i,j}, \mathsf Y_i, W_{i,j}, U_i, B_1, B_2}\left\{
\begin{aligned}
    &2\left(
        \frac1N\sum_{i=1}^N
        U_i\mathcal I_i
        \left\langle
            \mathsf Y_i-\wt\mu_i^{\trtwo},
            \overline{\mathsf Y}-\wt\mu^{\trtwo}
        \right\rangle
    \right)^2\\
    &\le
    O(\alpha)\norm{c}_2^2
    \left(
        1+
        \norm{\overline{\mathsf Z}-\wt\mu^{\trtwo}}_2^2+
        \norm{c}_2^2
    \right).
    \end{aligned}
    \right\}.
\end{align}
Combining the bounds of both terms (Eq.~\eqref{eq:first_term} and Eq.~\eqref{eq:sec_term_final}) of Eq.~\eqref{eq:bounding_ov_Y_wt_mu} together, we have
\begin{align}\label{eq:bound_without_exp}
\mathsf{B}\sststile{6}{\mathsf Z_{i,j}, \mathsf Y_i, W_{i,j}, U_i, B_1, B_2}\left\{
\begin{aligned}
    & ~ \norm{\overline{\mathsf{Y}} - \wt\mu^{\trtwo}}_2^4 \\
    \le & ~ O\paren{\eps} \paren{4\paren{\frac{1}{n} + \tau}\norm{\overline{\mathsf{Y}}-\wt{\mu}^{\trtwo}}_2^2+\norm{\overline{\mathsf{Y}}-\wt{\mu}^{\trtwo}}_2^4}  \\
    + & ~ O\paren{\alpha} \paren{6\norm{\overline{\mathsf{Y}}-\wt{\mu}^{\trtwo}}_2^2 + 4\norm{\overline{\mathsf{Y}}-\wt{\mu}^{\trtwo}}_2^2 \norm{\overline{\mathsf{Z}} - \wt\mu^{\trtwo}}_2^2 + 5 \norm{\overline{\mathsf{Y}}-\wt{\mu}^{\trtwo}}_2^4}
    \end{aligned}
    \right\}.
\end{align}
By Lemma~\ref{lem:sos_mean}, by taking $\eps + \alpha$ as the given level of corruption, we also have
\begin{align}\label{eq:naive_sos}
\mathsf{B}\sststile{6}{\mathsf Z_{i,j}, \mathsf Y_i, W_{i,j}, U_i, B_1, B_2}\left\{
    \norm{\overline{\mathsf{Z}}-\wt{\mu}^{\trtwo}}_2^4\le O\paren{\eps+\alpha}\paren{4\norm{\overline{\mathsf{Z}}-\wt{\mu}^{\trtwo}}_2^2+\norm{\overline{\mathsf{Z}}-\wt{\mu}^{\trtwo}}_2^4}
\right\}.
\end{align}
Now, we consider the bound on the pseudoexpectation of Eq.~\eqref{eq:bound_without_exp}:
\begin{align}\label{eq:pseudo_ov_y}
    & ~ \wt\E\bracket{\norm{\overline{\mathsf{Y}}-\wt{\mu}^{\trtwo}}_2^4} \notag \\
    \le & ~ \wt\E\bracket{
    \begin{aligned}
        & ~ O\paren{\eps} \paren{4\paren{\frac{1}{n} + \tau}\norm{\overline{\mathsf{Y}}-\wt{\mu}^{\trtwo}}_2^2+\norm{\overline{\mathsf{Y}}-\wt{\mu}^{\trtwo}}_2^4}\\
        + & ~ O\paren{\alpha} \norm{\overline{\mathsf{Y}}-\wt{\mu}^{\trtwo}}_2^2 \paren{6 + 4\norm{\overline{\mathsf{Z}} - \wt\mu^{\trtwo}}_2^2 + 5 \norm{\overline{\mathsf{Y}}-\wt{\mu}^{\trtwo}}_2^2}
    \end{aligned}
    } \notag\\
    \le & ~ O\paren{\eps} \paren{4\paren{\frac{1}{n} + \tau}\wt\E\bracket{\norm{\overline{\mathsf{Y}}-\wt{\mu}^{\trtwo}}_2^2} + \wt\E\bracket{\norm{\overline{\mathsf{Y}}-\wt{\mu}^{\trtwo}}_2^4}} \notag\\
    + & ~ O\paren{\alpha} \paren{6 \wt\E\bracket{ \norm{\overline{\mathsf{Y}}-\wt{\mu}^{\trtwo}}_2^2} + 4\wt\E\bracket{\norm{\overline{\mathsf{Y}}-\wt{\mu}^{\trtwo}}_2^2 \norm{\overline{\mathsf{Z}} - \wt\mu^{\trtwo}}_2^2} + 5 \wt\E\bracket{\norm{\overline{\mathsf{Y}}-\wt{\mu}^{\trtwo}}_2^4}}
\end{align}
where the second step follows from the linearity of pseudoexpectation (see Fact~\ref{fac:linearity}).
In particular, by Fact~\ref{fact:cs_pseudo}, we have
\begin{align}\label{eq:cs_pseudo}
    \wt\E\bracket{\norm{\overline{\mathsf{Y}}-\wt{\mu}^{\trtwo}}_2^2 \norm{\overline{\mathsf{Z}} - \wt\mu^{\trtwo}}_2^2} \leq \sqrt{\wt\E\bracket{\norm{\overline{\mathsf{Y}}-\wt{\mu}^{\trtwo}}_2^4}} \sqrt{\wt\E\bracket{\norm{\overline{\mathsf{Z}} - \wt\mu^{\trtwo}}_2^4}}
\end{align}

Considering $\wt\E\bracket{\norm{\overline{\mathsf{Z}}-\wt{\mu}^{\trtwo}}_2^4}$, by Eq.~\eqref{eq:naive_sos}, we have
\begin{align*}
    \wt\E\bracket{\norm{\overline{\mathsf{Z}}-\wt{\mu}^{\trtwo}}_2^4}
    \leq & ~ \wt\E\bracket{O\paren{\eps+\alpha}\paren{4\norm{\overline{\mathsf{Z}}-\wt{\mu}^{\trtwo}}_2^2+\norm{\overline{\mathsf{Z}}-\wt{\mu}^{\trtwo}}_2^4}}\\
    \leq & ~ O\paren{\eps+\alpha}\paren{4 \wt\E\bracket{\norm{\overline{\mathsf{Z}}-\wt{\mu}^{\trtwo}}_2^2} + \wt\E\bracket{\norm{\overline{\mathsf{Z}}-\wt{\mu}^{\trtwo}}_2^4}},
\end{align*}
where the first step follows from Eq.~\eqref{eq:naive_sos} and the second step follows from the linearity of pseudoexpectation (see Fact~\ref{fac:linearity}).
This implies that
\begin{align*}
    \paren{1 - O\paren{\eps+\alpha}}
    \wt\E\bracket{\norm{\overline{\mathsf{Z}}-\wt{\mu}^{\trtwo}}_2^4}
    \leq
    O\paren{\eps+\alpha}
    \wt\E\bracket{\norm{\overline{\mathsf{Z}}-\wt{\mu}^{\trtwo}}_2^2}.
\end{align*}
For sufficiently small universal $\eps+\alpha$, we may absorb the left-hand factor and obtain Eq.~\eqref{eq:pseudo_ov_z4} below.

Therefore, we can get
\begin{align}\label{eq:pseudo_ov_z4}
    \wt\E\bracket{\norm{\overline{\mathsf{Z}}-\wt{\mu}^{\trtwo}}_2^4} = O\paren{\eps + \alpha} \wt\E\bracket{\norm{\overline{\mathsf{Z}}-\wt{\mu}^{\trtwo}}_2^2}.
\end{align}

By \textbf{Part 1} of Fact~\ref{fac:pseudoexpectation}, we can get
\begin{align*}
    \wt\E\bracket{\norm{\overline{\mathsf{Z}}-\wt{\mu}^{\trtwo}}_2^2}^2 \leq \wt\E\bracket{\norm{\overline{\mathsf{Z}}-\wt{\mu}^{\trtwo}}_2^4}
\end{align*}
so that combining with Eq.~\eqref{eq:pseudo_ov_z4} gives us
\begin{align}\label{eq:sec_degree_Z_mu}
    \wt\E\bracket{\norm{\overline{\mathsf{Z}}-\wt{\mu}^{\trtwo}}_2^2} \leq O\paren{\eps + \alpha}
\end{align}
Furthermore, combining Eq.~\eqref{eq:sec_degree_Z_mu} with Eq.~\eqref{eq:pseudo_ov_z4} again implies 
\begin{align}\label{eq:pseudo_ov_z2}
    \sqrt{\wt\E\bracket{\norm{\overline{\mathsf{Z}}-\wt{\mu}^{\trtwo}}_2^4}} \leq \sqrt{O\paren{\eps + \alpha} \cdot O\paren{\eps + \alpha}} = O\paren{\eps + \alpha}.
\end{align}

Therefore, combining Eq.~\eqref{eq:pseudo_ov_y}, Eq.~\eqref{eq:pseudo_ov_z2}, and Eq.~\eqref{eq:cs_pseudo}, we obtain
\begin{align*}
    &\wt\E\bracket{\norm{\overline{\mathsf{Y}}-\wt{\mu}^{\trtwo}}_2^4} \\
    \leq{}&
    O\paren{\frac{\eps}{n}+\eps\tau+\alpha}
    \wt\E\bracket{\norm{\overline{\mathsf{Y}}-\wt{\mu}^{\trtwo}}_2^2}
    +O\paren{\eps+\alpha}
    \wt\E\bracket{\norm{\overline{\mathsf{Y}}-\wt{\mu}^{\trtwo}}_2^4}\\
    &\quad+
    O\paren{\alpha\paren{\eps+\alpha}}
    \sqrt{\wt\E\bracket{\norm{\overline{\mathsf{Y}}-\wt{\mu}^{\trtwo}}_2^4}}\\
    ={}&
    O\paren{\frac{\eps}{n}+\alpha}
    \wt\E\bracket{\norm{\overline{\mathsf{Y}}-\wt{\mu}^{\trtwo}}_2^2}
    +O\paren{\eps+\alpha}
    \wt\E\bracket{\norm{\overline{\mathsf{Y}}-\wt{\mu}^{\trtwo}}_2^4}\\
    &\quad+
    O\paren{\alpha\paren{\eps+\alpha}}
    \sqrt{\wt\E\bracket{\norm{\overline{\mathsf{Y}}-\wt{\mu}^{\trtwo}}_2^4}},
\end{align*}
where the equality uses \(\tau=\alpha/\eps\). For sufficiently small universal \(\eps+\alpha\), the fourth-moment term on the right-hand side can be absorbed into the left-hand side. Consequently,
\begin{align*}
\wt\E\bracket{\norm{\overline{\mathsf{Y}}-\wt{\mu}^{\trtwo}}_2^4}
\leq{}&
O\paren{\frac{\eps}{n}+\alpha}
\wt\E\bracket{\norm{\overline{\mathsf{Y}}-\wt{\mu}^{\trtwo}}_2^2}
+O\paren{\alpha\paren{\eps+\alpha}}
\sqrt{\wt\E\bracket{\norm{\overline{\mathsf{Y}}-\wt{\mu}^{\trtwo}}_2^4}}.
\end{align*}

Let $b := O\paren{\alpha\paren{\eps + \alpha}}$, $c := O\paren{\frac{\eps}{n} + \alpha}  \wt\E\bracket{\norm{\overline{\mathsf{Y}}-\wt{\mu}^{\trtwo}}_2^2}$, and $x := \sqrt{\wt\E\bracket{\norm{\overline{\mathsf{Y}}-\wt{\mu}^{\trtwo}}_2^4}}$. To make $x^2 - bx - c \leq 0$, we need
\begin{align*}
    \frac{b - \sqrt{b^2 + 4c}}{2} \leq x \leq \frac{b + \sqrt{b^2 + 4c}}{2}
\end{align*}
so that
\begin{align*}
    x^2 \leq \max \left \{ \paren{\frac{b - \sqrt{b^2 + 4c}}{2}}^2, \paren{\frac{b + \sqrt{b^2 + 4c}}{2}}^2 \right\}.
\end{align*}
By the definition of $b, c$, we know that $b, \sqrt{b^2 + 4c} \geq 0$, so 
\begin{align*}
    \max \left \{ \paren{\frac{b - \sqrt{b^2 + 4c}}{2}}^2, \paren{\frac{b + \sqrt{b^2 + 4c}}{2}}^2 \right\} = \paren{\frac{b + \sqrt{b^2 + 4c}}{2}}^2.
\end{align*}
Therefore, we have
\begin{align*}
    \wt\E\bracket{\norm{\overline{\mathsf{Y}}-\wt{\mu}^{\trtwo}}_2^4} 
    \leq & ~ \paren{\frac{O\paren{\alpha\paren{\eps + \alpha}} + \sqrt{O\paren{\alpha\paren{\eps + \alpha}}^2 + O\paren{\frac{\eps}{n} + \alpha}  \wt\E\bracket{\norm{\overline{\mathsf{Y}}-\wt{\mu}^{\trtwo}}_2^2}}}{2}}^2 \\
    \leq & ~ O\paren{\alpha^2\paren{\eps + \alpha}^2}  + O\paren{\frac{\eps}{n} + \alpha}  \wt\E\bracket{\norm{\overline{\mathsf{Y}}-\wt{\mu}^{\trtwo}}_2^2}.
\end{align*}
Since, by \textbf{Part 1} of Fact~\ref{fac:pseudoexpectation}, we have 
\begin{align*}
    \wt\E\bracket{\norm{\overline{\mathsf{Y}}-\wt{\mu}^{\trtwo}}_2^4} \geq \wt\E\bracket{\norm{\overline{\mathsf{Y}}-\wt{\mu}^{\trtwo}}_2^2}^2,
\end{align*}
we can get
\begin{align*}
    \wt\E\bracket{\norm{\overline{\mathsf{Y}}-\wt{\mu}^{\trtwo}}_2^2}^2 - O\paren{\alpha^2\paren{\eps + \alpha}^2}  - O\paren{\frac{\eps}{n} + \alpha}  \wt\E\bracket{\norm{\overline{\mathsf{Y}}-\wt{\mu}^{\trtwo}}_2^2} \leq 0.
\end{align*}
Similarly, we can get
\begin{align*}
    \wt\E\bracket{\norm{\overline{\mathsf{Y}}-\wt{\mu}^{\trtwo}}_2^2} 
    \leq & ~ \frac{O\paren{\frac{\eps}{n} + \alpha} + \sqrt{O\paren{\frac{\eps}{n} + \alpha}^2 + O\paren{\alpha^2\paren{\eps + \alpha}^2}}}{2} \\
    = & ~ O\paren{\frac{\eps}{n} + \alpha + \eps\alpha + \alpha^2} \\
    = & ~ O\paren{\frac{\eps}{n} + \alpha}.
\end{align*}

In addition, by \textbf{Part 3} of Fact~\ref{fac:pseudoexpectation}, we have
\begin{align}\label{eq:bound_in_pseudo}
    \norm{\wt\E\bracket{\overline{\mathsf{Y}}} - \wt{\mu}^{\trtwo}}_2^2
    \leq & ~
    \wt\E\bracket{
        \norm{\overline{\mathsf{Y}}-\wt{\mu}^{\trtwo}}_2^2
    } \notag\\
    \leq & ~ O\paren{\frac{\eps}{n}+\alpha}.
\end{align}

It remains to compare the twice-truncated empirical mean
$\wt{\mu}^{\trtwo}$ with the true mean $\mu$. First, for each sample $(i,j)$, we truncate the
clean random variable $Z_{i,j}$ and obtain $Z_{i,j}^{\mathrm{tr}}$. Then, for
each user $i$, we define
\[
    Z_i^{\mathrm{tr}}
    :=
    \frac1n\sum_{j=1}^n Z_{i,j}^{\mathrm{tr}},
\]
and apply a second, user-level truncation to obtain
$Z_i^{\trtwo}$. Let
\[
    \mu^{\mathrm{tr}}
    :=
    \E\bracket{Z_i^{\mathrm{tr}}},
    \qquad
    \mu^{\trtwo}
    :=
    \E\bracket{Z_i^{\trtwo}},
\]
and define the empirical average of the twice-truncated user-level variables by
\[
    \wt{\mu}^{\trtwo}
    :=
    \frac1N\sum_{i=1}^N z_i^{\trtwo}.
\]
By the triangle inequality,
\begin{align}
    \norm{\wt{\mu}^{\trtwo}-\mu}_2
    \leq&
    \norm{\wt{\mu}^{\trtwo}-\mu^{\trtwo}}_2
    +
    \norm{\mu^{\trtwo}-\mu^{\mathrm{tr}}}_2
    +
    \norm{\mu^{\mathrm{tr}}-\mu}_2 .
    \label{eq:twice_truncated_empirical_decomposition}
\end{align}

We first bound the empirical fluctuation term. Define
\[
    X_i
    :=
    Z_i^{\trtwo}-\mu^{\trtwo},
    \qquad i\in[N].
\]
Then $\E[X_i]=0$. Moreover, by the user-level truncation,
\[
    \norm{X_i}_2
    \leq
    O\left(\sqrt{\frac{d}{n\eps}}\right)
\]
almost surely. Also, since
$\cov\bracket{Z_i^{\trtwo}}\preceq \frac1n\mathbb I_d$, we have
\[
    \E\bracket{\norm{X_i}_2^2}
    =
    \mathrm{tr}\left[\cov\bracket{Z_i^{\trtwo}}\right]
    \leq
    \frac dn .
\]
Applying Fact~\ref{fac:vec_bernstein} to $\{X_i\}_{i=1}^N$, with probability
at least $1-\delta$, we get
\begin{align*}
    \norm{
        \sum_{i=1}^N
        \paren{Z_i^{\trtwo}-\mu^{\trtwo}}
    }_2
    \leq
    O\left(
        \sqrt{\frac{Nd\log(1/\delta)}{n}}
        +
        \sqrt{\frac{d}{n\eps}}\log(1/\delta)
    \right).
\end{align*}
Dividing both sides by $N$, we obtain
\begin{align*}
    \norm{\wt{\mu}^{\trtwo}-\mu^{\trtwo}}_2
    \leq
    O\left(
        \sqrt{\frac{d\log(1/\delta)}{Nn}}
        +
        \frac1N\sqrt{\frac{d}{n\eps}}\log(1/\delta)
    \right).
\end{align*}
Under the sample-complexity assumption
\[
    Nn
    \geq
    \Omega\left(\frac d\alpha\log\left(\frac d\delta\right)\right),
\]
the empirical fluctuation term is bounded by
\begin{align}
    \norm{\wt{\mu}^{\trtwo}-\mu^{\trtwo}}_2
    \leq
    O\paren{\sqrt{\alpha}}.
    \label{eq:twice_truncated_empirical_concentration}
\end{align}

Next, by the truncation-bias bounds from Lemma~\ref{lem:C_7} and
Lemma~\ref{lem:C_11}, we have
\[
    \norm{\mu^{\mathrm{tr}}-\mu}_2
    \leq
    O\paren{\sqrt{\alpha}},
    \qquad
    \norm{\mu^{\trtwo}-\mu^{\mathrm{tr}}}_2
    \leq
    O\paren{\sqrt{\frac{\eps}{n}}}.
\]
Combining these two bounds with
Eq.~\eqref{eq:twice_truncated_empirical_decomposition} and
Eq.~\eqref{eq:twice_truncated_empirical_concentration}, we get, with
probability at least $1-\delta$,
\begin{align}\label{eq:bound_in_empirical}
    \norm{\wt{\mu}^{\trtwo}-\mu}_2^2
    \leq
    O\paren{\frac{\eps}{n}+\alpha}.
\end{align}
In the parameter regime of Problem~\ref{prob:arbitrary_alpha}, where
$\alpha \gtrsim \eps/n$, this gives
\[
    \norm{\wt{\mu}^{\trtwo}-\mu}_2^2
    \leq
    O(\alpha).
\]

Therefore, we can bound the output of the SoS algorithm with the true mean
$\mu$ as follows:
\begin{align*}
    \norm{\wt\E\bracket{\overline{\mathsf{Y}}}-\mu}_2^2
    = & ~
    \norm{
        \wt\E\bracket{\overline{\mathsf{Y}}}
        -\wt{\mu}^{\trtwo}
        +\wt{\mu}^{\trtwo}-\mu
    }_2^2\\
    \leq & ~
    2\norm{
        \wt\E\bracket{\overline{\mathsf{Y}}}
        -\wt{\mu}^{\trtwo}
    }_2^2
    +
    2\norm{
        \wt{\mu}^{\trtwo}-\mu
    }_2^2\\
    \leq & ~
    O\paren{\frac{\eps}{n}+\alpha}.
\end{align*}
This completes the proof of Lemma~\ref{lem:mean_close_5}.
\end{proof}

\subsection{Technical Difficulties of Obtaining Sub-Gaussian Rate for Double Truncation}
\label{sub:arbitrary_alpha:subgau}

A natural first attempt is to combine the double-corruption analysis from
\Cref{sub:technique:ques1} with existing stability-based techniques for robust
mean estimation under local and global corruptions~\cite{DiakonikolasKP20,
nietert24localglobal,pp25}. These works obtain nearly sub-Gaussian sample
complexity by proving that, with high probability, the clean samples contain a
large stable subset; see, for example, \cite[Theorem~1.4 and
Proposition~1.5]{DiakonikolasKP20} and \cite[Theorem~3 and
Proposition~8]{pp25}. However, this stable-subset construction cannot be used
as a black box in our double-truncation framework. Their argument is based on a
data-dependent selection of a stable subset, while our SoS satisfiability proof
uses a deterministic truncation: clean samples \(z_{i,j}\) with large
\(\norm{z_{i,j}-\mu}_2\) are replaced by \(\mu\). This deterministic operation
preserves the product structure of the clean data. In particular, after the
first truncation, the batch mean
    $z_i^{\tr}:=\frac1n\sum_{j=1}^n z_{i,j}^{\tr}$
is still an average of independent truncated samples, so it has covariance at
scale \(1/n\). This property is essential for the second-level truncation.
By contrast, if one applies a data-dependent stable-subset construction inside
the first level, the resulting objects are no longer simple averages of
independent truncated samples. Consequently, the second-level truncation
argument from \Cref{sub:technique:ques1} cannot be applied directly.

A second attempt is to apply the stable-subset theorem separately inside each
batch. This would preserve independence across batches and would allow us to
run a second-level truncation on the resulting batch representatives. However,
it would require each batch to contain enough samples to certify stability on
its own, namely
    $n \ge
    \Omega\left(
        \frac{d+\log(N/\delta)}{\alpha}
    \right)$.
This is incompatible with our goal of handling the
regime \(n=o(d)\) and total sample complexity
    $nN \le \widetilde O\left(\frac d\alpha\right)$.
Thus, proving a high-dimensional stable subset separately within every batch is
too expensive.

A third attempt is to adapt the minimax-duality arguments from
\cite{DiakonikolasKP20,pp25} directly to our two-level truncation setting.
This is where the main technical divergence occurs. With only one level of corruption,
\cite{DiakonikolasKP20,pp25} relax the subset selection to weights vector \(w\in \Delta_{n ,\alpha} := \cbracket{w \in \R^n : 0 \leq w_j \leq 1/((1 - \alpha)n); \sum_{j \in [n]} w_j = 1}\), so the
covariance objective has the bilinear form
\[
    \min_{w\in\Delta_{n,\eps}}
    \max_{M\in\mathcal M}
    \sum_{i=1}^n
    w_i \langle M,(x_i-\mu)(x_i-\mu)^\top \rangle,
\]
where $\mathcal M := \{M : M \succeq 0, \tr\bracket{M} = 1\}$. Because this objective is linear in \(w\) and linear in \(M\), one can apply
minimax duality to exchange the order of \(\min_w\) and \(\max_M\). This
exchange is crucial: for each fixed \(M\), one can remove the samples with
large quadratic scores
    $(x_i-\mu)^\top M(x_i-\mu)$,
and then argue that there exists a single large subset whose covariance is
bounded in all directions.

In our two-level setting, the user-level covariance constraint is imposed on
cleaned batch means. If we try to imitate the same argument, the natural object
is
    $\mathsf Z_i
    =
    \frac1n\sum_{j=1}^n W_{i,j}(z_{i,j}-\mu_i)$,
and the corresponding user-level quadratic form is
    $\sum_{i=1}^N
    U_i
    \langle
        M,
        \mathsf Z_i\mathsf Z_i^\top \rangle$.
Expanding this expression gives
\[
    \frac1{n^2}
    \sum_{i=1}^N
    \sum_{j,k=1}^n
    U_i W_{i,j}W_{i,k}
    \langle z_{i,j}-\mu_i, M(z_{i,k}-\mu_i)\rangle.
\]
Thus the objective contains cross terms \(W_{i,j}W_{i,k}\). It is no longer
bilinear in the selection variables and the dual matrix \(M\). Consequently,
the minimax-duality step used in the one-level arguments does not
apply.

\subsection{Technical Difficulties of Sequential Filtering}
\label{sub:arbitrary_alpha:filter}
We give a simple example illustrating an obstruction to first applying a
sample-level filtering algorithm and then filtering users.  For simplicity,
we work in one dimension and set the sample-level corruption rate to
$\alpha=0$.  Let $\eps<1/2$, and define
\[
    p:=\frac{\eps}{2(1-\eps)},
    \qquad
    a:=\frac{1}{\sqrt{p}}
    =
    \sqrt{\frac{2(1-\eps)}{\eps}}.
\]

Consider the following two instances.

\textbf{Instance $+$.}
The clean distribution is
\[
    P_+
    =
    (1-p)\delta_0+p\delta_a,
\]
where $\delta_x$ denotes the point mass at $x$.  Let $K_+$ denote the
realized number of clean samples equal to $+a$.  An $\eps$-fraction of
users are corrupted, and the adversary fills the corrupted batches with
exactly $K_+$ copies of $-a$, placing all remaining corrupted samples at
$0$.

\textbf{Instance $-$.}
The clean distribution is the reflection of $P_+$,
\[
    P_-
    =
    (1-p)\delta_0+p\delta_{-a}.
\]
Let $K_-$ denote the realized number of clean samples equal to $-a$.
An $\eps$-fraction of users are corrupted, and the adversary fills the
corrupted batches with exactly $K_-$ copies of $+a$, placing all remaining
corrupted samples at $0$.

Both clean distributions satisfy the covariance assumption.  Indeed,
\[
    \E_{P_+}[X]
    =
    pa
    =
    \sqrt{p},
    \qquad
    \E_{P_-}[X]
    =
    -\sqrt{p},
\]
while
\[
    \E_{P_\pm}[X^2]
    =
    pa^2
    =
    1.
\]
Consequently,
\[
    \Var_{P_\pm}[X]
    =
    1-p
    \leq 1.
\]

Let $M:=Nn$ denote the total number of samples.  Since there are
$(1-\eps)M$ samples belonging to clean users,
\[
    K_\pm
    \sim
    \mathrm{Bin}\bigl((1-\eps)M,p\bigr),
\]
and therefore
\[
    \E[K_\pm]
    =
    (1-\eps)Mp
    =
    \frac{\eps M}{2}.
\]
By a Chernoff bound, for any fixed $\delta\in(0,1)$,
\[
\Pr\left[
    \left|K_\pm-\frac{\eps M}{2}\right|
    \ge
    \delta\frac{\eps M}{2}
\right]
\le
2\exp\left(-\Omega(\delta^2\eps M)\right).
\]
Hence, with high probability,
\[
    K_\pm
    =
    (1\pm\delta)\frac{\eps M}{2}
    <
    \eps M,
\]
so the corrupted users contain enough sample locations for the adversary
to realize the construction.

Now suppose that the first-stage filtering algorithm ignores the user
identities and only sees the pooled samples.  Since $K_+$ and $K_-$ have
the same distribution, we may couple the two instances so that
\[
    K_+=K_-=:K.
\]
Conditioned on this realization, the pooled samples in \emph{both}
instances consist of exactly
\[
    K \text{ copies of } +a,\qquad
    K \text{ copies of } -a,\qquad
    M-2K \text{ copies of } 0.
\]
Thus, writing
\[
    q:=\frac{K}{M},
\]
the pooled empirical distribution in either instance is
\[
    \widehat Q_K
    =
    (1-2q)\delta_0
    +
    q\delta_a
    +
    q\delta_{-a}.
\]
In particular, a procedure that ignores the batch structure receives
exactly the same pooled data in the two instances.

The pooled empirical distribution is symmetric around zero, so
\[
    \E_{\widehat Q_K}[X]=0.
\]
Moreover,
\[
    \Var_{\widehat Q_K}[X]
    =
    2qa^2.
\]
Since, with high probability,
\[
    q
    =
    (1\pm o(1))\frac{\eps}{2}
\]
and
\[
    a^2
    =
    \frac{2(1-\eps)}{\eps},
\]
we obtain
\[
    \Var_{\widehat Q_K}[X]
    =
    2(1-\eps)(1\pm o(1)).
\]
Hence, for sufficiently small $\eps$, the pooled empirical variance is
close to $2$, whereas the variance of either clean distribution is at
most $1$.  A covariance-based sample-level filtering procedure must
therefore modify the sample set before reaching a stopping condition with
variance close to the clean covariance bound.

To see why this forces the filter to remove clean samples in one of the
two instances, let $\xi_+$ and $\xi_-$ denote the fractions of the total
sample set removed from the atoms at $+a$ and $-a$, respectively, and let
$R$ denote the total fraction of samples retained by the filter.  For any
acceptable filtering procedure we have
\[
    R=1-O(\eps).
\]
The empirical variance of the retained distribution can be lower bounded as
\[
\begin{aligned}
    \widehat{\Var}(X)
    &=
    a^2
    \left(
        \frac{2q-\xi_+-\xi_-}{R}
        -
        \frac{(\xi_+-\xi_-)^2}{R^2}
    \right) \\
    &\ge
    a^2
    \left(
        \frac{2q-\xi_+-\xi_-}{R}
        -
        \frac{(\xi_++\xi_-)^2}{R^2}
    \right) \\
    &=
    a^2
    \left(
        \frac{2q-s}{R}
        -
        \frac{s^2}{R^2}
    \right),
\end{aligned}
\]
where $s:=\xi_++\xi_-$, and the inequality follows from
$|\xi_+-\xi_-|\leq \xi_++\xi_-$.  Suppose, towards a contradiction, that
\[
    \xi_++\xi_-=o(\eps).
\]
Since the filtering procedure retains $1-O(\eps)$ of the samples, we have
$R=1-O(\eps)$.  Moreover, with high probability,
$q=(1\pm o(1))\eps/2$.  Therefore,
\[
    a^2\frac{2q-s}{R}
    =
    2-o(1),
\]
while
\[
    a^2\frac{s^2}{R^2}
    =
    o(1).
\]
Hence
\[
    \widehat{\Var}(X)
    \geq
    2-o(1).
\]
Thus, for any fixed $\tau<1$ and sufficiently small $\eps$, the filter
cannot reach the stopping condition
$\widehat{\Var}(X)\leq 1+\tau$ while removing only $o(\eps)$ mass from
the two extreme components.  Consequently, it must remove
\[
    \xi_++\xi_-=\Omega(\eps),
\]
and hence at least one of the components at $+a$ or $-a$ loses
$\Omega(\eps)$ mass.

The crucial point is that the first-stage filtering procedure sees exactly
the same pooled samples in the two instances and therefore cannot determine
which extreme component is genuinely clean.  Thus, removing $\Omega(\eps)$ mass from either extreme component necessarily
removes $\Omega(\eps)$ clean mass in one of the two instances.  Thus, in at least one of
the two instances, the sample-level filtering step necessarily discards
too many clean samples.

Importantly, this error cannot be repaired by the subsequent user-level
filtering step.  Even if the second-stage filter performs optimally and
perfectly identifies and removes all corrupted users, the clean samples
discarded in the first stage are already lost.  Since these samples lie at
distance
\[
    a=\Theta(\eps^{-1/2}),
\]
removing $\Omega(\eps)$ clean mass from the minority component changes the
mean by at least
\[
    \Omega(\eps)\cdot\Theta(\eps^{-1/2})
    =
    \Omega(\sqrt{\eps}).
\]
Thus, even after all corrupted users are subsequently removed, the estimator
still suffers $\Omega(\sqrt{\eps})$ error solely because the first-stage
sample filter discarded too many clean samples.

Equivalently, the two possible clean means satisfy
\[
    \left|
        \E_{P_+}[X]-\E_{P_-}[X]
    \right|
    =
    2\sqrt{p}
    =
    \Theta(\sqrt{\eps}).
\]
Therefore, a sample-level refinement chosen without using the user structure
may irreversibly incur $\Omega(\sqrt{\eps})$ error, despite there being no
sample-level corruption at all.  This illustrates the fundamental obstacle
to sequential filtering: the sample-level filter must commit to a refinement
before knowing which users are corrupted.  Our SoS formulation instead solves
jointly for the sample indicators $W_{i,j}$ and user indicators $U_i$, so that
the sample-level and user-level refinements are chosen simultaneously.

%% file: 3_adversarial_mean_shift.tex
\section{Learning From Doubly-Corrupted and Heterogeneous Batches}
\label{sec:combined}

In this section we combine the two sources of statistical difficulty studied in
the preceding sections. At the user level, an $\eps$-fraction of batches may be
fully adversarial. At the sample level, even inside an uncorrupted batch, an
$\alpha$-fraction of samples may be replaced by arbitrary points. Finally, the
uncorrupted users may be heterogeneous: user $i$ has mean $\mu_i$, and these
means are only assumed to lie within distance $\sqrt\beta$ of the target mean
$\mu$.

\ProbHeterogeneousTwoLevel*

\begin{definition}[Polynomial system $\mathsf D$]
\label{def:D}
Let $C_D >0$ be sufficiently large universal
constants, and let $c_\rho>0$ be a sufficiently small universal constant.
Define
\[
    \eta:=\max\{\alpha,\beta\},
    \qquad
    \rho:=\min\{\max\{\eps,n\eta\},c_\rho\},
    \qquad
    \tau:=\frac{\eta}{\rho}.
\]
For each $i\in[N]$ and $j\in[n]$, let $x_{i,j}\in\R^d$ denote the observed
data point. The variable $W_{i,j}\in\{0,1\}$ indicates whether the sample
$(i,j)$ is selected as an observed agreement, and the variable
$U_i\in\{0,1\}$ indicates whether user $i$ is selected at the user level. The
SoS variable $\mathsf Z_{i,j}\in\R^d$ represents the cleaned sample-level
point, and
\[
    \mathsf Z_i:=\frac1n\sum_{j=1}^n \mathsf Z_{i,j}
\]
represents the cleaned empirical mean of user $i$. The SoS variable
$\mathsf Y_i\in\R^d$ represents the cleaned user-level mean used for the final
aggregation. Let $B_1,B_2\in\R^{d\times d}$ be matrix slack variables.

The polynomial system $\mathsf D$ consists of the following constraints:
\begin{align}
    \intertext{\textbf{Crude sample-level refinement:}}
    \mathsf Z_i
    &=
    \frac1n \sum_{j=1}^n \mathsf Z_{i,j},
    &&\forall i\in[N],                                      \label{eq:D_Z_i}\\
    \overline{\mathsf Z}
    &=
    \frac1{Nn}\sum_{i=1}^N\sum_{j=1}^n \mathsf Z_{i,j},
    \label{eq:D_Z_bar}\\
    W_{i,j}^2
    &=
    W_{i,j},
    &&\forall i\in[N],\ j\in[n],                            \label{eq:D_W_bool}\\
    W_{i,j}(\mathsf Z_{i,j}-x_{i,j})
    &=
    0,
    &&\forall i\in[N],\ j\in[n],                            \label{eq:D_sample_agree}\\
    \sum_{i=1}^N\sum_{j=1}^n U_iW_{i,j}
    &=
    (1- \Theta\paren{\eta})n\sum_{i=1}^N U_i,
    \label{eq:D_global_sample_count}\\
    \frac1{Nn}\sum_{i=1}^N\sum_{j=1}^n
    (\mathsf Z_{i,j}-\overline{\mathsf Z})
    (\mathsf Z_{i,j}-\overline{\mathsf Z})^\top
    &=
    C_D\mathbb I_d-B_1B_1^\top.
    \label{eq:D_sample_cov}\\
    \intertext{\textbf{User-level refinement:}}
    U_i^2
    &=
    U_i,
    &&\forall i\in[N],                                      \label{eq:D_U_bool}\\
    \sum_{i=1}^N U_i
    &=
    (1-\Theta\paren{\rho})N,                                  \label{eq:D_user_count}\\
    U_i(\mathsf Y_i-\mathsf Z_i)
    &=
    0,
    &&\forall i\in[N],                                      \label{eq:D_user_agree}\\
    \overline{\mathsf Y}
    &=
    \frac1N\sum_{i=1}^N \mathsf Y_i,
    \label{eq:D_Y_bar}\\
    \frac1N\sum_{i=1}^N
    (\mathsf Y_i-\overline{\mathsf Y})
    (\mathsf Y_i-\overline{\mathsf Y})^\top
    &=
    C_D\left(\frac1n+\beta+\tau\right)\mathbb I_d
    -B_2B_2^\top.
    \label{eq:D_user_cov}
\end{align}
\end{definition}

\subsection{Satisfiability}
\label{sub:heterogeneous_two_level:satisfy}

We prove that the polynomial system $\mathsf D$ is satisfiable with high
probability. We view all users as initially clean: for each $i\in[N]$, the
clean samples $Z_{i,1},\ldots,Z_{i,n}$ are drawn independently from $P_i$,
where $\E[Z_{i,j}]=\mu_i$, $\cov[Z_{i,j}]\preceq\mathbb I_d$, and
$\|\mu_i-\mu\|_2\le\sqrt\beta$. After these clean batches are generated, the
adversary may replace an $\eps$-fraction of entire batches and, in each
remaining batch, may replace an $\alpha$-fraction of samples.

\begin{lemma}[Satisfiability of $\mathsf D$]
\label{lem:D_satisfiability}
Under the setting of Problem~\ref{prob:heterogeneous_two_level}, with probability at least $1-\delta$, there exists an assignment to the SoS variables $W_{i, j}, \mathsf Z_{i, j}, \mathsf Y_{i}, U_i, B_1, B_2$ in Definition~\ref{def:D} satisfying all constraints of
$\mathsf D$, as long as
\[
    nN
    \ge
    \Omega\left(
        \frac d\eta
        \log\left(\frac d\delta\right)
    \right).
\]
\end{lemma}

\begin{proof}
We first define the two truncation steps, then construct a feasible assignment for
the polynomial system.

\textbf{Sample-level truncation.}
For every $i\in[N]$ and $j\in[n]$, we define
\[
    Z_{i,j}^{\tr}
    :=
    \begin{cases}
        Z_{i,j},
        &\text{if }\|Z_{i,j}-\mu_i\|_2\le 2\sqrt{d/\eta},\\
        \mu_i,
        &\text{otherwise}.
    \end{cases}
\]
Let $Z_i^{\tr}:=\frac1n\sum_{j=1}^n Z_{i,j}^{\tr}$ and $\mu_i^{\tr}:=\E[Z_i^{\tr}]$.
Applying Claim~\ref{cla:covariance-condition_part_1_4} to the $Nn$ clean
sample-level variables, with covariance scale $\gamma=1$ and truncation
parameter $\eta$, gives with probability at least $1-\exp\paren{-\frac{\eta N}{4}}$,
\begin{equation}
    \frac1{Nn}\sum_{i=1}^N\sum_{j=1}^n
    \indic{Z_{i,j}\neq Z_{i,j}^{\tr}}
    \le \eta .
    \label{eq:D_sample_tr_count_event}
\end{equation}
To ensure $\exp\paren{-\frac{\eta N}{4}} \leq \delta / 4$, we need
    $N \geq \frac{4}{\eta}\log\paren{4 / \delta}$.
Since $n = o\paren{d}$ and $nN
    \ge
    \Omega\left(
        \frac d\eta
        \log\left(\frac d\delta\right)
    \right)$, we can get
    \begin{align*}
        nN 
        \geq & ~ \Omega\left(
        \frac{n}{\eta}
        \log\left(\frac d\delta\right)
    \right)\\
        \geq & ~ \frac{4n}{\eta}\log\paren{4 / \delta}.
    \end{align*}
By Claim~\ref{cla:covariance-condition_part_3_4}, we have for every $i\in[N]$,
\begin{equation}
    \|\mu_i^{\tr}-\mu_i\|_2
    \le
    O(\sqrt\eta).
    \label{eq:D_sample_tr_bias}
\end{equation}

\begin{claim}[Sample-level empirical covariance after truncation]
\label{cla:D_sample_covariance}
With probability at least $1-\delta/4$,
\[
    \frac1{Nn}
    \sum_{i=1}^N\sum_{j=1}^n
    (Z_{i,j}^{\tr}-\mu)(Z_{i,j}^{\tr}-\mu)^\top
    \preceq
    C\mathbb I_d
\]
as long as
\[
    nN
    \ge
    \Omega\left(
        \frac d\eta
        \log\left(\frac d\delta\right)
    \right).
\]
\end{claim}

\begin{proof}
For each $(i,j)\in[N]\times[n]$, define
\[
    M_{i,j}
    :=
    \frac1{Nn}
    \paren{Z_{i,j}^{\mathrm{tr}}-\mu}
    \paren{Z_{i,j}^{\mathrm{tr}}-\mu}^\top .
\]
By the definition of the sample-level truncation,
\[
    \norm{Z_{i,j}^{\mathrm{tr}}-\mu_i}_2
    \le
    2\sqrt{\frac d\eta}.
\]
Therefore,
\begin{align*}
    \norm{Z_{i,j}^{\mathrm{tr}}-\mu}_2
    &\le
    \norm{Z_{i,j}^{\mathrm{tr}}-\mu_i}_2
    +
    \norm{\mu_i-\mu}_2                                      \\
    &\le
    2\sqrt{\frac d\eta}
    +
    \sqrt{\beta}.
\end{align*}
Hence
\[
    \norm{M_{i,j}}_2
    \le
    O\left(
        \frac{d/\eta+\beta}{Nn}
    \right).
\]
Thus
\[
    \mathbf 0_{d\times d}
    \preceq
    M_{i,j}
    \preceq
    R\mathbb I_d,
    \qquad
    R
    =
    O\left(
        \frac{d/\eta+\beta}{Nn}
    \right).
\]

Next, since the truncation maps $Z_{i,j}$ either to itself or to its mean
$\mu_i$, it does not increase the second moment around $\mu_i$. Thus
\[
    \cov[Z_{i,j}^{\mathrm{tr}}]
    \preceq
    \E\left[
        \paren{Z_{i,j}^{\mathrm{tr}}-\mu_i}
        \paren{Z_{i,j}^{\mathrm{tr}}-\mu_i}^\top
    \right]
    \preceq
    \cov[Z_{i,j}]
    \preceq
    \mathbb I_d .
\]
Moreover,
\[
    \norm{\mu_i^{\mathrm{tr}}-\mu}_2^2
    \le
    O(\eta+\beta)
    \le
    O(1),
\]
where the first inequality uses the truncation mean-shift bound and
$\norm{\mu_i-\mu}_2\le\sqrt{\beta}$. Therefore,
\[
    \E\left[
        \frac1{Nn}
        \sum_{i=1}^N\sum_{j=1}^n
        \paren{Z_{i,j}^{\mathrm{tr}}-\mu}
        \paren{Z_{i,j}^{\mathrm{tr}}-\mu}^\top
    \right]
    \preceq
    C_0\mathbb I_d
\]
for a universal constant $C_0>0$.
Choose a sufficiently large universal constant $C_1>C_0$. Applying the matrix
Chernoff inequality gives
\begin{align*}
    \Pr\left[
        \frac1{Nn}
        \sum_{i=1}^N\sum_{j=1}^n
        \paren{Z_{i,j}^{\mathrm{tr}}-\mu}
        \paren{Z_{i,j}^{\mathrm{tr}}-\mu}^\top
        \succ
        C_1\mathbb I_d
    \right]
    &\le
    d\exp\left(
        -\Omega\left(
            \frac1R
        \right)
    \right)                                               \\
    &\le
    d\exp\left(
        -\Omega\left(
            \frac{Nn}{d/\eta+\beta}
        \right)
    \right).
\end{align*}
Since $\beta\le\eta$ and $\eta\in(0,1)$, we have
    $d/\eta+\beta
    \le
    O(d/\eta)$.
Thus the failure probability is at most
\[
    d\exp\left(
        -\Omega\left(
            \frac{\eta Nn}{d}
        \right)
    \right).
\]
This is at most $\delta / 4$ as long as
\[
    nN
    \ge
    \Omega\left(
        \frac d\eta
        \log\left(\frac d\delta\right)
    \right). \qedhere
\]
\end{proof}

\textbf{User-level truncation.}
Since $\cov[Z_i^{\tr}]\preceq \frac{1}{n}\mathbb I_d$, define
\[
    Z_i^{\trtwo}
    :=
    \begin{cases}
        Z_i^{\tr},
        &\text{if }\|Z_i^{\tr}-\mu_i^{\tr}\|_2
        \le 2\sqrt{d/(n\rho)},\\
        \mu_i^{\tr},
        &\text{otherwise}.
    \end{cases}
\]
Let $\mu_i^{\trtwo}:=\E[Z_i^{\trtwo}]$. Applying
Claim~\ref{cla:covariance-condition_part_1_4} with covariance scale $1/n$ and
truncation parameter $\rho$, if $N\rho\ge\Omega(\log(1/\delta))$, then
with probability at least $1-\delta/4$,
\begin{equation}
    \frac1N\sum_{i=1}^N
    \indic{Z_i^{\tr}\neq Z_i^{\trtwo}}
    \le \rho.
    \label{eq:D_user_tr_count_event}
\end{equation}
Since $nN
    \ge
    \Omega\left(
        \frac d\eta
        \log\left(\frac d\delta\right)
    \right)$ and $n = o(d)$, we can get $N \ge
    \Omega\left(
        \frac 1\eta
        \log\left(\frac d\delta\right)
    \right)$.
In addition, since $\rho=\min\{\max\{\eps,n\eta\},c_\rho\}$, we consider the cases when $\rho = c_\rho$ and $\rho = \max\{\eps,n\eta\}$. First, suppose $\rho = c_\rho$. We can get $N\rho\ge c_\rho \Omega\left(
        \frac 1\eta
        \log\left(\frac d\delta\right)
    \right) \geq \Omega(\log(1/\delta))$. Second, suppose $\rho = \max\{\eps,n\eta\} \geq n\eta$. We have
    \begin{align}\label{eq:nrho}
        N\rho
        \geq & ~ Nn \eta \notag\\
        \geq & ~ \Omega\left(
        d
        \log\left(\frac d\delta\right)
    \right) \notag\\
    \geq & ~ \Omega(\log(1/\delta)).
    \end{align}
Therefore, in all cases, we can get $N\rho\ge\Omega(\log(1/\delta))$ and apply Claim~\ref{cla:covariance-condition_part_1_4}.
Moreover, by Claim~\ref{cla:covariance-condition_part_3_4}, we have for all $i\in[N]$,
\begin{equation}
    \|\mu_i^{\trtwo}-\mu_i^{\tr}\|_2
    \le
    O\left(\sqrt{\frac\rho n}\right).
    \label{eq:D_user_tr_bias}
\end{equation}

\begin{claim}
\label{cla:D_user_covariance}
With probability at least $1-\delta/4$,
\[
    \frac1N \sum_{i=1}^N
    (Z_i^{\trtwo}-\mu)(Z_i^{\trtwo}-\mu)^\top
    \preceq
    C\left(\frac1n+\beta+\tau\right)\mathbb I_d
\]
as long as
\[
    nN
    \ge
    \Omega\left(
        \frac d\eta
        \log\left(\frac d\delta\right)
    \right).
\]
\end{claim}

\begin{proof}
For each $k\in[N]$, define
\[
    M_k
    :=
    \frac1N
    \paren{Z_k^{\mathrm{tr},(2)}-\mu}
    \paren{Z_k^{\mathrm{tr},(2)}-\mu}^\top .
\]
By the definition of the second truncation,
\[
    \norm{Z_k^{\mathrm{tr},(2)}-\mu_k^{\mathrm{tr}}}_2
    \le
    2\sqrt{\frac d{n\rho}}.
\]
Therefore,
\begin{align*}
    \|Z_k^{\mathrm{tr},(2)}-\mu\|_2
    &\le
    \|Z_k^{\mathrm{tr},(2)}-\mu_k^{\mathrm{tr}}\|_2
    +
    \|\mu_k^{\mathrm{tr}}-\mu_k\|_2
    +
    \|\mu_k-\mu\|_2                                      \\
    &\le
    O\left(
        \sqrt{\frac d{n\rho}}
        +
        \sqrt{\eta}
        +
        \sqrt{\beta}
    \right).
\end{align*}
Hence
\begin{align*}
    \|M_k\|_2
    &=
    \frac1N
    \left\|
        (Z_k^{\mathrm{tr},(2)}-\mu)
        (Z_k^{\mathrm{tr},(2)}-\mu)^\top
    \right\|_2                                           \\
    &=
    \frac1N
    \norm{Z_k^{\mathrm{tr},(2)}-\mu}_2^2                    \\
    &\le
    O\left(
        \frac d{n\rho N}
        +
        \frac{\eta+\beta}{N}
    \right)                                              \\
    &\le
    O\left(
        \frac d{n\rho N}
        +
        \frac{\beta+\tau}{N}
    \right),
\end{align*}
where the last line uses $\eta\le \tau$ since $\tau=\eta/\rho$ and
$\rho\le1$. Thus
\[
    \mathbf 0_{d\times d}
    \preceq
    M_k
    \preceq
    R\mathbb I_d,
    \qquad
    R
    =
    O\left(
        \frac d{n\rho N}
        +
        \frac{\beta+\tau}{N}
    \right).
\]

Next, since the second truncation maps $Z_k^{\mathrm{tr}}$ to either itself or
its mean $\mu_k^{\mathrm{tr}}$, it does not increase covariance. Hence
\[
    \cov[Z_k^{\mathrm{tr},(2)}]
    \preceq
    \cov[Z_k^{\mathrm{tr}}]
    \preceq
    \frac1n \mathbb I_d .
\]
Since $\|\mu_k^{\mathrm{tr},(2)}-\mu_k^{\mathrm{tr}}\|_2
    \le
    O\left(\sqrt{\frac{\rho}{n}}\right)$, $\|\mu_k^{\mathrm{tr}}-\mu_k\|_2
    \le
    O(\sqrt{\eta})$, and $\|\mu_k-\mu\|_2\le\sqrt{\beta}$, we have
\[
    \|\mu_k^{\mathrm{tr},(2)}-\mu\|_2^2
    \le
    O\left(
        \frac{\rho}{n}
        +
        \eta
        +
        \beta
    \right)
    \le
    O\left(
        \frac1n
        +
        \beta
        +
        \tau
    \right),
\]
where we used $\rho\le1$, $\eta\le\tau$, and $\eta\ge\beta$. Therefore,
\begin{align*}
    \E\left[
        \frac1N\sum_{k=1}^N
        (Z_k^{\mathrm{tr},(2)}-\mu)
        (Z_k^{\mathrm{tr},(2)}-\mu)^\top
    \right]
    &=
    \frac1N\sum_{k=1}^N
    \left(
        \cov[Z_k^{\mathrm{tr},(2)}]
        +
        (\mu_k^{\mathrm{tr},(2)}-\mu)
        (\mu_k^{\mathrm{tr},(2)}-\mu)^\top
    \right)                                             \\
    &\preceq
    C_0
    \left(
        \frac1n+\beta+\tau
    \right)\mathbb I_d
\end{align*}
for a universal constant $C_0>0$. Thus
\[
    \lambda_{\max}
    :=
    \left\|
        \E\left[\sum_{k=1}^N M_k\right]
    \right\|_2
    \le
    C_0
    \left(
        \frac1n+\beta+\tau
    \right).
\]

Let
    $L \eqdef \frac1n+\beta+\tau$.
Choose a sufficiently large universal constant $C_1>C_0$. By the matrix
Chernoff inequality,
\begin{align*}
    \Pr\left[
        \left\|\sum_{k=1}^N M_k\right\|_2
        >
        C_1 L
    \right]
    &\le
    d\exp\left(
        -\Omega\left(\frac{L}{R}\right)
    \right)                                             \\
    &\le
    d\exp\left(
        -\Omega\left(
            \frac{
                N(1+n(\beta+\tau))
            }{
                d/\rho+n(\beta+\tau)
            }
        \right)
    \right).
\end{align*}
It remains to simplify the exponent. Let $s:=\beta+\tau$. To make the failure probability upper bounded by $\delta$, it suffices to make
\[
    \frac{N(1+ns)}{d/\rho+ns}
    \ge
    C\log\left(\frac d\delta\right).
\]
Equivalently, it is enough that
\[
    nN
    \ge
    C\log\left(\frac d\delta\right)
    \cdot
    \frac{n(d/\rho+ns)}{1+ns}.
\]
Since $s\ge\tau=\eta/\rho$,
    $1+ns
    \ge
    n\tau
    =
    \frac{n\eta}{\rho}$.
Thus, we have
\[
    \frac{nd/\rho}{1+ns}
    \le
    \frac d\eta.
\]
Also,
\[
    \frac{n^2s}{1+ns}
    \le
    n.
\]
Therefore,
\[
    \frac{n(d/\rho+ns)}{1+ns}
    \le
    \frac d\eta+n.
\]
In the regime $n=o(d)$ and $\eta\in(0,1)$, we have
$n\le O(d/\eta)$. Hence
\[
    nN
    \ge
    \Omega\left(
        \frac d\eta
        \log\left(\frac d\delta\right)
    \right)
\]
is sufficient to make the failure probability at most $\delta / 4$. This completes
the proof of the claim.
\end{proof}

\textbf{Constructing a feasible assignment for the polynomial system $\mathsf D$.}
Let $\mathcal G$ be the set of users not corrupted at the batch level, so
$|\mathcal G|\ge(1-\eps)N$. Define
\[
    \mathcal H
    :=
    \{i\in\mathcal G: Z_i^{\trtwo}=Z_i^{\tr}\}.
\]
On the event Eq.~\eqref{eq:D_user_tr_count_event}, and since $\eps\le\rho$,
\[
    |\mathcal H|
    \ge
    (1-\eps-\rho)N
    \ge
    (1-2\rho)N.
\]
Choose a subset $\mathcal S\subseteq\mathcal H$ satisfying
\[
    |\mathcal S|=(1-\Theta\paren{\rho})N.
\]

We set $\mathsf Z_{i,j}=z_{i,j}^{\tr}$, $\mathsf Z_i=z_i^{\tr}:=\frac1n\sum_{j=1}^n z_{i,j}^{\tr}$, $\mathsf Y_i=z_i^{\trtwo}$, and $U_i=\indic{i\in\mathcal S}$.
For $i\in\mathcal S$, let
\[
    A_i:=\{j\in[n]: z_{i,j}^{\tr}=x_{i,j}\}.
\]
Since $i\in\mathcal S\subseteq\mathcal G$, at most $\alpha n$ observed samples
in user $i$ are sample-corrupted. Therefore
\[
    |A_i|
    \ge
    n-\alpha n-\sum_{j=1}^n
    \indic{z_{i,j}\neq z_{i,j}^{\tr}}.
\]
Summing over $i\in\mathcal S$ and using
Eq.~\eqref{eq:D_sample_tr_count_event},
\[
    \sum_{i\in\mathcal S}|A_i|
    \ge
    n|\mathcal S|-\alpha n|\mathcal S|-\eta Nn
    \ge
    (1- \Theta\paren{\eta})n|\mathcal S|.
\]
We may therefore choose
exactly $(1-\Theta\paren{\eta})n|\mathcal S|$ pairs $(i,j)$ with
$i\in\mathcal S$ and $j\in A_i$. Set $W_{i,j}=1$ on these chosen pairs and
$W_{i,j}=0$ everywhere else. Then
\[
    W_{i,j}(\mathsf Z_{i,j}-x_{i,j})=0
\]
for all $i,j$, and the global sample-count constraint
Eq.~\eqref{eq:D_global_sample_count} holds.

Moreover,
if $U_i=1$, then $i\in\mathcal S\subseteq\mathcal H$, so
$z_i^{\trtwo}=z_i^{\tr}$ and hence
\[
    U_i(\mathsf Y_i-\mathsf Z_i)=0.
\]
It remains only to verify the two covariance constraints. By
Claim~\ref{cla:D_sample_covariance},
\[
    \frac1{Nn}\sum_{i,j}
    (z_{i,j}^{\tr}-\mu)(z_{i,j}^{\tr}-\mu)^\top
    \preceq C\mathbb I_d.
\]
Centering at the empirical mean can only decrease the empirical second moment,
so Eq.~\eqref{eq:D_sample_cov} holds after choosing $C_D$ sufficiently large
and taking $B_1$ to be a square root of the resulting positive semidefinite
slack matrix. Similarly, Claim~\ref{cla:D_user_covariance} and the same
centering argument imply Eq.~\eqref{eq:D_user_cov}, after increasing $C_D$ if
necessary and choosing $B_2$ as a square root of the corresponding slack
matrix. A union bound over the four high-probability events completes the
proof.
\end{proof}

\subsection{Identifiability}
\label{sub:heterogeneous_two_level:identify}

After proving satisfiability of $\mathsf D$, we show that any feasible
pseudoexpectation satisfying $\mathsf D$ yields an accurate estimate of the
true mean $\mu$.

\begin{lemma}[Identifiability of $\mathsf D$]
\label{lem:D_identifiability}
Under the setting of Problem~\ref{prob:heterogeneous_two_level}, suppose
$\eta=\max\{\alpha,\beta\}<c$ and $\eps<c$ for a sufficiently small universal
constant $c>0$. Suppose $\alpha n\ge1$ and $nN
    \ge
    \Omega\left(
        \frac d\eta
        \log\left(\frac d\delta\right)
    \right)$.
Then, with probability at least $1-\delta$, every pseudoexpectation
$\widetilde{\mathbb E}$ satisfying $\mathsf D$ yields
\[
    \|\widehat\mu-\mu\|_2
    \le
    O\left(\sqrt{\alpha+\beta}\right).
\]
\end{lemma}

\begin{proof}
Let $z_{i,j}^{\tr}$ and $z_i^{\trtwo}$ denote the realized truncated samples from the satisfiability proof, and define
    $\widetilde\mu_i^{\tr}
    :=
    \frac1n\sum_{j=1}^n z_{i,j}^{\tr}$,
    $\widetilde\mu_i^{\trtwo}
    :=
    z_i^{\trtwo}$, and
    $\widetilde\mu^{\trtwo}
    :=
    \frac1N\sum_{i=1}^N \widetilde\mu_i^{\trtwo}$.

We first compare $\widetilde{\mathbb E}[\overline{\mathsf Y}]$ to the fixed
reference average $\widetilde\mu^{\trtwo}$. Let $\mathcal I_i$ be the
deterministic indicator of the event that user $i$ is a good reference user:
user $i$ is not corrupted at the batch level, $z_i^{\trtwo}=z_i^{\tr}$, and
the observed batch agrees with the sample-level truncation except on the
sample-corruption and truncation indices. The proof of
Lemma~\ref{lem:mean_close_5} applies with the following substitutions:
    $\eps \mapsto O(\rho)$, $\alpha \mapsto O(\eta)$, and
    $\frac1n+\tau \mapsto \frac1n+\beta+\tau$.
We can get
\[
    \left\|
        \widetilde{\mathbb E}[\overline{\mathsf Y}]
        -
        \widetilde\mu^{\trtwo}
    \right\|_2^2
    \le
    O\left(
        \rho\left(\frac1n+\beta+\tau\right)
        +
        \eta
    \right).
\]
Recall that $\tau=\eta/\rho$ and $\rho = \min\{\max\{\eps,n\eta\},c_\rho\}$. Therefore, we have
\[
    \rho\left(\frac1n+\beta+\tau\right)+\eta
    =
    \frac\rho n+\rho\beta+\eta+\eta
    \le
    O\left(\frac{\eps}{n}+\eta\right).
\]
Consequently,
\begin{equation}
    \left\|
        \widetilde{\mathbb E}[\overline{\mathsf Y}]
        -
        \widetilde\mu^{\trtwo}
    \right\|_2
    \le
    O\left(
        \sqrt{\frac{\eps}{n}}
        +
        \sqrt\eta
    \right).
    \label{eq:D_ident_pseudo_to_ref}
\end{equation}

It remains to compare $\widetilde\mu^{\trtwo}$ with $\mu$.
By the triangle inequality,
\begin{equation}
\begin{aligned}
    \|\widetilde\mu^{\trtwo}-\mu\|_2
    \le&
    \left\|
        \frac1N\sum_{i=1}^N
        (z_i^{\trtwo}-\mu_i^{\trtwo})
    \right\|_2
    +
    \left\|
        \frac1N\sum_{i=1}^N
        (\mu_i^{\trtwo}-\mu_i^{\tr})
    \right\|_2                                      \\
    &+
    \left\|
        \frac1N\sum_{i=1}^N
        (\mu_i^{\tr}-\mu_i)
    \right\|_2
    +
    \left\|
        \frac1N\sum_{i=1}^N
        (\mu_i-\mu)
    \right\|_2 .
\end{aligned}
\label{eq:D_ident_ref_decomp}
\end{equation}
For the first term, the variables $Z_i^{\trtwo}-\mu_i^{\trtwo}$ are
independent, mean zero, have covariance at most $n^{-1}\mathbb I_d$, and are
bounded by $O(\sqrt{d/(n\rho)})$. Therefore Fact~\ref{fac:vec_bernstein}
gives, with probability at least $1-\delta$,
\[
    \left\|
        \frac1N\sum_{i=1}^N
        (z_i^{\trtwo}-\mu_i^{\trtwo})
    \right\|_2
    \le
    O\left(
        \sqrt{\frac{d\log(1/\delta)}{nN}}
        +
        \frac1N
        \sqrt{\frac d{n\rho}}
        \log(1/\delta)
    \right)
    \le
    O(\sqrt\eta),
\]
where the last inequality follows from $N \rho \geq \Omega(\log(1/\delta))$ (Eq.~\eqref{eq:nrho}) and $nN
    \ge
    \Omega\left(
        \frac d\eta
        \log\left(\frac d\delta\right)
    \right)$. The remaining three terms in
Eq.~\eqref{eq:D_ident_ref_decomp} are bounded by
    $O\left(\sqrt{\frac\rho n}\right)$ (Eq.~\eqref{eq:D_user_tr_bias}),
    $O(\sqrt\eta)$ (Eq.~\eqref{eq:D_sample_tr_bias}), and
    $\sqrt\beta$ (Problem~\ref{prob:heterogeneous_two_level}),
respectively. Since
\[
    \frac\rho n
    \le
    O\left(\frac{\eps}{n}+\eta\right),
\]
we obtain
\begin{equation}
    \|\widetilde\mu^{\trtwo}-\mu\|_2
    \le
    O\left(
        \sqrt{\frac{\eps}{n}}
        +
        \sqrt\eta
        +
        \sqrt\beta
    \right).
    \label{eq:D_ident_ref_to_mu}
\end{equation}

Combining Eqs.~\eqref{eq:D_ident_pseudo_to_ref} and
\eqref{eq:D_ident_ref_to_mu},
\[
    \|\widehat\mu-\mu\|_2
    \le
    O\left(
        \sqrt{\frac{\eps}{n}}
        +
        \sqrt\eta
        +
        \sqrt\beta
    \right).
\]
Finally, $\alpha n\ge1$ and $\eps<1$ imply $\eps/n\le\alpha$, while
$\eta=\max\{\alpha,\beta\}$. Hence
\[
    \|\widehat\mu-\mu\|_2
    \le
    O\left(\sqrt{\alpha+\beta}\right).
\]
\end{proof}

\begin{proof}[Proof of Theorem~\ref{thm:heterogeneous_two_level_upper}]
By Lemma~\ref{lem:D_satisfiability}, with probability at least $1-\delta/2$,
the polynomial system $\mathsf D$ is satisfiable. We run the standard
constant-degree SoS SDP relaxation for $\mathsf D$ and output
\[
    \widehat\mu:=\widetilde{\mathbb E}[\overline{\mathsf Y}].
\]
By Lemma~\ref{lem:D_identifiability}, with probability at least $1-\delta/2$,
every feasible pseudoexpectation gives
\[
    \|\widehat\mu-\mu\|_2
    \le
    O(\sqrt{\alpha+\beta}).
\]
A union bound gives the desired success probability. The polynomial system has
$Nnd+Nd+d^2+Nn+N$ variables and constant degree, so the standard SoS SDP
algorithm runs in time $(Nnd+d^2)^{O(1)}$.
\end{proof}

%% file: hardness.tex
\section{Hardness}
\label{sec:hard}

In this section, we present the proof of the lower bound of our main result.

\subsection{Hardness of Problem~\ref{prob:bounded_mean}}
\label{sub:hard:bounded_mean}

\begin{theorem}\label{thm:bounded_mean_lower}
Consider the model in Problem~\ref{prob:bounded_mean}.
Then there exists a choice of uncorrupted user distributions $\{\mathcal{D}_i\}$ and an $\eps$-user adversary such that for any estimator
$\wh{\mu}$,
\begin{align*}
    \Pr\left[ \|\wh{\mu}-\mu\|_2 \ge \Omega\paren{ \sqrt{\frac{\eps}{n}} + \sqrt{\frac{d}{nN}} + \sqrt{\beta} } \right]\ge\frac{1}{2}.
\end{align*}
\end{theorem}
\begin{proof}
    It follows from Lemma~\ref{lem:alpha0_lower_bound_sqrteps_over_n}. The $\Omega(\sqrt{\beta})$ lower bound is straightforward: suppose the true mean is $\mu$, for all users, the adversary could simply change the mean to some $\mu'$ that is $\sqrt{\beta}$ apart, and no algorithms could tell whether the true mean is $\mu$ or $\mu'$.
\end{proof}

\subsection{Hardness of Problem~\ref{prob:arbitrary_alpha}}
\label{sub:hard:arbitrary_alpha}

\begin{theorem}\label{thm:arbitrary_alpha_lower}
Given $\eps \in \paren{0,\frac{1}{18}}$ and $\alpha \in \paren{0,\frac{1}{90}}$ satisfying $\eps + 5\alpha < \frac{1}{18}$ and $\alpha n \ge 4\log\paren{\frac{4}{\eps}}$, and $n=o(d)$ samples from each of $N$ users, there exists an instance
of Problem~\ref{prob:arbitrary_alpha}, consisting of a choice of
uncorrupted data distributions with mean $\mu\in\R^d$ and a choice
of adversarial corruptions, such that for any estimator
$\wh\mu\in\R^d$,
$\norm{\wh\mu-\mu}_2 = \Omega\paren{\sqrt{\frac{\eps}{n}}+ \sqrt{\frac{d}{nN}}+\sqrt{\alpha}}$.
\end{theorem}
\begin{proof}
    The $\Omega(\sqrt{\eps/n})$ term follows from
Lemma~\ref{lem:alpha0_lower_bound_sqrteps_over_n},
and the $\Omega(\sqrt{\alpha})$ term follows from
Lemma~\ref{lem:alpha_lower_bound_sqrtalpha}.
The classical lower bound for mean estimation from $Nn$ samples with
bounded covariance gives
$\Omega(\sqrt{d/(Nn)})$.
Combining these three lower bounds completes the proof.
\end{proof}

\subsubsection{The Case When \texorpdfstring{$\alpha = 0$}{}}

\begin{lemma}
\label{lem:alpha0_lower_bound_sqrteps_over_n}
Consider Problem~\ref{prob:arbitrary_alpha} with $\alpha = 0$ and
$\eps \in (0,1/18)$. Then there exist two hypotheses $H_0,H_1$ over
$\mathbb{R}$ with $\Var\bracket{H_0},\Var\bracket{H_1}\le 1$ such that
$\abs{\mu_{H_0}-\mu_{H_1}} = \Omega\paren{\sqrt{\frac{\eps}{n}}}$.
Moreover, for all estimator $\wt\mu$, at least one of the two
hypotheses satisfies
\begin{align*}
    \Pr\bracket{\abs{\wt\mu-\mu} \ge \frac{1}{4} \sqrt{\frac{\eps}{n}}} \ge \frac{3}{4}.
\end{align*} 
\end{lemma}
\begin{proof}
We define $\eps_0 \eqdef \frac{\eps}{4n}$ and respectively define the hypotheses $H_0$ and $H_1$ on $\R$ as 
\begin{align*}
    \Pr_{X \sim H_0}\bracket{X = 0} = 1 - \eps_0, \quad \Pr_{X \sim H_0}\bracket{X = \frac{1}{\sqrt{\eps_0}}} = \eps_0,
\end{align*}
and
\begin{align*}
    \Pr_{Y \sim H_1}\bracket{Y = 0} = 1.
\end{align*}
We first compute their means and variances. For the means, we have
\begin{align*}
    \mu_{H_0}
    = & ~ \E_{X \sim H_0}\bracket{X}\\
    = & ~ \frac{1}{\sqrt{\eps_0}} \cdot \eps_0\\
    = & ~ \sqrt{\eps_0} \\
    = & ~ \sqrt{\frac{\eps}{4n}}.
\end{align*}
and 
\begin{align*}
    \mu_{H_1} = 0.
\end{align*}
For their variances, we have
\begin{align*}
    \Var_{X \sim H_0}\bracket{X}
    = & ~ \E_{X \sim H_0}\bracket{X^2} - \E_{X \sim H_0}\bracket{X}^2\\
    = & ~ \frac{\eps_0}{\eps_0} - \eps_0 \\
    \leq & ~ 1,
\end{align*}
and
\begin{align*}
    \Var_{Y \sim H_1}\bracket{Y} = 0 < 1.
\end{align*}
Therefore, we have
\begin{align}\label{eq:mean_diff}
    \abs{\mu_{H_0} - \mu_{H_1}} = \sqrt{\frac{\eps}{4n}}.
\end{align}
Now, for each user $i \in [N]$, let
\begin{align*}
    I_i \eqdef \indic{\text{the $i$-th user's batch contains at least one nonzero sample}}.
\end{align*}
For each user $i \in [N]$ and each sample index $j \in [n]$, let $X_{i,j}$
denote the $j$-th sample in the batch of user $i$. Under the hypothesis $H_0$,
the random variables $X_{i,1},\dots,X_{i,n}$ are i.i.d.\ draws from $H_0$. Since
\begin{align*}
    \{I_i = 1\} = \bigcup_{j=1}^n \{X_{i,j} \neq 0\},
\end{align*}
by the union bound, we have
\begin{align*}
    \Pr\bracket{I_i = 1}
    = &~ \Pr\bracket{\bigcup_{j=1}^n \{X_{i,j} \neq 0\}} \\
    \leq &~ \sum_{j=1}^n \Pr\bracket{X_{i,j} \neq 0} \\
    = &~ \sum_{j=1}^n \eps_0 \\
    = &~ n \eps_0.
\end{align*}
Since we have $\eps_0 = \frac{\eps}{4n}$, we can get
\begin{align*}
    \Pr\bracket{I_i = 1} \leq \frac{\eps}{4}.
\end{align*}
We define $M \eqdef \sum_{i = 1}^N I_i$ to be the number of users whose batch contains at least one nonzero sample. Therefore, we have
\begin{align*}
    \E\bracket{M}
    = & ~ \E\bracket{\sum_{i = 1}^N I_i}\\
    = & ~ \sum_{i = 1}^N \E\bracket{I_i}\\
    = & ~ \sum_{i = 1}^N \Pr\bracket{I_i = 1}\\
    \leq & ~ \frac{\eps N}{4}.
\end{align*}
Hence, applying Markov's inequality (Fact~\ref{fac:markov}), we have
\begin{align}\label{eq:M_eps_N}
    \Pr\bracket{M > \eps N} \leq \frac{\E\bracket{M}}{\eps N} \leq \frac{1}{4}.
\end{align}
We note that on the event $\cbracket{M \leq \eps N}$, the adversary has the power to corrupt all users whose batch has at least one nonzero sample. Replacing all samples in each such batch by $0$, the adversary makes the entire observed dataset equal to the all-zero dataset, which we denote it by $z$. In $H_1$, the observed dataset is already all 0 without corruption. For all estimator $\wt\mu\paren{z}$, by the triangle inequality, we have
\begin{align*}
    \abs{\mu_{H_0} - \mu_{H_1}}
    = & ~ \abs{\mu_{H_0} - \wt\mu\paren{z} + \wt\mu\paren{z} - \mu_{H_1}} \\
    \leq & ~ \abs{\mu_{H_0} - \wt\mu\paren{z}} + \abs{\wt\mu\paren{z} - \mu_{H_1}}.
\end{align*}
Thus, it is impossible that both $\abs{\mu_{H_0} - \wt\mu\paren{z}} < \frac{1}{2} \abs{\mu_{H_0} - \mu_{H_1}}$ and $\abs{\wt\mu\paren{z} - \mu_{H_1}} < \frac{1}{2} \abs{\mu_{H_0} - \mu_{H_1}}$. Via De Morgan's Laws, we have either
\begin{align}\label{eq:demo_1}
    \abs{\mu_{H_0} - \wt\mu\paren{z}} \geq \frac{1}{2} \abs{\mu_{H_0} - \mu_{H_1}}
\end{align}
or
\begin{align}\label{eq:demo_2}
    \abs{\wt\mu\paren{z} - \mu_{H_1}} \geq \frac{1}{2} \abs{\mu_{H_0} - \mu_{H_1}}.
\end{align}

Suppose that Eq.~\eqref{eq:demo_1} holds. Then, under the hypothesis $H_0$, on the event $\cbracket{M \leq \eps N}$ where the adversary makes the observed dataset equal to $z$, we have
\begin{align*}
    \Pr_{H_0}\bracket{\abs{\mu_{H_0} - \wt\mu\paren{z}} \geq \frac{1}{2} \abs{\mu_{H_0} - \mu_{H_1}}}
    \geq & ~ \Pr_{H_0}\bracket{\cbracket{M \leq \eps N}} \\ 
    \geq & ~ \frac{3}{4},
\end{align*}
where the first step follows from the event $\cbracket{M \leq \eps N}$ implies $\abs{\mu_{H_0} - \wt\mu\paren{z}} \geq \frac{1}{2} \abs{\mu_{H_0} - \mu_{H_1}}$ and the second step follows from Eq.~\eqref{eq:M_eps_N}.

Suppose that Eq.~\eqref{eq:demo_2} holds. Then, under the hypothesis $H_1$, the observed dataset is always equal to $z$, we have
\begin{align*}
    \Pr_{H_1}\bracket{\abs{\wt\mu\paren{z} - \mu_{H_1}} \geq \frac{1}{2} \abs{\mu_{H_0} - \mu_{H_1}}} = 1.
\end{align*}
We complete the proof by $\abs{\mu_{H_0} - \mu_{H_1}} = \sqrt{\frac{\eps}{4n}}$ (Eq.~\eqref{eq:mean_diff}).
\end{proof}

\subsubsection{The \texorpdfstring{$\Omega(\sqrt{\alpha})$}{} Lower Bound}

\begin{lemma}
\label{lem:alpha_lower_bound_sqrtalpha}
Consider Problem~\ref{prob:arbitrary_alpha} with
$\eps \in (0,1/18)$ and $\alpha \in (0,1/90)$.
Suppose that
\[
    \alpha n \ge 4 \log\paren{\frac{4}{\eps}}.
\]
Then, there exist two hypotheses $H_2,H_3$ on $\R$ such that
$\Var\bracket{H_2} \leq 1$,
$\Var\bracket{H_3} \leq 1$, and
\[
    \abs{\mu_{H_2}-\mu_{H_3}}
    = \frac{1}{2}\sqrt{\alpha},
\]
and for every estimator $\widetilde{\mu}$, at least one of the two
hypotheses satisfies
\[
    \Pr\bracket{
        \abs{\widetilde{\mu}-\mu}
        \geq \frac{1}{4}\sqrt{\alpha}
    }
    \geq \frac{3}{4}.
\]
\end{lemma}

\begin{proof}
We respectively define the hypotheses $H_2$ and $H_3$ on $\R$ as
\begin{align*}
    \Pr_{X \sim H_2}\bracket{X = 0}
    &= 1 - \frac{\alpha}{4},
    &
    \Pr_{X \sim H_2}\bracket{
        X = \frac{2}{\sqrt{\alpha}}
    }
    &= \frac{\alpha}{4},
\end{align*}
and
\begin{align*}
    \Pr_{Y \sim H_3}\bracket{Y = 0} = 1.
\end{align*}
Therefore, their means are
\begin{align*}
    \mu_{H_2}
    &= \E_{X \sim H_2}\bracket{X}\\
    &= \paren{\frac{\alpha}{4}}
       \cdot
       \paren{\frac{2}{\sqrt{\alpha}}}\\
    &= \frac{\sqrt{\alpha}}{2},
\end{align*}
and
\begin{align*}
    \mu_{H_3}=0.
\end{align*}
Hence
\begin{align}
\label{eq:mean_change}
    \abs{\mu_{H_2}-\mu_{H_3}}
    =
    \frac{\sqrt{\alpha}}{2}.
\end{align}

Furthermore,
\begin{align*}
    \Var_{X \sim H_2}\bracket{X}
    &= \E_{X \sim H_2}\bracket{X^2}
       -\E_{X \sim H_2}\bracket{X}^2\\
    &= \paren{\frac{\alpha}{4}}
       \paren{\frac{4}{\alpha}}
       -\frac{\alpha}{4}\\
    &= 1-\frac{\alpha}{4}\\
    &\leq 1,
\end{align*}
while
\begin{align*}
    \Var_{Y \sim H_3}\bracket{Y}=0.
\end{align*}
Thus, both $H_2$ and $H_3$ satisfy the bounded variance assumption.

Now, for each user $i\in[N]$, consider $n$ i.i.d.\ samples
$X_{i,1},\dots,X_{i,n}$ drawn from $H_2$, and define
\[
    T_i
    \eqdef
    \sum_{j=1}^n
    \indic{X_{i,j}\neq 0}.
\]
Then
\[
    T_i
    \sim
    \mathrm{Bin}\paren{n,\frac{\alpha}{4}},
\]
and therefore
\begin{align}
\label{eq:mu_ti}
    \E[T_i]
    =
    \frac{\alpha n}{4}.
\end{align}

Using the multiplicative Chernoff bound
(Fact~\ref{fac:chernoff}) with $\delta=3$, we have
\begin{align*}
    \Pr\left[
        T_i
        \geq
        4\E[T_i]
    \right]
    \leq
    \exp\paren{-\E[T_i]}.
\end{align*}
Combining this with Eq.~\eqref{eq:mu_ti}, we obtain
\begin{align}
\label{eq:overflow_prob}
    \Pr\left[
        T_i > \alpha n
    \right]
    \leq
    \Pr\left[
        T_i \geq \alpha n
    \right]
    \leq
    \exp\paren{-\frac{\alpha n}{4}}.
\end{align}

We call user $i$ an \emph{overflow user} if
\[
    T_i>\alpha n.
\]
Let
\[
    B_i
    \eqdef
    \indic{T_i>\alpha n}
\]
and let
\[
    S
    \eqdef
    \sum_{i=1}^N B_i
\]
denote the total number of overflow users.

By Eq.~\eqref{eq:overflow_prob},
\begin{align*}
    \E[S]
    &=
    \sum_{i=1}^N
    \Pr[T_i>\alpha n]\\
    &\leq
    N\exp\paren{-\frac{\alpha n}{4}}.
\end{align*}

We now use the user-level corruption budget.
By Markov's inequality,
\begin{align*}
    \Pr\bracket{S>\eps N}
    &\leq
    \frac{\E[S]}{\eps N}\\
    &\leq
    \frac{1}{\eps}
    \exp\paren{-\frac{\alpha n}{4}}.
\end{align*}
Since
\[
    \alpha n
    \geq
    4\log\paren{\frac{4}{\eps}},
\]
we have
\[
    \exp\paren{-\frac{\alpha n}{4}}
    \leq
    \frac{\eps}{4},
\]
and therefore
\begin{align*}
    \Pr\bracket{S>\eps N}
    \leq
    \frac14.
\end{align*}

Define the event
\begin{align*}
    \mathcal E
    \eqdef
    \cbracket{
        S\leq \eps N
    }.
\end{align*}
Then
\begin{align}
\label{eq:event_E_prob}
    \Pr_{H_2}\bracket{\mathcal E}
    \geq
    \frac34.
\end{align}

We next describe the adversary under $H_2$ on the event $\mathcal E$.

For every non-overflow user $i$, namely every user satisfying
\[
    T_i\leq \alpha n,
\]
the sample-level adversary replaces all $T_i$ nonzero samples by $0$.
This uses at most $\alpha n$ sample corruptions in that user's batch.

For every overflow user $i$, namely every user satisfying
\[
    T_i>\alpha n,
\]
the adversary uses the user-level corruption budget and replaces the
entire batch by zeros.

On the event $\mathcal E$, there are at most $\eps N$ overflow users,
so this uses at most the allowed $\eps$-fraction of fully corrupted users.
Hence the above is a valid double-corruption adversary.

After these modifications, every observed sample is equal to zero.
Let $z$ denote this all-zero dataset.

Under hypothesis $H_3$, the clean observed dataset is also always
equal to $z$.

Let $\widetilde{\mu}(z)$ denote the output of an arbitrary estimator
on $z$. By the triangle inequality,
\begin{align*}
    \frac{\sqrt{\alpha}}{2}
    &=
    \abs{\mu_{H_2}-\mu_{H_3}}\\
    &=
    \abs{
        \mu_{H_2}
        -\widetilde{\mu}(z)
        +
        \widetilde{\mu}(z)
        -\mu_{H_3}
    }\\
    &\leq
    \abs{
        \mu_{H_2}-\widetilde{\mu}(z)
    }
    +
    \abs{
        \widetilde{\mu}(z)-\mu_{H_3}
    },
\end{align*}
where the first equality follows from Eq.~\eqref{eq:mean_change}.

Therefore, it is impossible that both
\[
    \abs{
        \mu_{H_2}-\widetilde{\mu}(z)
    }
    <
    \frac{\sqrt{\alpha}}{4}
\]
and
\[
    \abs{
        \mu_{H_3}-\widetilde{\mu}(z)
    }
    <
    \frac{\sqrt{\alpha}}{4}
\]
hold simultaneously. Hence either
\begin{align}
\label{eq:demorgan_1_new}
    \abs{
        \mu_{H_2}-\widetilde{\mu}(z)
    }
    \geq
    \frac{\sqrt{\alpha}}{4},
\end{align}
or
\begin{align}
\label{eq:demorgan_2_new}
    \abs{
        \mu_{H_3}-\widetilde{\mu}(z)
    }
    \geq
    \frac{\sqrt{\alpha}}{4}.
\end{align}

Suppose Eq.~\eqref{eq:demorgan_1_new} holds.
Under hypothesis $H_2$, on the event $\mathcal E$, the adversary makes
the observed dataset equal to $z$. Therefore,
\begin{align*}
    \Pr_{H_2}\left[
        \abs{
            \widetilde{\mu}-\mu_{H_2}
        }
        \geq
        \frac{\sqrt{\alpha}}{4}
    \right]
    &\geq
    \Pr_{H_2}\bracket{\mathcal E}\\
    &\geq
    \frac34,
\end{align*}
where the last inequality follows from
Eq.~\eqref{eq:event_E_prob}.

Otherwise, Eq.~\eqref{eq:demorgan_2_new} holds.
Under hypothesis $H_3$, the observed dataset is always $z$, and hence
\begin{align*}
    \Pr_{H_3}\left[
        \abs{
            \widetilde{\mu}-\mu_{H_3}
        }
        \geq
        \frac{\sqrt{\alpha}}{4}
    \right]
    =
    1.
\end{align*}
This completes the proof.
\end{proof}